\documentclass[11pt]{article}
\usepackage[
lmargin=3cm,rmargin=3cm,
top=1in,
bottom=1in,
]{geometry}
\usepackage[utf8]{inputenc}
\usepackage[T1]{fontenc}
\usepackage{natbib}

\usepackage{microtype}
\usepackage{graphicx}
\usepackage{subcaption}
\usepackage{booktabs} 

\usepackage{tikz}

\usepackage{hyperref}

\usepackage{algorithm}
\usepackage{algorithmic}

\usepackage{amsmath}
\usepackage{amssymb}
\usepackage{mathtools}
\usepackage{amsthm}

\usepackage[textsize=tiny]{todonotes}

\theoremstyle{plain}
\newtheorem{theorem}{Theorem}[section]

\newtheorem{lemma}[theorem]{Lemma}

\theoremstyle{definition}
\newtheorem{definition}[theorem]{Definition}
\newtheorem{assumption}[theorem]{Assumption}
\theoremstyle{remark}
\newtheorem{remark}[theorem]{Remark}

\DeclareMathOperator*{\argmin}{argmin}
\DeclareMathOperator{\clip}{clip}
\DeclareMathOperator{\Exp}{\mathbb E}

\newcommand{\Rbb}{\mathbb{R}}

\newcommand{\Acal}{\mathcal{A}}

\newcommand{\Ccal}{\mathcal{C}}
\newcommand{\Dcal}{\mathcal{D}}
\newcommand{\Scal}{\mathcal{S}}
\newcommand{\Ncal}{\mathcal{N}}
\newcommand{\Ocal}{\mathcal{O}}
\newcommand{\innerp}[2]{\left\langle #1, #2 \right\rangle}
\newcommand{\mean}{\mathrm{E}}
\newcommand{\var}{\mathrm{Var}}
\newcommand\Oi{\Omega_i}
\newcommand\Oj{\Omega^j}
\newcommand\bigO[1]{\Ocal\left(#1\right)}
\newcommand\bigOtilde[1]{\tilde\Ocal\left(#1\right)}
\long\def\commented#1{{}}
\definecolor{mydarkblue}{rgb}{0,0.08,0.45}
\hypersetup{ %
pdftitle={},
pdfauthor={},
pdfkeywords={},
colorlinks=true,
linkcolor=mydarkblue,
citecolor=mydarkblue,
filecolor=mydarkblue,
urlcolor=mydarkblue}
\usepackage[shortlabels]{enumitem}

\begin{document}

\title{Multi-Task Differential Privacy Under Distribution Skew}
\date{}
\newcommand\auth[1]{#1\footnotemark[1]}
\author{
Walid Krichene\footnote{Google Research}\and
\auth{Prateek Jain} \quad\and
\auth{Shuang Song} \and
Mukund Sundararajan\thanks{Google} \and
\auth{Abhradeep Thakurta} \and
\auth{Li Zhang}
}

\maketitle

\begin{abstract}
We study the problem of multi-task learning under user-level differential privacy, in which $n$ users contribute data to $m$ tasks, each involving a subset of users. One important aspect of the problem, that can significantly impact quality, is the distribution skew among tasks. Certain tasks may have much fewer data samples than others, making them more susceptible to the noise added for privacy. It is natural to ask whether algorithms can adapt to this skew to improve the overall utility.

We give a systematic analysis of the problem, by studying how to optimally allocate a user's privacy budget among tasks. We propose a generic algorithm, based on an adaptive reweighting of the empirical loss, and show that when there is task distribution skew, this gives a quantifiable improvement of excess empirical risk.

Experimental studies on recommendation problems that exhibit a long tail of small tasks, demonstrate that our methods significantly improve utility, achieving the state of the art on two standard benchmarks.
\end{abstract}

\section{Introduction}

Machine learning models trained on sensitive user data present the risk of leaking private user information~\citep{dwork2007price,korolova2010privacy,calandrino2011you,shokri2017membership}. Differential Privacy (DP)~\citep{dwork2006calibrating} mitigates this risk, and has become a gold standard widely adopted in industry and government~\citep{abowd2018census,DPSQL,rogers2021linkedin,plume}.

In this work, we adopt the notion of user-level DP~\citep{dwork2014algorithmic,kearns2014mechanism}, which seeks to protect \emph{all data samples} of a given user, a harder goal than protecting a single sample. User-level DP has been studied under different settings, including empirical risk minimization~\citep{amin2019bounding,levy2021learning}, mean estimation~\citep{cummings2022mean}, and matrix completion~\citep{jain2018differentially,DPALS}. The setting we study is that of multi-task learning, in which users contribute data to multiple tasks. This is for example the case in recommendation systems: given a set of $n$ users and $m$ items (pieces of content, such as movies or songs), the goal is to learn a representation for each item based on the users' preferences. Another example is multi-class classification, where the goal is to learn $m$ classifiers (one per class). Notice that the training data of a given user may be relevant to only a small subset of tasks. This is certainly the case in recommendation, where users typically interact with a small subset of the catalog of available content. This can also be the case in classification if the number of classes is very large. The interaction pattern between users and tasks can be described using a bipartite graph $\Omega \subseteq [m] \times [n]$, such that user $j$ contributes data to task $i$ if and only if $(i, j) \in \Omega$. We will also denote by $\Oi = \{j: (i,j) \in \Omega\}$ the set of users contributing to task~$i$. The goal is then to minimize the empirical loss summed over tasks,
\begin{equation}
\label{eq:obj}
L(\theta) = \sum_{i = 1}^m \sum_{j \in \Oi} \ell(\theta_i; x_{ij}, y_{ij}),
\end{equation}
where $\theta_i$ represents the model parameters for task $i$, and $x_{ij}, y_{ij}$ are respectively the features and labels contributed by user $j$ to task $i$. In general, we will assume that $\Omega$, $x$, and $y$ are all private.

In practice, the tasks can be imbalanced, i.e. the distribution of task sizes $|\Oi|$ can be heavily skewed. This is the case in recommendation systems, in which there is a long tail of items with orders of magnitude fewer training examples than average~\citep{yin2012longtail,liu2020longtail}. Classification tasks can also exhibit a similar skew among classes~\citep{kubat1997addressing}. Under differential privacy constraints, this skew makes tail tasks (those with less data) more susceptible to quality losses. This disparate impact of differential privacy is well documented, for example in classification with class imbalance~\citep{bagdasaryan2019} or in DP language models~\citep{mcmahan2018}.

A natural question is whether DP algorithms can be made adaptive to such distribution skew. Some heuristics were found to be useful in practice. For example,~\cite{DPALS} propose a sampling heuristic, which limits the number of samples per user, while biasing the sampling towards tail tasks. This was found to improve performance in practice (compared to uniform sampling), though no analysis was provided on the effect of this biased sampling. Another heuristic was proposed by~\cite{xu2021removing}, which applies to DPSGD, and operates by increasing the clipping norm of gradients on tail tasks. Similarly, this was observed to improve performance (compared to uniform clipping).

In both cases, the intuition is that by allocating a larger proportion of a user's privacy budget to the tail (either via sampling or clipping), one can improve utility on the tail, and potentially on the overall objective~\eqref{eq:obj}. However, a formal analysis of biased sampling and clipping remains elusive, as it introduces a bias that is hard to bound. We take a different approach, by introducing task-user weights $w_{ij}$ to have a finer control over the budget allocation. The weights play a similar role to sampling or clipping: by assigning larger weights to a task, we improve utility of that task. The question is then how to optimally choose weights $w$ to minimize excess risk on the overall objective~\eqref{eq:obj}, under privacy constraints (which translate into constraints involving $w$ and $\Omega$). Fortunately, this problem is amenable to formal analysis, and we derive optimal choices of weights, and corresponding error bounds, under different assumptions on the loss and the task-user graph~$\Omega$.



\subsection{Contributions}
\begin{enumerate}[leftmargin=13pt,topsep=0pt,itemsep=0ex,partopsep=0ex,parsep=1pt]
\item We give a formal analysis of the private multi-task learning problem under distribution skew. We propose a generic method based on computing task-user weights, and applying those weights to the loss function~\eqref{eq:obj}, as a mechanism to control each user's privacy budget allocation among tasks. The method is generic in that it applies to several algorithms such as DPSGD~\citep{bassily2014erm,abadi2016dpsgd} or Sufficient Statistics Perturbation~\citep{foulds2016ssp}.
\item We derive utility bounds that explicitly show how the budget allocation  trades-off privacy and utility (Theorems~\ref{thm:tradeoff-ssp} and~\ref{thm:tradeoff-dpsgd}). By adapting the allocation to the task sizes, we obtain utility bounds (Theorem~\ref{thm:exact}) with a guaranteed improvement compared to uniform allocation. The improvement increases with the degree of skew in the task distribution.
\item We conduct experiments on synthetic data and two standard recommendation benchmarks which exhibit a heavy distribution skew. Our methods significantly improve accuracy, even compared to strong baselines such as the tail sampling heuristic of~\cite{DPALS}. We also provide a detailed analysis of quality impact across tasks. We find that adaptive budget allocation visibly improves quality on tail tasks, with a limited impact on head tasks.
\end{enumerate}

\subsection{Related Work}
A common technique in user-level differential privacy is to adaptively control the sensitivity with respect to the data of a user. This is typically done via clipping~\citep{abadi2016dpsgd}, sampling~\citep{kasiviswanathan2013,amin2019bounding}, using privacy filters~\citep{feldman2021individual}, or using weights in the objective function~\citep{proserpio2014wPinq,epasto2020smoothly} to scale the contribution of high sensitivity users. Our approach is related, in that we adaptively control sensitivity. But while these methods focus on adapting to \emph{users} (users with high sensitivity are assigned lower weights), we focus on adapting to \emph{tasks} (within each user, easier tasks are assigned lower weights).

We mention in particular~\citep{epasto2020smoothly}, who use weights to bound user sensitivity, and analyze the optimal choice of weights under certain assumptions. Their approach is perhaps the most similar to ours, in that the choice of weights is phrased as an optimization problem. But our setting is inherently different: they consider the traditional ERM setting with a single task, while we consider the multi-task setting (with heterogeneous tasks). As a result, their solution assigns identical weights to all examples of a given user, while our solution is designed to allocate a user's budget among tasks, and generally assigns different weights to different examples of a user.

Turning to the multi-task setting, prior work on DP multi-task learning~\citep{li2020metalearning,hu2022multitask} has focused on \emph{task-level} privacy, where every user is identified with one task. We consider a more general setting, in which tasks do not necessarily coincide with users. In particular, under our model, it's important that a user can contribute to more than one task.

Another related problem is that of answering multiple linear or convex queries under DP, e.g.~\citep{dwork2010boosting,hardt2010mw,ullman2015private}. Though related, the two problems are different because the utility measure is inherently different: in the multi-query setting, the goal is to answer \emph{all} queries accurately (i.e. a minimax objective), while in the multi-task setting, the utility is a sum across tasks, as in~\eqref{eq:obj}. Besides, algorithms in the multi-query convex setting (closest to ours) often have to maintain a distribution over a discretization of the data space, making them prohibitively expensive in continuous feature spaces, due to an exponential dependence on the dimension.

In practice, budget allocation among tasks is often done through uniform sampling, for example in DP SQL~\citep{DPSQL}, LinkedIn's Audience Engagement API~\citep{rogers2021linkedin}, and Plume \citep{plume}. The techniques that do adapt to the task distribution (via sampling as in~\cite{DPALS} or clipping as in~\cite{xu2021removing}) tend to be heuristic in nature, and lack a rigorous utility analysis.

To the best of our knowledge, we present the first systematic analysis of multi-task, user-level DP under distribution skew. Our method operates directly on the loss function, which makes it applicable to standard algorithms such as DPSGD, and has a guaranteed improvement under distribution skew.


\section{Preliminaries}
\subsection{Notation}
Throughout the paper, we will denote by $\Omega \subseteq [m] \times [n]$ the task-user membership graph, where $m,n$ are the number of tasks and users, respectively. We assume that $m \leq n$. For all $i$, let $\Oi := \{j : (i, j) \in \Omega\}$ be the set of users who contribute to task $i$, and for all $j$, let $\Oj := \{i: (i, j) \in \Omega\}$ be the set of tasks to which $j$ contributes data. We will denote by $n_i = |\Oi|$, and by $n^j = |\Oj|$. Let $\|x\|$ be the $L_2$ norm of a vector $x$, and $\|X\|_F$ be the Frobenius norm of a matrix $X$. For a symmetric matrix $A$, let $A^\dagger$ denote the pseudo-inverse of $A$'s projection on the PSD cone. For a closed convex set $\Ccal$, let $\|\Ccal\|$ denote its Euclidean diameter, and $\Pi_\Ccal(\cdot)$ denote the Euclidean projection on $\Ccal$. Let $\Ncal^d$ denote the multivariate normal distribution (of zero mean and unit variance), and $\Ncal^{d \times d}$ denote the distribution of symmetric $d \times d$ matrices whose upper triangle entries are i.i.d. normal. Finally, whenever we state that a bound holds with high probability (abbreviated as w.h.p.), we mean with probability $1-1/n^c$, where $c$ is any positive constant, say~$2$.

\subsection{User-Level Differential Privacy}
Let $\Dcal$ be a domain of data sets. Two data sets $D, D' \in \Dcal$ are called neighboring data sets if one is obtained from the other by removing all of the data samples from a single user. Let $\Acal: \Dcal \to \Scal$ be a randomized algorithm with output space $\Scal$.
\begin{definition}[User-Level Differential Privacy~\citep{dwork2014algorithmic}]
Algorithm $\Acal$ is $(\epsilon,\delta)$-differentially private (DP) if for all neighboring data sets $D, D' \in \Dcal$, and all measurable $S \in \Scal$,
\[
\Pr\left(\Acal(D)\in S\right)\leq e^\epsilon\Pr\left(\Acal(D')\in S\right) + \delta.
\]
\end{definition}
For simplicity of presentation, our results will be stated in terms of R\'enyi differential privacy (RDP). Translation from RDP to DP is standard, and can be done for example using \citep[Proposition~3]{mironov2017renyi}.
\begin{definition}[User-Level R\'enyi Differential Privacy~\citep{mironov2017renyi}]
An algorithm $\Acal$ is $(\alpha,\rho)$-R\'enyi differentially private (RDP) if for all neighboring data sets $D, D' \in \Dcal$,
\[
D_{\alpha}\left(\Acal(D))||\Acal(D')\right)\leq \rho,
\]
where $D_{\alpha}$ is the R\'enyi divergence of order $\alpha$.
\label{def:diffPriv}
\end{definition}

\section{Budget allocation via weights}
\label{sec:budget}
First, observe that the objective function~\eqref{eq:obj} can be decomposed as
\begin{equation}
L(\theta) = \sum_{i = 1}^m L_i(\theta_i); \quad L_i(\theta_i) = \sum_{j \in \Oi} \ell(\theta_i; x_{ij}, y_{ij}),
\label{eq:loss}
\end{equation}
where $\theta_i \in  \Rbb^{d}$ are the model parameters for task $i$, and $\theta$ is a shorthand for $(\theta_1, \dots, \theta_m)$. For all $i$, we will denote the non-private solution of task $i$ by
\[
\theta_i^* = \argmin_{\theta_i} L_i(\theta_i).
\]
The goal is to learn a private solution $\hat \theta$ under user-level DP. The quality of $\hat \theta$ will be measured in terms of the excess empirical risk,
\[
\sum_{i = 1}^m \Exp[L_i(\hat \theta_i)] - L_i(\theta_i^*),
\]
where the expectation is over randomness in the algorithm.

\begin{remark} Although the objective is decomposable as a sum of $L_i(\theta_i)$, the problems are coupled through the privacy constraint: since user $j$ contributes data to all tasks in $\Oj$, the \emph{combined sensitivity} of these tasks with respect to user $j$'s data should be bounded. One is then faced with the question of how to allocate the sensitivity among tasks. This will be achieved by using weights in the objective function.
\end{remark}

Our general approach will be to apply a DP algorithm to a weighted version of the objective $L(\theta)$. Let $w := (w_{ij})_{(i,j) \in \Omega}$ be a collection of arbitrary positive weights, and define the \emph{weighted multi-task objective}
\begin{equation}
\label{eq:obj_weighted}
L_w(\theta) = \sum_{i = 1}^m \sum_{j \in \Oi} w_{ij} \ell(\theta_i; x_{ij}, y_{ij}).
\end{equation}
We first give some intuition on the effect of the weights (a formal analysis will be given in the next section). Suppose $\nabla \ell$ is bounded, and that we compute $\hat \theta$ by applying noisy gradient descent (with fixed noise standard deviation) to $L_w(\theta)$. Then $w$ will have the following effect on privacy and utility: increasing the weights of task $i$ will generally improve its utility (since it increases the norm of task $i$'s gradients, hence increases its signal-to-noise ratio). Increasing the weights of user $j$ will decrease her privacy (since it increases the norm of user $j$'s gradients, hence increases the $L_2$ sensitivity w.r.t. her data), so guaranteeing a target privacy level will directly translate into constraints on user weights $(w_{ij})_{i \in \Oj}$. Utility depends on task weights, while sensitivity/privacy depend on user weights. By characterizing how sensitivity and utility precisely scale in terms of the weights $w$, we will be able to choose $w$ that achieves the optimal trade-off.


\begin{remark}
Although we will apply a DP algorithm to the weighted objective $L_w$, utility will always be measured in terms of the original, unweighted objective $L$. The weights should be thought of as a tool to allocate users' privacy budget, rather than a change in the objective function.
\end{remark}


\subsection{Weighted Multi-Task Ridge Regression}
\label{sec:linear_regression}

\begin{algorithm}[t!]
\caption{Task-weighted SSP (Sufficient Statistic Perturbation)\label{alg:ssp}}
\begin{algorithmic}[1]
\STATE{\bf Inputs}: Task-user graph $\Omega$, user features and labels $(x_{ij}, y_{ij})_{(i,j) \in \Omega}$, clipping parameters $\Gamma_x, \Gamma_*$, weights~$w$, domain~$\Ccal$.\\
\STATE For all $i,j$, $x_{ij} \leftarrow \clip(x_{ij}, \Gamma_x)$, $y_{ij} \leftarrow \clip(y_{ij}, \Gamma_x\Gamma_*)$, where $\clip(x, \Gamma) = x \min(1, \Gamma/\|x\|)$
\FOR{$1 \leq i \leq m$}
\STATE Sample $\xi_i \sim \Ncal^{d}, \quad \Xi_i \sim \Ncal^{d \times d}$\\
\STATE $\hat A_i \leftarrow \sum_{j \in \Oi} w_{ij} (x_{ij} x_{ij}^\top + \lambda I) + \Gamma_x^2 \Xi_i$\label{line:Ai}\\
\STATE $\hat b_i \leftarrow \sum_{j \in \Oi} w_{ij} y_{ij}x_{ij} + \Gamma_x^2\Gamma_* \xi_i$\label{line:bi}\\
\STATE $\hat \theta_i \leftarrow \hat A_i^\dagger\hat b_i$
\ENDFOR
\STATE {\bf Return} $\hat \theta$
\end{algorithmic}
\end{algorithm}

We start by analyzing the case of multi-task ridge regression, optimized using the Sufficient Statistics Perturbation (SSP) method~\citep{foulds2016ssp,wang2018revisiting}. This will pave the way to the more general convex setting in the next section.

Suppose that each task is a ridge regression problem. The multi-task objective~\eqref{eq:loss} becomes
\begin{equation}
\label{eq:quad_loss}
L(\theta) = \sum_{i = 1}^m \sum_{j \in \Oi} \left[(\theta_i \cdot x_{ij} - y_{ij})^2 + \lambda \|\theta_i\|^2\right],
\end{equation}
where $\lambda$ is a regularization constant. The exact solution of~\eqref{eq:quad_loss} is given by $\theta_i^* = A_i^{-1} b_i$ where
\begin{equation}
A_i = \sum_{j \in \Oi} (x_{ij}x_{ij}^\top + \lambda I), \quad b_i = \sum_{j \in \Oi}y_{ij}x_{ij}. \label{eq:Ai_bi}
\end{equation}
We propose a task-weighted version of SSP, described in Algorithm~\ref{alg:ssp}. The private solution $\hat \theta_i$ is obtained by forming \emph{weighted and noisy} estimates $\hat A_i, \hat b_i$ of the sufficient statistics~\eqref{eq:Ai_bi}, then returning $\hat \theta_i = \hat A_i^\dagger \hat b_i$.  Notice that if we use larger weights for task $i$, the \emph{relative noise scale} in $\hat A_i, \hat b_i$ becomes smaller (see Lines~\ref{line:Ai}-\ref{line:bi}), which improves the quality of the estimate, while also increasing the sensitivity. Once we analyze how the weights affect privacy and utility, we will be able to reason about an optimal choice of weights.

First, we state the privacy guarantee of Algorithm~\ref{alg:ssp}. All proofs are deferred to Appendix~\ref{app:proofs}.

\begin{theorem}[Privacy guarantee of Algorithm~\ref{alg:ssp}]
\label{thm:privacy-ssp}
Let $(w_{ij})_{(i,j) \in \Omega}$ be a collection of task-user weights, and let $\beta = \max_{j \in [n]} \sum_{i \in \Oj} w_{ij}^2$. Then Algorithm~\ref{alg:ssp} run with weights $w$ is $(\alpha, \alpha \beta)$-RDP for all $\alpha > 1$. 
\end{theorem}
The result can be easily translated to traditional DP, for example, for all $\epsilon, \delta > 0$ with $\epsilon \leq \log(1/\delta)$, if $\beta \leq \frac{\epsilon^2}{8\log(1/\delta)}$, then the algorithm is $(\epsilon, \delta)$-DP.

\begin{remark} The theorem can be interpreted as follows: each user $j$ has a total budget~$\beta$, that can be allocated among the tasks in $\Oj$ by choosing weights such that $\sum_{i \in \Oj} w_{ij}^2 \leq~\beta$. For this reason, we will refer to $\beta$ as the \emph{user RDP budget}.
\end{remark}
\begin{remark}
Sampling a fixed number of tasks per user is a special case of this formulation, where $w_{ij} = 1$ if task $i$ is sampled for user $j$, and $0$ otherwise. Using general weights allows a finer trade-off between tasks, as we shall see next.
\end{remark}

For the utility analysis, we will make the following standard assumption.

\begin{assumption}
\label{assump:SSP}
We assume that there exist $\Gamma_x, \Gamma_* > 0$ such that for all $i, j$, $\|x_{ij}\| \leq \Gamma_x$,  $\|\theta_i^*\| \leq \Gamma_*$, and $\|y_{ij}\| \leq \Gamma_x\Gamma_*$.
\end{assumption}
\begin{remark}
We state the result in the ridge regression case with bounded data to simplify the presentation. Similar results can be obtained under different assumptions, for instance when $x_{ij}$ are i.i.d. Gaussian, in which case one can bound the minimum eigenvalue of the covariance matrices $A_i$ without the need for regularization, see~\citep{wang2018revisiting} for a detailed discussion.
\end{remark}

\begin{theorem}[Utility guarantee of Algorithm~\ref{alg:ssp}]
\label{thm:tradeoff-ssp}
Suppose Assumption~\ref{assump:SSP} holds. Let $\omega \in \Rbb^m$ be a vector of positive weights such that for all~$j$, $\sum_{i \in \Oj} \omega_i^2 \leq \beta$. Let $\hat \theta$ be the output of Algorithm~\ref{alg:ssp} with weights $w_{ij} = \omega_i \ \forall (i,j) \in \Omega$. Then Algorithm~\ref{alg:ssp} is $(\alpha, \alpha\beta)$-RDP, for all $\alpha > 1$, and w.h.p.,
\begin{equation}
\label{eq:bound-ssp}
L(\hat \theta) - L(\theta^*) = \bigO{\frac{\Gamma_x^4\Gamma_*^2d}{\lambda} \sum_{i = 1}^m \frac{1}{\omega_i^2n_i}}.
\end{equation}
\end{theorem}
Theorem~\ref{thm:tradeoff-ssp} highlights that under the user-level privacy constraint ($\max_j \sum_{i \in \Oj} \omega_i^2 \leq \beta$), there is a certain trade-off between tasks: by increasing $\omega_i$ we improve quality for task $i$, but this may require decreasing $\omega$ for other tasks to preserve a constant total privacy budget. This allows us to reason about how to choose weights that achieve the optimal trade-off under a given RDP budget $\beta$. Remarkably, this problem is convex:
\[
\min_{\omega \in \Rbb_+^d} \sum_{i = 1}^m \frac{1}{\omega_i^2n_i} \quad \text{s.t. } \sum_{i \in \Oj} \omega_i^2 \leq \beta \quad \forall j \in \{1, \dots, n\}.
\]
If the task-user graph $\Omega$ were public, one could solve the problem exactly. But since task-user membership often represents private information, we want to protect $\Omega$. Our approach will be to compute an approximate solution. Before tackling this problem, we first derive similar weight-dependent privacy and utility bounds for noisy gradient descent.


\begin{algorithm}[t!]
\caption{Task-weighted noisy gradient descent\label{alg:dpsgd}}
\begin{algorithmic}[1]
\STATE{\bf Inputs}: Task-user graph $\Omega$, user features and labels $x, y$, domain $\Ccal$, clipping parameter $\Gamma$, weights~$w$, initial parameters $\hat \theta^{(0)}$, numer of steps $T$, learning rates $\eta_i^{(t)}$.\\
\FOR{$1 \leq i \leq m$}
\FOR{$1 \leq t \leq T$}
\STATE Sample $\xi^{(t)}_i \sim \Ncal^{d}$\\
\STATE $g^{(t)}_i \leftarrow \text{clip}(\sum_{j \in \Omega_i}  w_{ij}\nabla_{\theta_i} \ell(\theta_i^{(t-1)}; x_{ij}, y_{ij}), \Gamma)$ \label{line:gradient}\\
\STATE $\hat \theta_i^{(t)} \leftarrow \Pi_{\Ccal}[\theta_i^{(t-1)} - \eta^{(t)}_i (g^{(t)}_i + \Gamma\sqrt{T/2} \cdot \xi^{(t)}_i)]$
\ENDFOR
\ENDFOR
\STATE {\bf Return} $\hat \theta^{(T)} = (\hat \theta_1^{(T)}, \dots, \hat \theta_m^{(T)})$\label{line:multi_end}
\end{algorithmic}
\end{algorithm}

\subsection{Weighted Multi-Task Convex Minimization}
\label{sec:dpsgd}
We now consider the more general problem of minimizing empirical risk for convex losses in the multi-task setting. Our method applies noisy GD~\citep{bassily2014erm} to the weighted objective~\eqref{eq:obj_weighted}. This is summarized in Algorithm~\ref{alg:dpsgd}. We will consider two standard settings, in which noisy GD is known to achieve nearly optimal empirical risk bounds:
\begin{assumption}[Lipschitz convex on a bounded domain]
\label{assump:lip}
Assume that $\Ccal$ is a bounded convex domain, and that there exists $\Gamma > 0$ such that $\ell(\theta; x, y)$ is a convex and $\Gamma$-Lipschitz function of $\theta \in \Ccal$, uniformly in $x, y$.
\end{assumption}
\begin{assumption}[Lipschitz, strongly convex]
\label{assump:sc}
Assume that there exist $\lambda, \Gamma > 0$ such that $\ell(\theta; x, y)$ is a $\lambda$-strongly convex and $\Gamma$-Lipschitz function of $\theta$, uniformly in $x, y$.
\end{assumption}
Next, we derive privacy and utility bounds.
\begin{theorem}[Privacy guarantee of Algorithm~\ref{alg:dpsgd}]
\label{thm:privacy-dpsgd}
Let $(w_{ij})_{(i,j) \in \Omega}$ be a collection of task-user weights, and let $\beta = \max_{j \in [n]} \sum_{i \in \Oj} w_{ij}^2$. Then Algorithm~\ref{alg:dpsgd} is $(\alpha, \alpha \beta)$-RDP for all $\alpha > 1$.
\end{theorem}
\begin{theorem}[Utility guarantee of Algorithm~\ref{alg:dpsgd}]
\label{thm:tradeoff-dpsgd}
Let $\omega \in \Rbb^m$ be a vector of positive weights such that for all~$j$, $\sum_{i \in \Oj} \omega_i^2 \leq \beta$. Let $\hat \theta$ be the output of Algorithm~\ref{alg:dpsgd} with weights $w_{ij} = \omega_i$. Then Algorithm~\ref{alg:dpsgd} is $(\alpha, \alpha\beta)$-RDP, for all $\alpha > 1$. Furthermore,
\begin{enumerate}[(a)]
\item Under Assumption~\ref{assump:lip}, if
$T = \frac{2}{d}\frac{\sum_{i = 1}^m n_i^2}{\sum_{i = 1}^m 1/\omega_i^2}$
and $\eta_i^{(t)} = \frac{\|C\|}{\Gamma \sqrt{\omega_i^2n_i^2 + Td/2}}\frac{1}{\sqrt t}$,
then
\begin{equation}
\label{eq:bound-lip}
\Exp[L(\theta)] - L(\theta^*) = \bigOtilde{\|\Ccal\|\Gamma \sqrt{md}\sqrt{\sum_{i = 1}^m \frac{1}{\omega_i^2}}},
\end{equation}
where $\tilde\Ocal$ hides polylog factors in $n_i$.
\item Under Assumption~\ref{assump:sc}, if
$T = \frac{2}{d}\frac{|\Omega|}{\sum_{i = 1}^m 1/\omega_i^2n_i}$,
and $\eta_i^{(t)} = \frac{1}{\omega_i^2 n_i \lambda t}$, then
\begin{equation}
\label{eq:bound-sc}
\Exp[L(\theta)] - L(\theta^*) = \bigOtilde{\frac{\Gamma^2d}{\lambda}\sum_{i = 1}^m \frac{1}{\omega_i^2n_i}}.
\end{equation}
\end{enumerate}
\end{theorem}
We consider some special cases, as a consistency check.

\paragraph{Single task} Suppose there is a single task ($m = 1$), then the RDP budget constraint is simply satisfied with $\omega_1^2 = \beta$. Plugging this into the bounds from the theorem, and dividing by the number of examples $n$ (as our loss $L$ is a summation, instead of average, over all examples), we get that the average empirical risk is bounded by $\bigOtilde{\frac{\|\Ccal\|\Gamma\sqrt{d}}{n\sqrt\beta}}$ and $\bigOtilde{\frac{\Gamma^2d}{\lambda n^2 \beta}}$ in the convex and strongly convex case, respectively. This recovers the usual convex ERM bounds~\citep{bassily2014erm} -- with the correspondence $\beta = \bigO{\epsilon^2/\log(1/\delta)}$.

\paragraph{Complete bipartite $\Omega$} Suppose there are $m$ tasks and each user participates in \emph{all} tasks. Then the budget constraint becomes $m\omega_i^2 = \beta$ for all $i$. Plugging this into the bounds and dividing by $mn$ (total number of examples), we get that the average empirical risk is $\bigOtilde{\frac{\|\Ccal\|\Gamma \sqrt{md}}{n\sqrt\beta}}$ and $\bigOtilde{\frac{\Gamma^2md}{\lambda n^2 \beta}}$  in the convex and strongly convex case, respectively. Notice that this is equivalent to solving a single task in dimension~$md$, and the resulting ERM bound is the same.


\section{Optimizing Weights Under Distribution Skew}
\label{sec:optimal}
Equipped with the utility bounds of the previous section, we can ask what choice of weights minimizes the error bound under a given privacy budget. Observe that all utility bounds~\eqref{eq:bound-ssp},~\eqref{eq:bound-lip},~\eqref{eq:bound-sc} are increasing in the quantity $\sum_{i = 1}^m \frac{1}{\omega_i^2n_i^\gamma}$, where $\gamma = 0$ in the convex case, and $\gamma=1$ in the strongly convex case (for both SSP and noisy GD). This amounts to solving a constrained optimization problem of the form
\begin{equation}
\label{eq:opt_exact}
\min_{\omega \in \Rbb^m_+} \sum_{i = 1}^m \frac{1}{\omega_i^2n_i^\gamma} \quad \text{s.t. } \sum_{i \in \Oj} \omega_i^2 \leq \beta, \quad \forall j \in \{1, \dots, n\}.
\end{equation}

In general, the task-user graph $\Omega$ is private (for example, in the recommendation setting, $\Omega$ represents \emph{which} pieces of content each user interacted with, while the labels $y_{ij}$ represent \emph{how much} a user liked that piece of content; we want to protect both). We cannot solve~\eqref{eq:opt_exact} exactly, but under distributional assumptions on $\Omega$, we can solve it approximately.

In modeling the distribution of $\Omega$, we take inspiration from the closely related problem of matrix completion, where a standard assumption is that elements of $\Omega$ are sampled uniformly at random, see e.g.~\citep{jain2013lowrank} -- in other words, $\Omega$ is an Erdős-R\'enyi random graph. However, making such assumption implies that the task sizes $n_i = |\Omega_i|$ concentrate around their mean, and this fails to capture the task-skew problem we set out to solve. Instead, we will relax the assumption by removing uniformity of tasks while keeping uniformity of users. Specifically,

\begin{assumption}
\label{assump:omega}
Assume that $\Omega$ is obtained by sampling, for each task~$i$, $n_i$ users independently and uniformly from $\{1, \dots, n\}$. Furthermore, assume that the task sizes $n_i$ are publicly known, and that $n_i \geq 1$ for all $i$.
\end{assumption}

In practice, both users and tasks can be heterogeneous. In that sense, Assumption~\ref{assump:omega} is a stylized model, that is meant to primarily capture the task skew problem. We leave extensions to the user heterogeneous case to future work.

The assumption that $n_i$ are publicly known is relatively mild: in many applications, this information is available, for example in video recommendation, the number of views of each video is often public. A similar assumption was made by~\cite{epasto2020smoothly}, where the number of examples contributed by each user is assumed to be public. When $n_i$ are not available, the solution can be based on private estimates of $n_i$; we give in Appendix~\ref{app:private_counts} the extension to this case. In experiments, we will always privately estimate $n_i$ and account for the privacy cost of doing so.

A consequence of Assumption~\ref{assump:omega} is that the quantity $B_j := \sum_{i \in \Oj} \omega_i^2$ (that appears in the inequality constraint in equation~\eqref{eq:opt_exact}) becomes a random variable, that is concentrated around its mean, which is equal to $\sum_{i = 1}^m \frac{n_i}{n} \omega_i^2$. More precisely, we will show that, w.h.p., all $B_j$ are bounded above by $c(n) \cdot \sum_{i = 1}^m \frac{n_i}{n} \omega_i^2$ for an appropriate $c(n)$. Then we can replace the $n$ constraints in~\eqref{eq:opt_exact} by a single constraint (for the privacy guarantee, the constraints will be enforced with probability 1, as we will discuss below; but for the purpose of optimizing the utility bound, it suffices that the constraints hold w.h.p.). The problem becomes:
\begin{equation}
\label{eq:opt}
\min_{\omega \in \Rbb^m_+} \sum_{i = 1}^m \frac{1}{\omega_i^2n_i^\gamma} \quad \text{s.t. } c(n)\sum_{i = 1}^m \frac{n_i}{n}\omega_i^2 \leq \beta,
\end{equation}
which we can solve in closed form as follows. Let $\bar \beta = \frac{\beta}{c(n)}$. The Lagrangian is
$\sum_{i = 1}^m \frac{1}{\omega_i^2n_i^\gamma} + \lambda (\sum_{i=1}^m n_i\omega_i^2 - n\bar\beta)$ ($\lambda$ is a Lagrange multiplier), and the KKT conditions yield: $\forall i, \ -\frac{1}{\omega_i^3 n_i^\gamma} + \lambda n_i \omega_i = 0, \ 
\sum_{i = 1}^m n_i\omega_i^2 = n\bar\beta$, which simplifies to
{\small%
\begin{equation}
\label{eq:w}
\omega_i^* = \frac{n_i^{-(\gamma+1)/4}}{\sqrt{\sum_{i' = 1}^m n_{i'}^{(1-\gamma)/2}/n\bar \beta}}\,.
\end{equation}}%
With $\omega_i$ in this form and $c(n) = c \log n$, it can be shown (see Lemma~\ref{lem:concentration}) that w.h.p., $\max_j \sum_{i \in \Oj} \omega_i^2 \leq \beta$. But since we need the privacy guarantee to always hold, we simply add the following clipping step: let $w^*$ be the task-user weights defined as
{\small%
\begin{equation}
\label{eq:w_final}
w^*_{ij} = \omega^*_i \min\left(1, \sqrt{\frac{\beta}{\sum_{i' \in \Omega_j} {\omega_{i'}^*}^2}} \right) \quad \forall i, j \in \Omega.
\end{equation}}%
In other words, for each user $j$, if the total sensitivity $\sum_{i'\in \Omega_j}{\omega_{i'}^*}^2$ exceeds $\beta$, we scale down the weights of this user to meet the budget. This step ensures that the privacy guarantee is exact. But for utility analysis, we know that w.h.p., $w^*_{ij} = \omega_i$ (since the constraint is satisfied w.h.p.), so utility bounds will hold w.h.p.

The next theorem provides the utility guarantee under weights $w^*$.

\begin{theorem}[Privacy-utility trade-off of Algorithms~\ref{alg:ssp} and~\ref{alg:dpsgd} under adaptive weights]
\label{thm:exact}
Suppose Assumption~\ref{assump:omega} holds.
\begin{enumerate}[(a)]
\item Under Assumption~\ref{assump:SSP}, let $\hat\theta$ be the output of Algorithm~\ref{alg:ssp} run with weights $w^*$ (eq.~\eqref{eq:w}-\eqref{eq:w_final} with $\gamma = 1$). Then w.h.p.,
\begin{equation}
\label{eq:final_bound_ssp}
L(\hat \theta) - L(\theta^*) = \bigOtilde{\frac{\Gamma_x^4\Gamma_*^2dm^2}{n\lambda\beta}}.
\end{equation}
\item Under Assumption~\ref{assump:lip}, if Algorithm~\ref{alg:dpsgd} is run with weights $w^*$ (eq.~\eqref{eq:w}-\eqref{eq:w_final} with $\gamma = 0$) and parameters listed in Theorem~\ref{thm:tradeoff-dpsgd}(a), then
\begin{equation}
\label{eq:final_bound_lip}
\tilde\Exp[L(\theta)] - L(\theta^*) = \bigOtilde{\|\Ccal\|\Gamma \sqrt{\frac{md}{n\beta}}\sum_{i = 1}^m n_i^{1/2}},
\end{equation}
where $\tilde\Exp$ denotes expectation conditioned on a high probability event.
\item Under Assumption~\ref{assump:sc}, if Algorithm~\ref{alg:dpsgd} is run with $w^*$ (eq.~\eqref{eq:w}-\eqref{eq:w_final} with $\gamma = 1$) and parameters listed in Theorem~\ref{thm:tradeoff-dpsgd}(b), then
\begin{equation}
\label{eq:final_bound_sc}
\tilde\Exp[L(\theta)] - L(\theta^*) = \bigOtilde{
\frac{\Gamma^2dm^2}{n\lambda\beta}
}.
\end{equation}
\end{enumerate}
In all three cases, the procedure is $(\alpha, \alpha\beta)$-RDP for all $\alpha > 1$.
\end{theorem}
To understand the effect of adapting to task skew, we compare these bounds to the case when we use uniform weights. With uniform weights, the privacy constraint yields $\omega^{\text{uniform}}_i = (n\bar\beta/\sum_{i' = 1}^m n_{i'})^{1/2}, \quad \forall i$. Taking the ratio between the utility bound for $\omega^{\text{uniform}}$ and the utility bound for $\omega^*$, we obtain the following: In the Lipschitz bounded case, the relative improvement is $R_0(n_1, \dots, n_m) = \frac{(m\sum_{i = 1}^m n_i)^{1/2}}{\sum_{i = 1}^m n_i^{1/2}}$. In the strongly convex case (both for SSP and noisy GD), the relative improvement is $R_1(n_1, \dots, n_m) = \frac{\sum_{i = 1}^m \frac{1}{n_i}\sum_{i = 1}^m n_i}{m^2}$. In particular, $R_0, R_1$ are lower bounded by $1$ (by Cauchy-Schwarz), and equal to $1$ when the task distribution is uniform (all $n_i$ are equal). In both cases, the improvement can be arbitrary large. To see this, consider the extreme case when $n_1 = m^{1+\nu}$ and $n_i$ is a constant for all other $m-1$ tasks ($\nu > 0$). A simple calculation shows that in this case, the leading term as $m \to \infty$ is $R_0 \approx m^{\nu/2}$ and $R_1 \approx m^{\nu}$. Both are unbounded in $m$.

Finally, we give a qualitative comment on the optimal weights~\eqref{eq:w}. While it's clear that increasing the weight of a task improves its quality, it was not clear, a priori, that increasing weights on tail tasks benefits the \emph{overall objective}~\eqref{eq:obj} (since tail tasks represent fewer terms in the sum). The analysis shows that this is indeed the case: the optimal trade-off (see eq.~\eqref{eq:w}) is obtained when $\omega_i \propto n_i^{-(1+\gamma)/4}$, i.e. larger weights are assigned to smaller tasks. This can be explained by a certain diminishing returns effect that depends on the task size: from the optimization problem~\eqref{eq:opt}, the marginal utility benefit of increasing $\omega_i$ is smaller for larger tasks (due to the $n_i^\gamma$ term in the objective). At the same time, the marginal privacy cost of increasing $\omega_i$ is higher for larger tasks (due to the $n_i$ term in the constraint).

\section{Empirical Evaluation}
To evaluate our methods, we run experiments on large-scale recommendation benchmarks on the MovieLens data sets~\citep{harper16movielens}. As mentioned in the introduction, recommendation problems are known to exhibit a long tail of content with much fewer training data than average, and this is the main practical issue we seek to address. We also run experiments on synthetic data, reported in Appendix~\ref{app:synthetic_expt}.

In our evaluation, we investigate the following questions:
\begin{enumerate}[leftmargin=13pt,topsep=0pt,itemsep=0ex,partopsep=0ex,parsep=1pt]
\item Whether on realistic data sets, our adaptive weight methods can improve the overall utility by shifting weights towards the tail, as the analysis indicates. In particular, Theorem~\ref{thm:exact} suggests that (under distributional assumptions on $\Omega$), the optimal task weights are of the form $\omega_i \propto n_i^{-(1+\gamma)/4}$. We experiment with weights of the form $\omega_i \propto n_i^{-\mu}$ for different values of $\mu \in [0, 1]$.
\item Do we observe similar improvements for different algorithms? Our analysis applies both to SSP (Algorithm~\ref{alg:ssp}) and noisy GD (Algorithm~\ref{alg:dpsgd}). We run the experiments with both algorithms.
\item Beyond improvements in the average metrics, what is the extent of quality impact across the head and tail of the task distribution?
\end{enumerate}

\paragraph{Experimental Setup}
Each of the MovieLens data sets consists of a sparse matrix of ratings given by users to movies. The first benchmark, from~\citet{lee13llorma}, is a rating prediction task on MovieLens 10 Million (abbreviated as ML10M), where the quality is evaluated using the RMSE of the predicted ratings. The second benchmark, from~\cite{liang18vae}, is a top-$k$ item recommendation task on MovieLens 20 Million (abbreviated as ML20M), where the model is used to recommend $k$ movies to each user, and the quality is measured using Recall@k. Figure~\ref{fig:app_distribution_skew} in the appendix shows the movie distribution skew in each data set. For example, in ML20M, the top 10\% movies account for 86\% of the training data.

In both cases, we train a DP matrix factorization model. We identify movies to tasks and apply our algorithms to learning the movie embedding representations, for details, see Appendix~\ref{app:movielens_expt}.

The current state of the art on these benchmarks is the DP alternating least squares (DPALS) method~\citep{DPALS}, which we use as a baseline. We will compare to DPALS both with uniform sampling (sample a fixed number of movies per user, uniformly at random) and with tail-biased sampling (sort all movies by increasing counts, then for each user, keep the $k$ first movies). The latter was a heuristic specifically designed by~\cite{DPALS} to address the movie distribution skew, and is the current SOTA.

We also experiment with DPSGD on the same benchmarks, and compare uniform sampling, tail-biased sampling, and adaptive weights.

When computing adaptive weights for our methods, we first compute private estimates $\hat n_i$ of the counts (which we include in the privacy accounting), then define task weights following eq.~\eqref{eq:w}-\eqref{eq:w_final}, but allowing a wider range of exponents. More precisely, we replace eq.~\eqref{eq:w} with
\[
\omega^*_i = \frac{\hat n_i^{-\mu}}{\sqrt{\sum_{i' = 1}^m \hat n_{i'}^{-2\mu+1}/n\bar \beta}},
\]
where $\mu$ is a hyper-parameter. In the analysis, the optimal choice of weights corresponds to $\mu = 1/4$ in the convex case (eq.~\eqref{eq:w} with $\gamma = 0$), and $\mu = 1/2$ in strongly convex case (eq.\eqref{eq:w} with $\gamma = 1$). We experiment with different values of $\mu \in [0, 1]$ to evaluate the effect of shifting the weight distribution towards the tail.

\def\wdth{0.4\textwidth}
\def\figsep{0.1\textwidth}
\begin{figure}[h]
\centering
\includegraphics[width=\wdth]{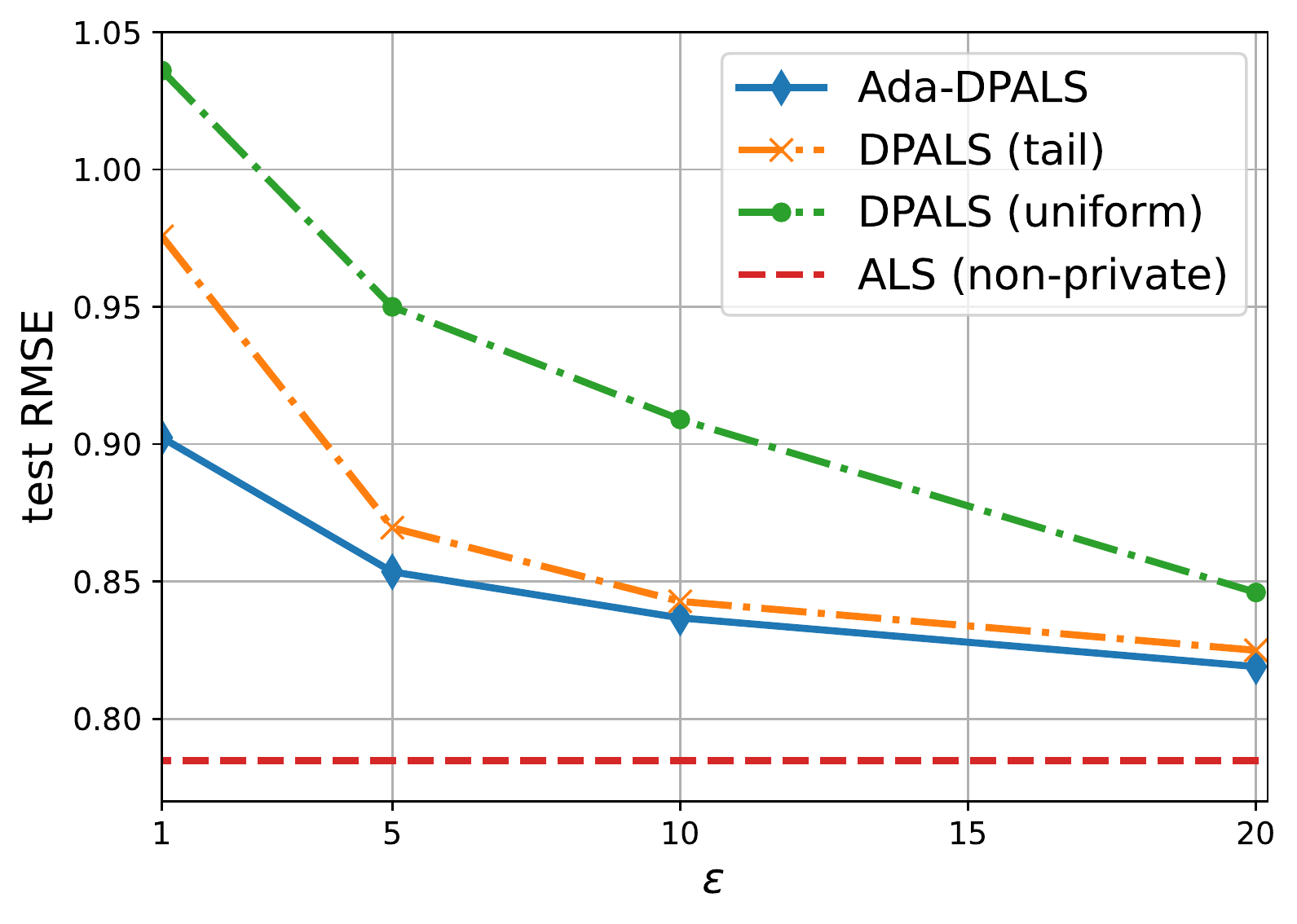}
\hspace{\figsep}
\includegraphics[width=\wdth]{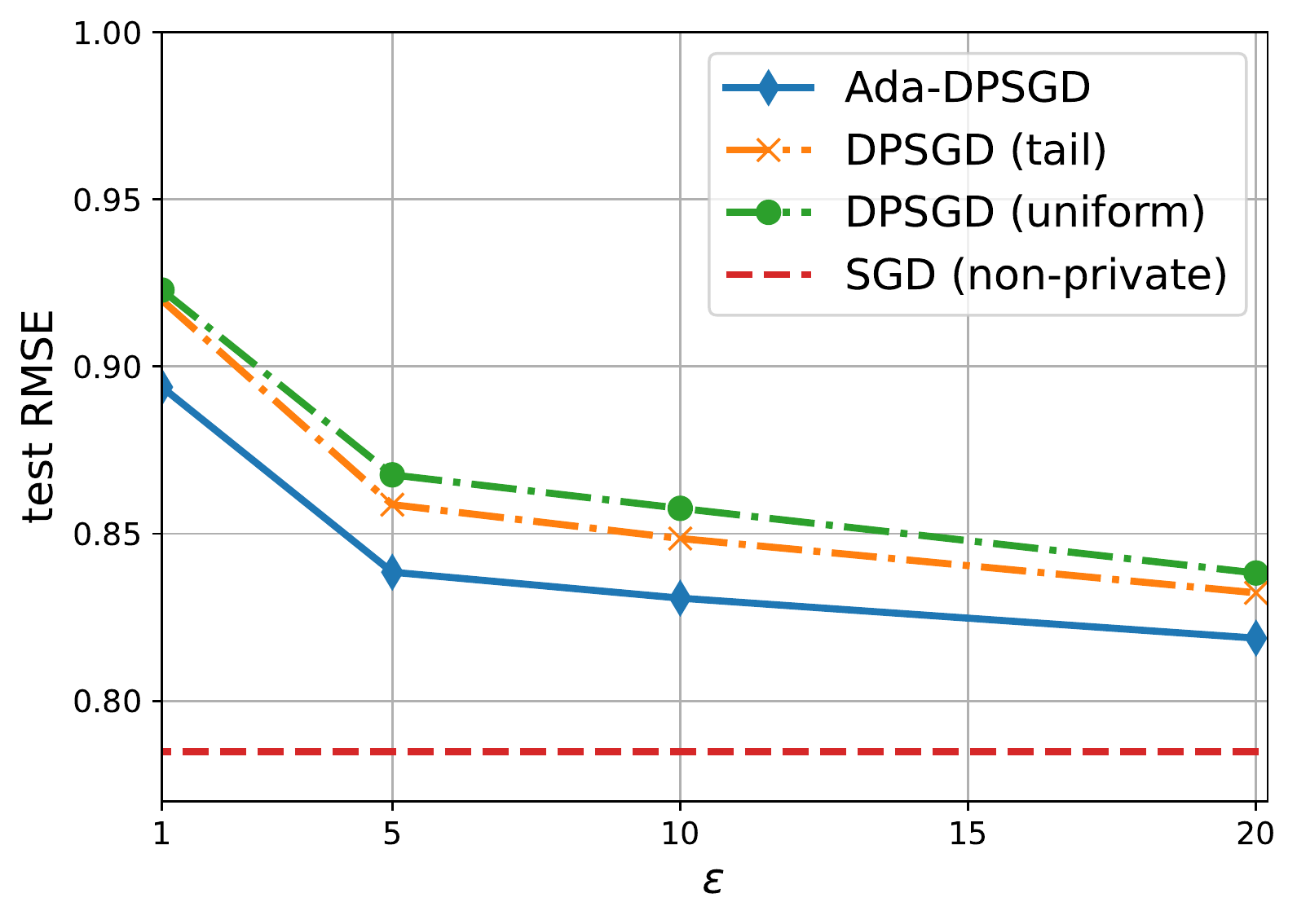}
\caption{Privacy-utility trade-off on ML10M, using uniform sampling, tail-biased sampling, and adaptive weights (our method), applied to DPALS and DPSGD. The utility is measured using RMSE (lower is better).}
\label{fig:trade-off}
\end{figure}
\begin{figure}[h!]
\centering
\includegraphics[width=\wdth]{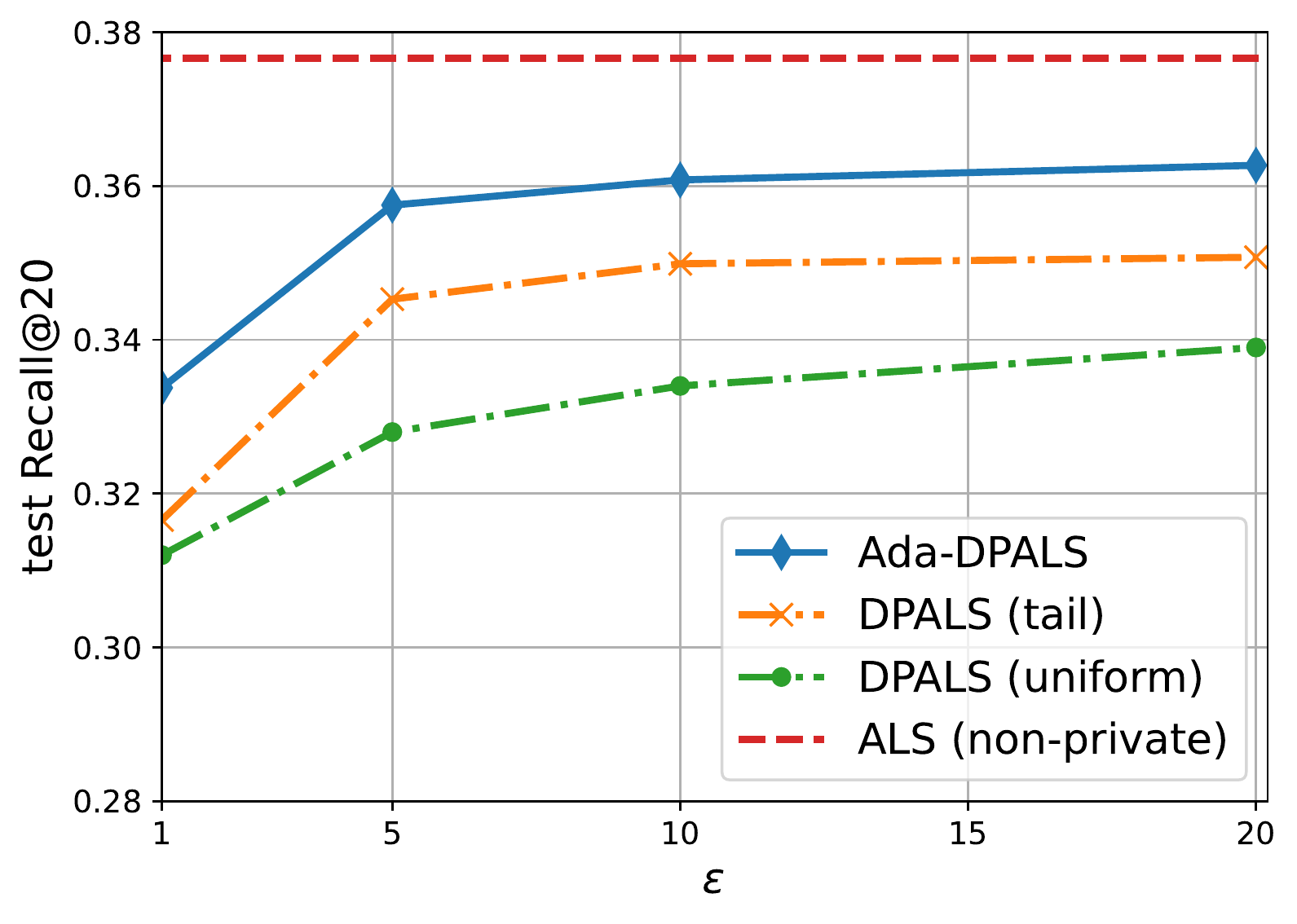}
\hspace{\figsep}
\includegraphics[width=\wdth]{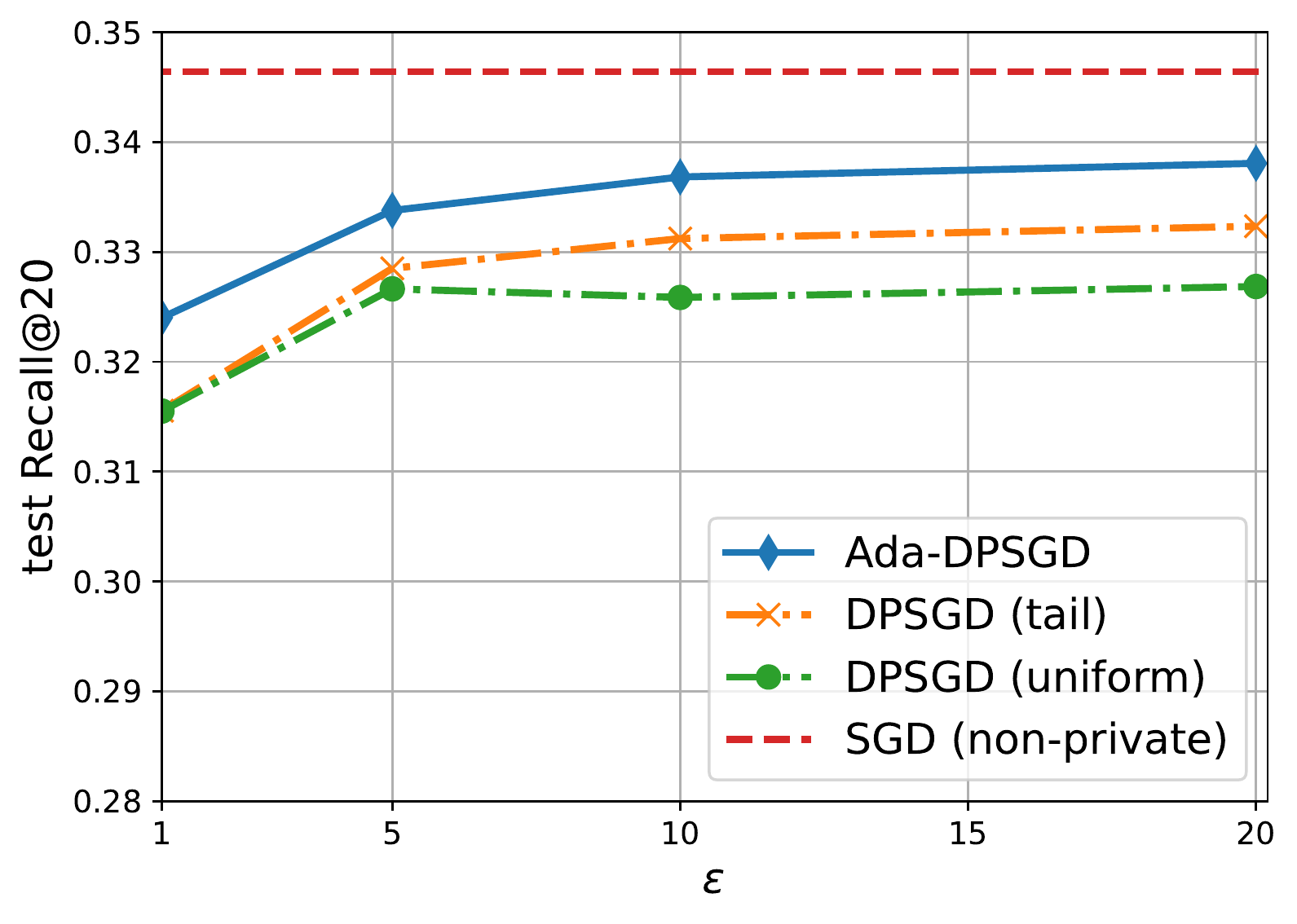}
\caption{Privacy-utility trade-off on ML20M, using uniform sampling, tail-biased sampling, and adaptive weights (our method), applied to DPALS and DPSGD. The utility is measured using Recall@20 (higher is better).}
\label{fig:trade-off2}
\end{figure}
\begin{figure}[h!]
\centering
\includegraphics[width=\wdth]{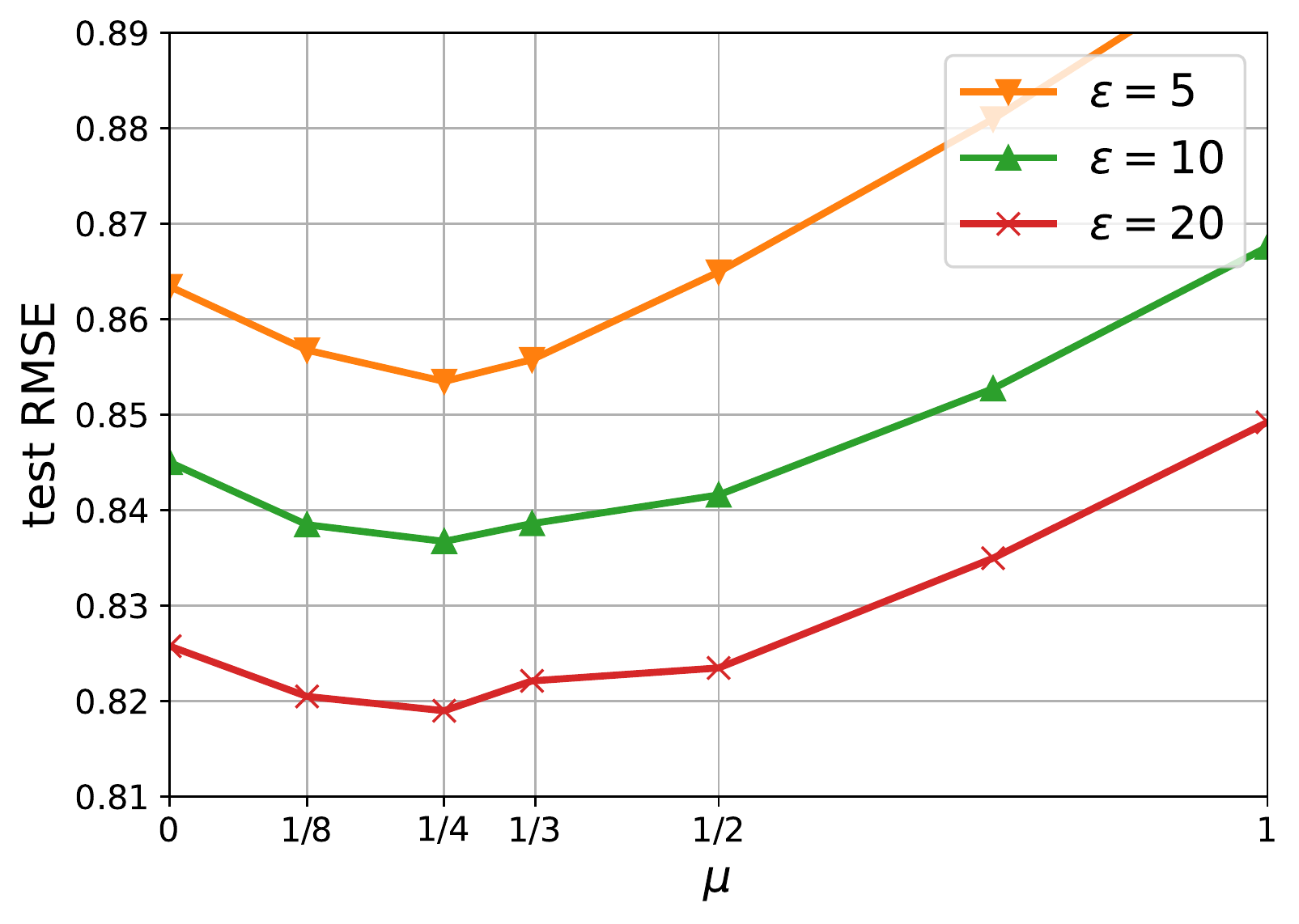}
\hspace{.1\textwidth}
\includegraphics[width=\wdth]{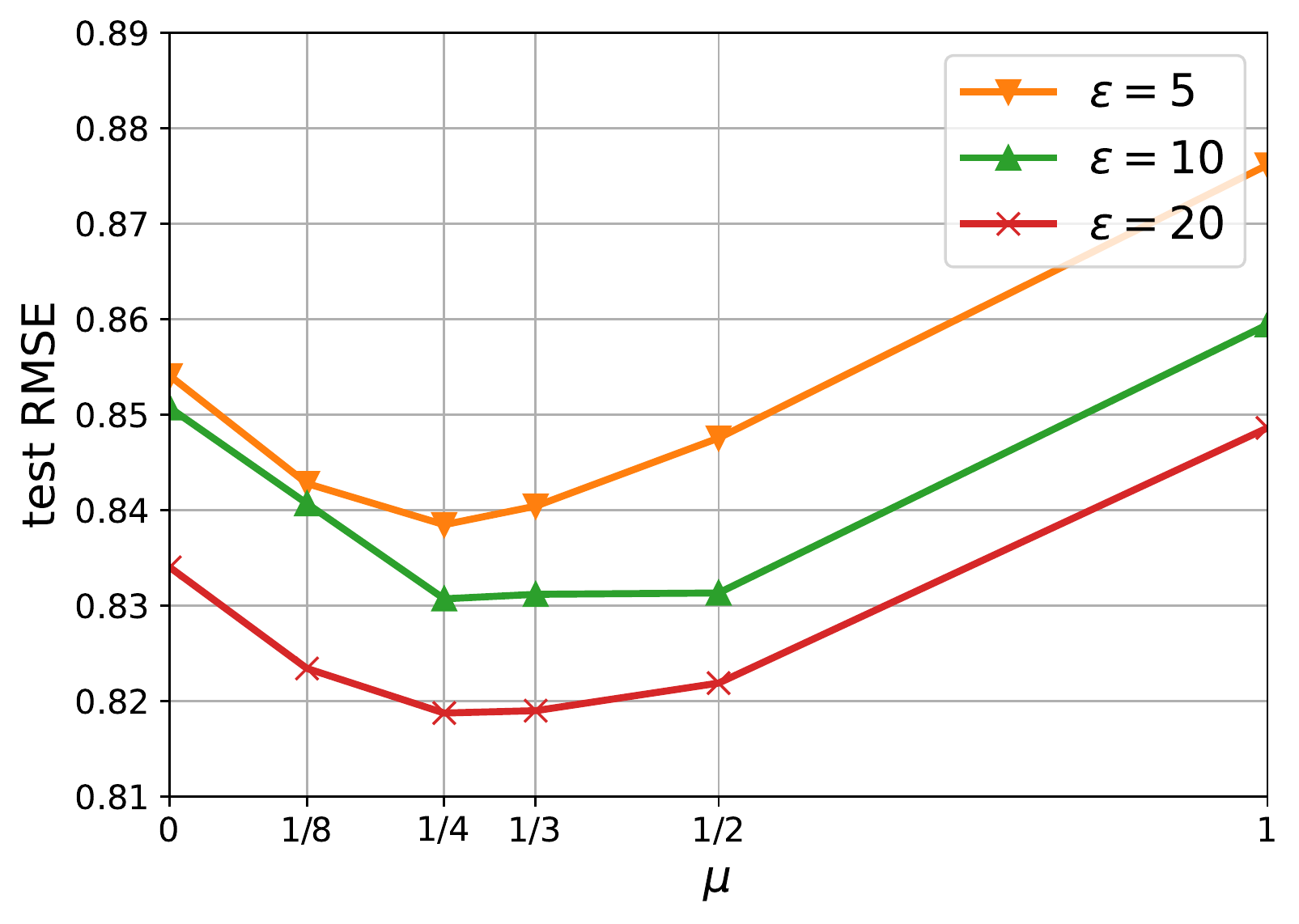}
\caption{Effect of the weight exponent $\mu$ on the ML10M benchmark, using DPALS and DPSGD.}
\label{fig:mu_ml10m}
\vspace{.2in}
\end{figure}

\paragraph{Effect of adaptive weights on privacy-utility trade-off}

We first evaluate the privacy/utility trade-off. The results are reported in Figure~\ref{fig:trade-off} (for ML10M) and Figure~\ref{fig:trade-off2} (for ML20M), where we also include non-private baselines for reference. Our adaptive weights methods achieve the state of the art on both benchmarks. We see improvements when applying adaptive weights both to DPSGD and DPALS. The extent of improvement varies depending on the benchmark. In particular, the improvement is remarkable on the ML20M benchmark across all values of $\epsilon$; using adaptive weights significantly narrows the gap between the previous SOTA and the non-private baseline. 

To illustrate the effect of the exponent $\mu$, we report, in Figure~\ref{fig:mu_ml10m}, the performance on ML10M for different values of $\mu$. The best performance is typically achieved when $\mu$ is between $\frac{1}{4}$ and $\frac{1}{2}$, a range consistent with the analysis. For larger values of $\mu$ (for example $\mu = 1$), there is a clear degradation of quality, likely due to assigning too little weight to head tasks. The trend is similar on ML20M, see Figure~\ref{fig:app_tradeoff_mu} in the appendix. 

\begin{figure}[h!]
\centering
\begin{subfigure}[b]{\wdth}
\centering
\includegraphics[width=\textwidth]{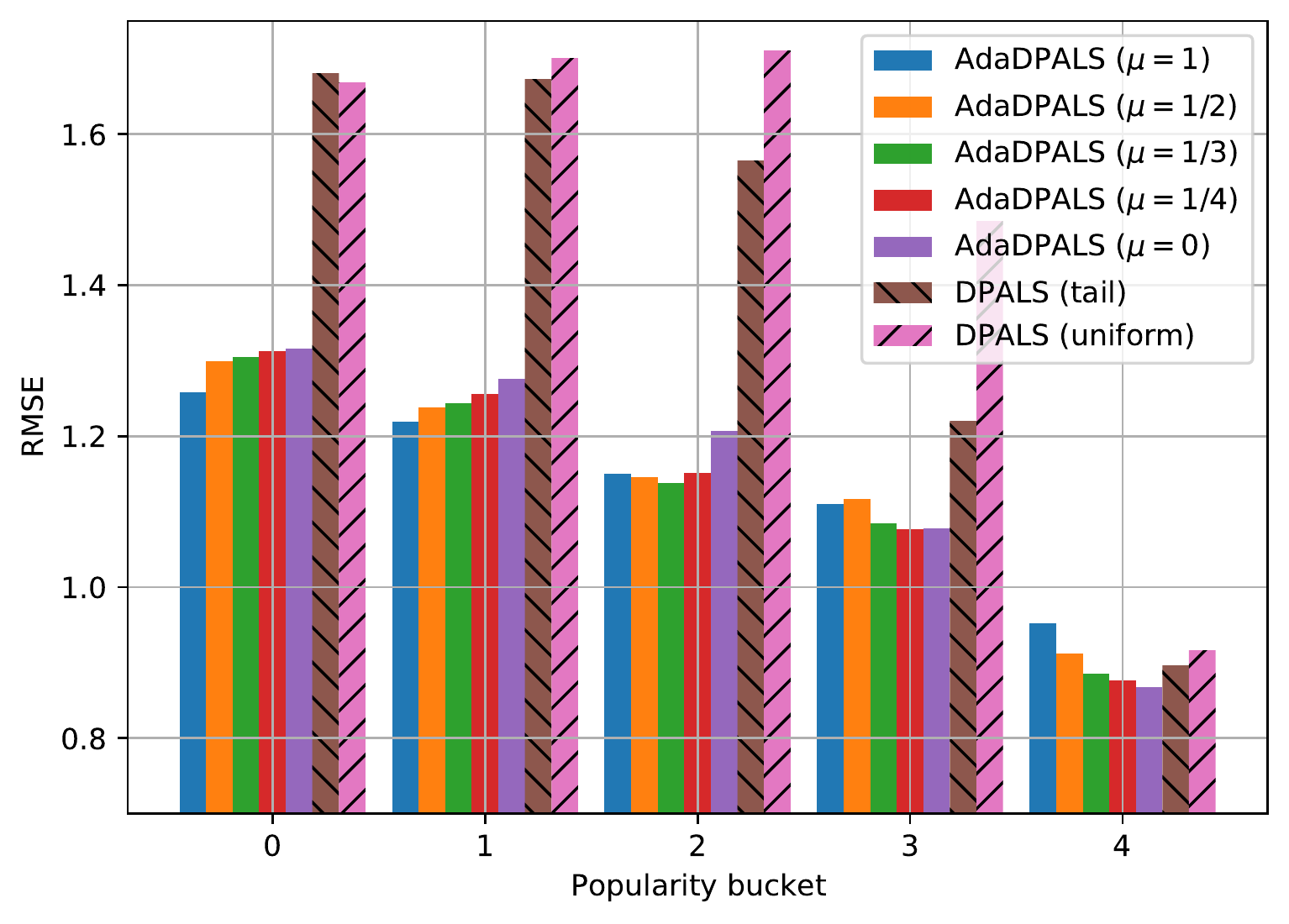}
\caption{RMSE on ML10M ($\epsilon = 1$).}
\end{subfigure}
\hspace{\figsep}
\begin{subfigure}[b]{\wdth}
\centering
\includegraphics[width=\textwidth]{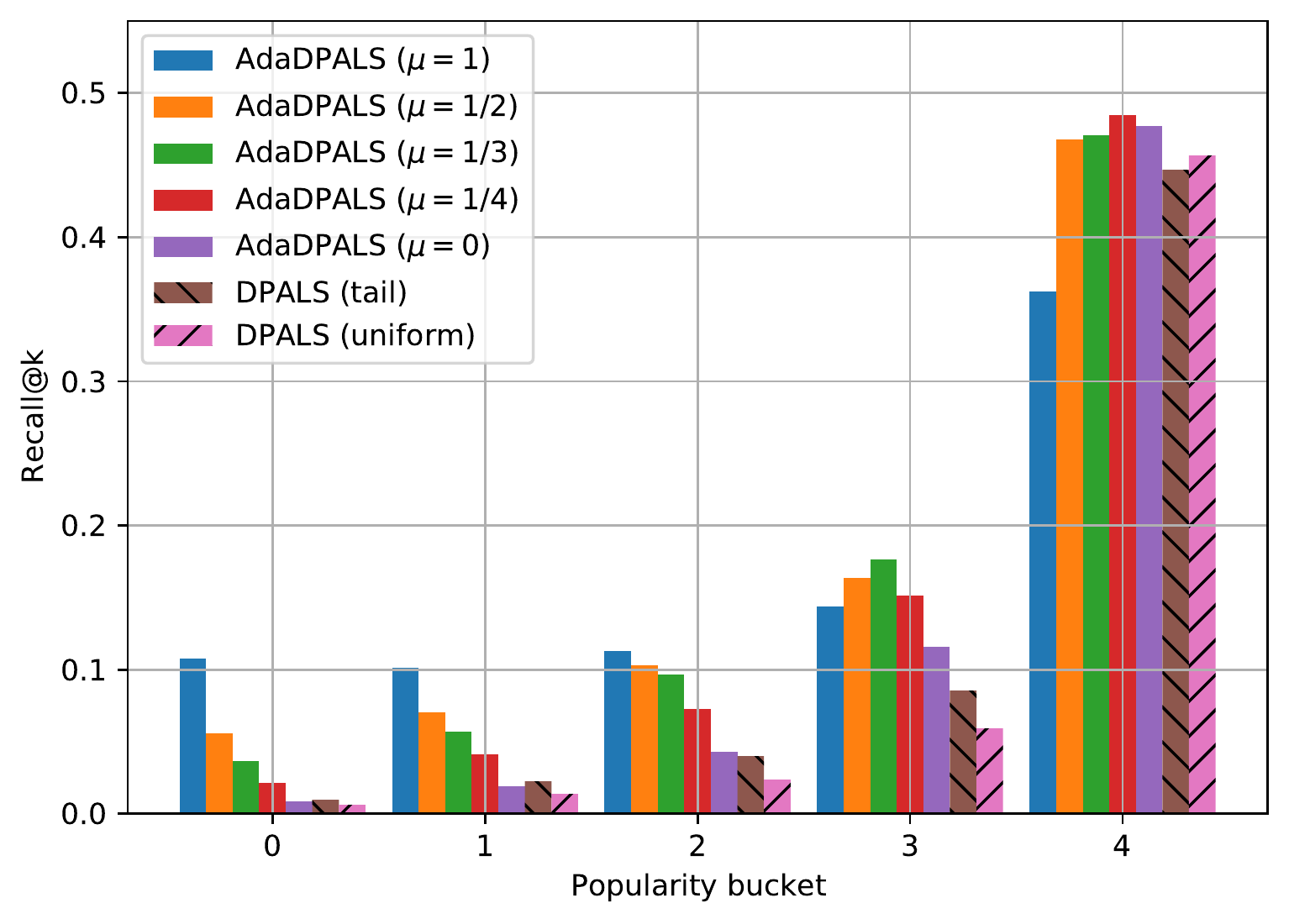}
\caption{Recall on ML20M ($\epsilon = 1$).}
\end{subfigure}
\caption{RMSE and Recall metrics, sliced by movie frequency. Each bucket contains an equal number of movies. Buckets are ordered by increasing movie counts.}
\label{fig:sliced}
\end{figure}

\paragraph{Impact on head and tail tasks}
To better understand the extent of quality impact on head/tail movies, we report the same metrics, sliced by movie counts. We sort the movies $i$ by increasing counts $n_i$, then group them into equally sized buckets, and compute average metrics\footnote{Since recall is naturally lower on tail tasks, we report Recall@k with larger $k$ for tail buckets: we use $k = 20$ on the most popular bucket, $40$ on the second, $60$ on the third, and so on. This allows for more readable plots.} on each bucket. The results are reported in Figure~\ref{fig:sliced} for $\epsilon = 1$, and similar plots are provided in Appendix~\ref{app:movielens_expt} for other values of $\epsilon$.

The results show a clear trade-off between head and tail tasks. On buckets $0$ and $1$ (tail), it is generally the case that the larger $\mu$ is, the better the quality, while on bucket $4$, the opposite trend can be observed. This is consistent with the intuition that as $\mu$ increases, more budget is assigned to the tail, which tends to shift the trade-off in favor of the tail.

Even compared to the previous SOTA (tail-sampling), the improvements on all but the top bucket are significant. Although the tail sampling heuristic was designed to improve tail quality, the experiments indicate that adaptive weights are much more effective. Furthermore, the parameter $\mu$ allows more control over this trade-off.

Finally, to give a concrete illustration of these improvements, we inspect the quality of recommendations on a few sample queries, reported in Appendix~\ref{app:movielens_expt}. We find that there is a visible improvement on tail recommendations in models trained using our method.


\section{Conclusion}
To address the practical problem of task distribution skew, we formally analyze the question of budget allocation among tasks under user-level DP. Our method is based on computing weights that adapt to the task distribution skew. We quantify utility improvements under optimal weights, in a range of settings. Importantly, the method achieves significant improvements on benchmarks, and allows finer control over quality trade-offs between head and tail tasks.

To compute optimal weights, our analysis relied on certain distributional assumptions on the task-user graph $\Omega$, and although this allowed us to model the task distribution skew, relaxing these assumptions is a promising direction. In particular, it may be possible to directly compute a privacy-preserving solution of problem~\eqref{eq:opt_exact} in its general form, using techniques similar to the constraint-private LPs studied by~\cite{hsu2014LP}.





\bibliography{ref}

\begin{thebibliography}{44}
\providecommand{\natexlab}[1]{#1}
\providecommand{\url}[1]{\texttt{#1}}
\expandafter\ifx\csname urlstyle\endcsname\relax
  \providecommand{\doi}[1]{doi: #1}\else
  \providecommand{\doi}{doi: \begingroup \urlstyle{rm}\Url}\fi

\bibitem[Abadi et~al.(2016)Abadi, Chu, Goodfellow, McMahan, Mironov, Talwar,
  and Zhang]{abadi2016dpsgd}
M.~Abadi, A.~Chu, I.~Goodfellow, H.~B. McMahan, I.~Mironov, K.~Talwar, and
  L.~Zhang.
\newblock Deep learning with differential privacy.
\newblock In \emph{Proceedings of the 2016 ACM SIGSAC Conference on Computer
  and Communications Security}, CCS '16, page 308–318, New York, NY, USA,
  2016. Association for Computing Machinery.

\bibitem[Abowd(2018)]{abowd2018census}
J.~M. Abowd.
\newblock The u.s. census bureau adopts differential privacy.
\newblock In \emph{Proceedings of the 24th ACM SIGKDD International Conference
  on Knowledge Discovery \& Data Mining}, KDD '18, page 2867, New York, NY,
  USA, 2018. Association for Computing Machinery.

\bibitem[Amin et~al.(2019)Amin, Kulesza, Munoz, and
  Vassilvtiskii]{amin2019bounding}
K.~Amin, A.~Kulesza, A.~Munoz, and S.~Vassilvtiskii.
\newblock Bounding user contributions: A bias-variance trade-off in
  differential privacy.
\newblock In K.~Chaudhuri and R.~Salakhutdinov, editors, \emph{Proceedings of
  the 36th International Conference on Machine Learning}, volume~97 of
  \emph{Proceedings of Machine Learning Research}, pages 263--271. PMLR, 09--15
  Jun 2019.

\bibitem[Amin et~al.(2022)Amin, Gillenwater, Joseph, Kulesza, and
  Vassilvitskii]{plume}
K.~Amin, J.~Gillenwater, M.~Joseph, A.~Kulesza, and S.~Vassilvitskii.
\newblock Plume: Differential privacy at scale.
\newblock \emph{CoRR}, abs/2201.11603, 2022.

\bibitem[Bagdasaryan et~al.(2019)Bagdasaryan, Poursaeed, and
  Shmatikov]{bagdasaryan2019}
E.~Bagdasaryan, O.~Poursaeed, and V.~Shmatikov.
\newblock Differential privacy has disparate impact on model accuracy.
\newblock In \emph{Advances in Neural Information Processing Systems},
  volume~32, 2019.

\bibitem[Bassily et~al.(2014)Bassily, Smith, and Thakurta]{bassily2014erm}
R.~Bassily, A.~Smith, and A.~Thakurta.
\newblock Private empirical risk minimization: Efficient algorithms and tight
  error bounds.
\newblock In \emph{2014 IEEE 55th Annual Symposium on Foundations of Computer
  Science}, pages 464--473, 2014.

\bibitem[Bertin-Mahieux et~al.(2011)Bertin-Mahieux, Ellis, Whitman, and
  Lamere]{bertinmahieux2011msd}
T.~Bertin-Mahieux, D.~P.~W. Ellis, B.~Whitman, and P.~Lamere.
\newblock The million song dataset.
\newblock In \emph{In Proceedings of the 12th International Conference on Music
  Information Retrieval (ISMIR)}, 2011.

\bibitem[Boucheron et~al.(2013)Boucheron, Lugosi, and
  Massart]{boucheron2013concentration}
S.~Boucheron, G.~Lugosi, and P.~Massart.
\newblock \emph{Concentration inequalities: A nonasymptotic theory of
  independence}.
\newblock Oxford university press, 2013.

\bibitem[Calandrino et~al.(2011)Calandrino, Kilzer, Narayanan, Felten, and
  Shmatikov]{calandrino2011you}
J.~A. Calandrino, A.~Kilzer, A.~Narayanan, E.~W. Felten, and V.~Shmatikov.
\newblock ``you might also like:'' privacy risks of collaborative filtering.
\newblock In \emph{2011 IEEE symposium on security and privacy}, pages
  231--246. IEEE, 2011.

\bibitem[Chien et~al.(2021)Chien, Jain, Krichene, Rendle, Song, Thakurta, and
  Zhang]{DPALS}
S.~Chien, P.~Jain, W.~Krichene, S.~Rendle, S.~Song, A.~Thakurta, and L.~Zhang.
\newblock Private alternating least squares: Practical private matrix
  completion with tighter rates.
\newblock In \emph{Proceedings of the 38th International Conference on Machine
  Learning}, 2021.

\bibitem[Cummings et~al.(2022)Cummings, Feldman, McMillan, and
  Talwar]{cummings2022mean}
R.~Cummings, V.~Feldman, A.~McMillan, and K.~Talwar.
\newblock Mean estimation with user-level privacy under data heterogeneity.
\newblock In A.~H. Oh, A.~Agarwal, D.~Belgrave, and K.~Cho, editors,
  \emph{Advances in Neural Information Processing Systems}, 2022.

\bibitem[Dwork and Roth(2014)]{dwork2014algorithmic}
C.~Dwork and A.~Roth.
\newblock The algorithmic foundations of differential privacy.
\newblock \emph{Foundations and Trends in Theoretical Computer Science},
  9\penalty0 (3--4):\penalty0 211--407, 2014.

\bibitem[Dwork et~al.(2006)Dwork, McSherry, Nissim, and
  Smith]{dwork2006calibrating}
C.~Dwork, F.~McSherry, K.~Nissim, and A.~Smith.
\newblock Calibrating noise to sensitivity in private data analysis.
\newblock In \emph{Proceedings of the Third Conference on Theory of
  Cryptography}, TCC'06, page 265–284, Berlin, Heidelberg, 2006.
  Springer-Verlag.

\bibitem[Dwork et~al.(2007)Dwork, McSherry, and Talwar]{dwork2007price}
C.~Dwork, F.~McSherry, and K.~Talwar.
\newblock The price of privacy and the limits of lp decoding.
\newblock In \emph{Proceedings of the thirty-ninth annual ACM Symposium on
  Theory of Computing}, pages 85--94, 2007.

\bibitem[Dwork et~al.(2010)Dwork, Rothblum, and Vadhan]{dwork2010boosting}
C.~Dwork, G.~N. Rothblum, and S.~Vadhan.
\newblock Boosting and differential privacy.
\newblock In \emph{2010 IEEE 51st Annual Symposium on Foundations of Computer
  Science}, pages 51--60, 2010.

\bibitem[Epasto et~al.(2020)Epasto, Mahdian, Mao, Mirrokni, and
  Ren]{epasto2020smoothly}
A.~Epasto, M.~Mahdian, J.~Mao, V.~Mirrokni, and L.~Ren.
\newblock Smoothly bounding user contributions in differential privacy.
\newblock In H.~Larochelle, M.~Ranzato, R.~Hadsell, M.~Balcan, and H.~Lin,
  editors, \emph{Advances in Neural Information Processing Systems}, volume~33,
  pages 13999--14010. Curran Associates, Inc., 2020.

\bibitem[Feldman and Zrnic(2021)]{feldman2021individual}
V.~Feldman and T.~Zrnic.
\newblock Individual privacy accounting via a r\'enyi filter.
\newblock In A.~Beygelzimer, Y.~Dauphin, P.~Liang, and J.~W. Vaughan, editors,
  \emph{Advances in Neural Information Processing Systems}, 2021.

\bibitem[Foulds et~al.(2016)Foulds, Geumlek, Welling, and
  Chaudhuri]{foulds2016ssp}
J.~Foulds, J.~Geumlek, M.~Welling, and K.~Chaudhuri.
\newblock On the theory and practice of privacy-preserving bayesian data
  analysis.
\newblock In \emph{Proceedings of the Thirty-Second Conference on Uncertainty
  in Artificial Intelligence}, UAI'16, page 192–201, Arlington, Virginia,
  USA, 2016. AUAI Press.

\bibitem[Hardt and Rothblum(2010)]{hardt2010mw}
M.~Hardt and G.~N. Rothblum.
\newblock A multiplicative weights mechanism for privacy-preserving data
  analysis.
\newblock In \emph{2010 IEEE 51st Annual Symposium on Foundations of Computer
  Science}, pages 61--70, 2010.

\bibitem[Harper and Konstan(2016)]{harper16movielens}
F.~M. Harper and J.~A. Konstan.
\newblock The movielens datasets: History and context.
\newblock \emph{Acm Transactions on Interactive Intelligent Systems (TiiS)},
  5\penalty0 (4):\penalty0 19, 2016.

\bibitem[Hsu et~al.(2014)Hsu, Roth, Roughgarden, and Ullman]{hsu2014LP}
J.~Hsu, A.~Roth, T.~Roughgarden, and J.~Ullman.
\newblock Privately solving linear programs.
\newblock In \emph{Automata, Languages, and Programming}, pages 612--624,
  Berlin, Heidelberg, 2014. Springer Berlin Heidelberg.

\bibitem[Hu et~al.(2021)Hu, Wu, and Smith]{hu2022multitask}
S.~Hu, Z.~S. Wu, and V.~Smith.
\newblock Private multi-task learning: Formulation and applications to
  federated learning.
\newblock \emph{CoRR}, abs/2108.12978, 2021.

\bibitem[Jain et~al.(2013)Jain, Netrapalli, and Sanghavi]{jain2013lowrank}
P.~Jain, P.~Netrapalli, and S.~Sanghavi.
\newblock Low-rank matrix completion using alternating minimization.
\newblock In \emph{Proceedings of the Forty-Fifth Annual ACM Symposium on
  Theory of Computing}, STOC '13, page 665–674, New York, NY, USA, 2013.
  Association for Computing Machinery.
\newblock ISBN 9781450320290.

\bibitem[Jain et~al.(2018)Jain, Thakkar, and Thakurta]{jain2018differentially}
P.~Jain, O.~D. Thakkar, and A.~Thakurta.
\newblock Differentially private matrix completion revisited.
\newblock In \emph{International Conference on Machine Learning}, pages
  2215--2224. PMLR, 2018.

\bibitem[Jain et~al.(2021)Jain, Rush, Smith, Song, and
  Guha~Thakurta]{jain2021private}
P.~Jain, J.~Rush, A.~Smith, S.~Song, and A.~Guha~Thakurta.
\newblock Differentially private model personalization.
\newblock In \emph{Advances in Neural Information Processing Systems},
  volume~34, pages 29723--29735. Curran Associates, Inc., 2021.

\bibitem[Kasiviswanathan et~al.(2013)Kasiviswanathan, Nissim, Raskhodnikova,
  and Smith]{kasiviswanathan2013}
S.~P. Kasiviswanathan, K.~Nissim, S.~Raskhodnikova, and A.~Smith.
\newblock Analyzing graphs with node differential privacy.
\newblock In \emph{Proceedings of the 10th Theory of Cryptography Conference on
  Theory of Cryptography}, TCC'13, page 457–476, Berlin, Heidelberg, 2013.
  Springer-Verlag.

\bibitem[Kearns et~al.(2014)Kearns, Pai, Roth, and Ullman]{kearns2014mechanism}
M.~Kearns, M.~Pai, A.~Roth, and J.~Ullman.
\newblock Mechanism design in large games: Incentives and privacy.
\newblock In \emph{Proceedings of the 5th conference on Innovations in
  theoretical computer science}, pages 403--410, 2014.

\bibitem[Korolova(2010)]{korolova2010privacy}
A.~Korolova.
\newblock Privacy violations using microtargeted ads: A case study.
\newblock In \emph{2010 IEEE International Conference on Data Mining
  Workshops}, pages 474--482. IEEE, 2010.

\bibitem[Kubat and Matwin(1997)]{kubat1997addressing}
M.~Kubat and S.~Matwin.
\newblock Addressing the curse of imbalanced training sets: One-sided
  selection.
\newblock In \emph{Proceedings of the Fourteenth International Conference on
  Machine Learning {(ICML} 1997), Nashville, Tennessee, USA, July 8-12, 1997},
  pages 179--186, 1997.

\bibitem[Lee et~al.(2013)Lee, Kim, Lebanon, and Singer]{lee13llorma}
J.~Lee, S.~Kim, G.~Lebanon, and Y.~Singer.
\newblock Local low-rank matrix approximation.
\newblock In \emph{Proceedings of the 30th International Conference on
  International Conference on Machine Learning - Volume 28}, ICML'13, pages
  II--82--II--90, 2013.

\bibitem[Levy et~al.(2021)Levy, Sun, Amin, Kale, Kulesza, Mohri, and
  Suresh]{levy2021learning}
D.~Levy, Z.~Sun, K.~Amin, S.~Kale, A.~Kulesza, M.~Mohri, and A.~T. Suresh.
\newblock Learning with user-level privacy.
\newblock In \emph{Advances in Neural Information Processing Systems},
  volume~34, pages 12466--12479, 2021.

\bibitem[Li et~al.(2020)Li, Khodak, Caldas, and Talwalkar]{li2020metalearning}
J.~Li, M.~Khodak, S.~Caldas, and A.~Talwalkar.
\newblock Differentially private meta-learning.
\newblock In \emph{8th International Conference on Learning Representations,
  {ICLR} 2020}, 2020.

\bibitem[Liang et~al.(2018)Liang, Krishnan, Hoffman, and Jebara]{liang18vae}
D.~Liang, R.~G. Krishnan, M.~D. Hoffman, and T.~Jebara.
\newblock Variational autoencoders for collaborative filtering.
\newblock WWW '18, page 689–698, 2018.

\bibitem[Liu and Zheng(2020)]{liu2020longtail}
S.~Liu and Y.~Zheng.
\newblock \emph{Long-Tail Session-Based Recommendation}, page 509–514.
\newblock Association for Computing Machinery, New York, NY, USA, 2020.

\bibitem[McMahan et~al.(2018)McMahan, Ramage, Talwar, and Zhang]{mcmahan2018}
H.~B. McMahan, D.~Ramage, K.~Talwar, and L.~Zhang.
\newblock Learning differentially private recurrent language models.
\newblock In \emph{6th International Conference on Learning Representations,
  {ICLR} 2018, Vancouver, BC, Canada}, 2018.

\bibitem[Mironov(2017)]{mironov2017renyi}
I.~Mironov.
\newblock R{\'e}nyi differential privacy.
\newblock In \emph{2017 IEEE 30th Computer Security Foundations Symposium
  (CSF)}, pages 263--275. IEEE, 2017.

\bibitem[Proserpio et~al.(2014)Proserpio, Goldberg, and
  McSherry]{proserpio2014wPinq}
D.~Proserpio, S.~Goldberg, and F.~McSherry.
\newblock Calibrating data to sensitivity in private data analysis: A platform
  for differentially-private analysis of weighted datasets.
\newblock \emph{Proc. VLDB Endow.}, 7\penalty0 (8):\penalty0 637–648, apr
  2014.

\bibitem[Rogers et~al.(2021)Rogers, Subramaniam, Peng, Durfee, Lee, Kancha,
  Sahay, and Ahammad]{rogers2021linkedin}
R.~Rogers, S.~Subramaniam, S.~Peng, D.~Durfee, S.~Lee, S.~K. Kancha, S.~Sahay,
  and P.~Ahammad.
\newblock Linkedin’s audience engagements api: A privacy preserving data
  analytics system at scale.
\newblock \emph{Journal of Privacy and Confidentiality}, 11\penalty0 (3), Dec.
  2021.

\bibitem[Shokri et~al.(2017)Shokri, Stronati, Song, and
  Shmatikov]{shokri2017membership}
R.~Shokri, M.~Stronati, C.~Song, and V.~Shmatikov.
\newblock Membership inference attacks against machine learning models.
\newblock In \emph{2017 IEEE Symposium on Security and Privacy (SP)}, pages
  3--18. IEEE, 2017.

\bibitem[Ullman(2015)]{ullman2015private}
J.~Ullman.
\newblock Private multiplicative weights beyond linear queries.
\newblock In \emph{Proceedings of the 34th ACM SIGMOD-SIGACT-SIGAI Symposium on
  Principles of Database Systems}, PODS '15, page 303–312, New York, NY, USA,
  2015. Association for Computing Machinery.

\bibitem[Wang(2018)]{wang2018revisiting}
Y.~Wang.
\newblock Revisiting differentially private linear regression: optimal and
  adaptive prediction {\&} estimation in unbounded domain.
\newblock In A.~Globerson and R.~Silva, editors, \emph{Proceedings of the
  Thirty-Fourth Conference on Uncertainty in Artificial Intelligence, {UAI}
  2018, Monterey, California, USA, August 6-10, 2018}, pages 93--103. {AUAI}
  Press, 2018.

\bibitem[Wilson et~al.(2020)Wilson, Zhang, Lam, Desfontaines, Simmons-Marengo,
  and Gipson]{DPSQL}
R.~J. Wilson, C.~Y. Zhang, W.~Lam, D.~Desfontaines, D.~Simmons-Marengo, and
  B.~Gipson.
\newblock Differentially private sql with bounded user contribution.
\newblock In \emph{Privacy Enhancing Technologies Symposium (PETS)}, 2020.

\bibitem[Xu et~al.(2021)Xu, Du, and Wu]{xu2021removing}
D.~Xu, W.~Du, and X.~Wu.
\newblock Removing disparate impact on model accuracy in differentially private
  stochastic gradient descent.
\newblock In \emph{Proceedings of the 27th ACM SIGKDD Conference on Knowledge
  Discovery \& Data Mining}, KDD '21, page 1924–1932, New York, NY, USA,
  2021. Association for Computing Machinery.

\bibitem[Yin et~al.(2012)Yin, Cui, Li, Yao, and Chen]{yin2012longtail}
H.~Yin, B.~Cui, J.~Li, J.~Yao, and C.~Chen.
\newblock Challenging the long tail recommendation.
\newblock \emph{Proc. VLDB Endow.}, 5\penalty0 (9):\penalty0 896–907, may
  2012.

\end{thebibliography}
\bibliographystyle{abbrvnat}

\newpage
\appendix
\section*{Appendix}
The proofs of the main results are provided in Appendix~\ref{app:proofs}. Appendix~\ref{app:private_counts} discusses the case when task counts are not public. We report additional experiments on synthetic data (Appendix~\ref{app:synthetic_expt}) and real data (Appendix~\ref{app:movielens_expt}).


\section{Proofs}
\label{app:proofs}
\subsection{Theorem~\ref{thm:privacy-ssp} (Privacy guarantee of Algorithm~\ref{alg:ssp})}

\begin{proof}
The result is an application of the Gaussian mechanism. The procedure computes, for $i \in [m]$, the estimates $\hat A_i$ and $\hat b_i$ (Lines~\ref{line:Ai}-\ref{line:bi}), given as follows
\begin{align*}
&\hat A_i = \bar A_i + \Gamma_x^2 \Xi_i, \quad \bar A_i = \sum_{j \in \Omega_i} w_{ij} (x_{ij} x_{ij}^\top + \lambda I) \\
&\hat b_i = \bar b_i + \Gamma_x\Gamma_y \xi_i, \quad \bar b_i = \sum_{j \in \Oi} w_{ij} y_{ij}x_{ij}
\end{align*}
where $\Xi_i \sim \Ncal^{d \times d}$, $\xi_i \sim \Ncal^d$. Let $\bar A$ be the matrix obtained by stacking $(\bar A_i)_{i \in [n]}$. If $\bar A'$ is the same matrix obtained without user $j$'s data, then
\[
\|\bar A - \bar A'\|_F^2
= \sum_{i \in \Oj} \|w_{ij}x_{ij} x_{ij}^\top\|_F^2
\leq \sum_{i \in \Oj} w_{ij}^2 \Gamma_x^4 \leq \beta \Gamma_x^4
\]
where we use the fact that for all $i$, $\|x_{ij}\| \leq \Gamma_x$ and $\sum_{j \in \Omega_i} w_{ij}^2 \leq \beta$. Since we add Gaussian noise with variance $\Gamma_x^4$, releasing $\hat A$ is $(\alpha, \alpha \frac{\beta}{2})$-RDP~\citep{mironov2017renyi}. 

Similarly, if $\bar b$ is obtained by stacking $(\bar b_i)_{i \in [m]}$, and $\bar b'$ is the same vector without user $j$'s data, then $\|\bar b - \bar b'\|^2 \leq \sum_{i \in \Oj} \|w_{ij}y_{ij}x_{ij}\|^2 \leq  \sum_{j \in \Omega_i} w_{ij}^2\Gamma_y^2\Gamma_x^2 \leq \beta\Gamma_y^2\Gamma_x^2$, and releasing $\hat b$ is $(\alpha, \alpha \frac{\beta}{2})$-RDP. By simple RDP composition, the process is $(\alpha, \alpha \beta)$-RDP.
\end{proof}


\subsection{Theorem~\ref{thm:tradeoff-ssp} (Utility guarantee of Algorithm~\ref{alg:ssp})}
\begin{proof}
Since by assumption, the weights satisfy $\sum_{i \in \Omega_j} \omega_i^2 \leq \beta$, the RDP guarantee is an immediate consequence of Theorem~\ref{thm:privacy-ssp}.

To prove the utility bound, recall that the total loss is a sum over tasks
\begin{equation}
L(\theta) = \sum_{i = 1}^n \|A_i\theta_i - b_i\|_F^2 + n_i\lambda\|\theta_i\|^2.\label{proof2:loss_decomposition}
\end{equation}
We will bound the excess risk of each term, using the following result. For a proof, see, e.g.~\citep[Appendix~B.1]{wang2018revisiting}.
\begin{lemma}
\label{lem:ssp}
Suppose Assumption~\ref{assump:SSP} holds. Consider the linear regression problem $L(\theta_i) = \|A_i\theta_i - b_i\|_F^2$, let $\theta_i^*$ be its solution, and $\hat \theta_i$ be the SSP estimate obtained by replacing $A_i$ and $b_i$ with their noisy estimates $\hat A_i = A_i + \sigma \Gamma_x^2 \Xi, \hat b_i = b_i + \sigma \Gamma_x^2\Gamma_* \xi$ where $\Xi \sim \Ncal^{d\times d}, \xi\sim\Ncal^d$. Then
\[
L_i(\hat \theta_i) - L_i(\theta_i^*) = \bigO{\frac{d^2\Gamma_x^2\Gamma_{*}^2}{\alpha n_i}\sigma^2},
\]
where $\alpha = \frac{\lambda_{\min}(A_i^\top A_i)d}{n_i\Gamma_x^2} = \frac{\lambda d}{\Gamma_x^2}$.
\end{lemma}

Fix a task $i$. Since by assumption all weights $w_{ij}$ are equal to $\omega_i$, Lines~\ref{line:Ai}-\ref{line:bi} become $\hat A_i = \omega_i A_i + \Gamma_x^2 \Xi_i$ and $\hat b_i = \omega_i b_i + \Gamma_x^2\Gamma_* \xi_i$, and $\hat \theta_i$ is the solution to $\hat A_i\theta_i = \hat b_i$. This corresponds to the SSP algorithm, applied to the loss $L_i$, with noise variances $\sigma^2 = \frac{1}{\omega_i^2}$. By Lemma~\ref{lem:ssp}, we have
\begin{equation}
L_i(\hat\theta_i) - L_i(\theta_i^*) = \bigO{\frac{d^2\Gamma_x^2\Gamma_*^2}{\alpha} \frac{1}{n_i\omega_i^2}} \label{proof:ssp-2}
\end{equation}

We conclude by summing~\eqref{proof:ssp-2} over $i \in [m]$.
\end{proof}

\subsection{Theorem~\ref{thm:privacy-dpsgd} (Privacy guarantee of Algorithm~\ref{alg:dpsgd})}
\begin{proof}
The result is an application of the Gaussian mechanism. At each step $t$ of Algorithm~\ref{alg:dpsgd}, the procedure computes, for $i \in [m]$, a noisy estimate of the gradient $g_i^{(t)} = \sum_{j \in \Omega_i} g_i^{(t)}$ (Line~\ref{line:gradient}). Let $g^{(t)} \in \Rbb^{md}$ be the vector obtained by stacking $g_i^{(t)}$ for all $i$. And let ${g'}^{(t)}$ be the same vector obtained without user $j$'s data. Then
\begin{align*}
\|{g'}^{(t)} - g^{(t)}\|_2^2
&= \sum_{i \in \Oj} \|w_{ij} \nabla \ell(\theta_i; x_{ij},y_{ij})\|_2^2 \\
&\leq \Gamma^2 \sum_{i \in \Oj} w_{ij^2} \\
&\leq \beta \Gamma^2.
\end{align*}
where we use the assumption that $\sum_{i \in \Oj} w_{ij^2} \leq \beta$. Since we add Gaussian noise with variance $\Gamma^2T/2$, the procedure is $(\alpha, \alpha \beta/T)$-RDP. Finally, by composition over $T$ steps, the algorithm is $(\alpha, \beta)$-RDP.
\end{proof}

\subsection{Theorem~\ref{thm:tradeoff-dpsgd} (Utility guarantee of Algorithm~\ref{alg:dpsgd})}
We will make use of the following standard lemmas (for example Lemma~2.5 and 2.6 in~\cite{bassily2014erm}). Let $f$ be a convex function defined on domain $\Ccal$, let $\theta^* = \argmin_{\theta \in \Ccal} f(\theta)$. Consider the SGD algorithm with learning $\eta_t$.
\[
\theta^{(t+1)} = \Pi_\Ccal[\theta^{(t)} - \eta_t g^{(t)}].
\]
Assume that there exists $G$ such that for all $t$, $\Exp[g^{(t)}] = \nabla \ell(\theta^{(t)})$ and $\Exp[\|g^{(t)}\|^2] \leq G^2$,
\begin{lemma}[Lipschitz case]\label{lem:convex} Let $\eta^{(t)} = \frac{\|\Ccal\|}{G\sqrt{t}}$. Then for all $T \geq 1$,
\[
\Exp[f(\theta^{(t)})] - f(\theta^*) = \bigO{\frac{\|C\|G \log T}{\sqrt T}}.
\]
\end{lemma}

\begin{lemma}[Strongly convex case]\label{lem:strongly-convex} Assume that $f$ is $\lambda$ strongly convex and let $\eta^{(t)} = \frac{1}{\lambda t}$. Then for all $T \geq 1$,
\[
\Exp[f(\theta^{(t)})] - f(\theta^*) = \bigO{\frac{G^2 \log T}{\lambda T}}.
\]
\end{lemma}

\begin{proof}[proof of Theorem~\ref{thm:tradeoff-dpsgd}]
First, since $w_{ij} = \omega_i$, the gradient $g_i^{(t)}$ in Line~\ref{line:gradient} of the algorithm becomes
\[
g_i^{(t)} = \omega_i \nabla L_i(\theta_i^{(t-1)}).
\]
Let $\hat g_i^{(t)} = g_i^{(t)} + \Gamma\sqrt{T/2}\xi_i^{(t)}$. Then $\Exp[\hat g_i^{(t)}] = g_i^{(t)}$ and
\begin{align*}
\Exp[\|\hat g_i^{(t)}\|^2]
&= \Exp[\|g_i^{(t)}\|^2] + \Exp[\|\Gamma\sqrt{T/2}\xi_i^{(t)}\|^2]\\
&\leq \Gamma^2[\omega_i^2n_i^2 + Td/2].
\end{align*}
In the first line we use independence of $\xi_i^{(t)}$ and $g_i^{(t)}$. In the second line we use that the variance of a multivariate normal is $d$, and the fact that $L_i$ has Lipschitz constant $n_i \Gamma$ (since it is the sum of $n_i$ terms, each being $\Gamma$-Lipschitz).

Define
\[
G_i^2 = \Gamma^2[\omega_i^2n_i^2 + Td/2].
\]

First, consider the Lipschitz bounded case. Applying Lemma~\ref{lem:convex} to $f = \omega_i L_i$, $G = G_i$, and $\eta_i = \frac{\|C\|}{G_i\sqrt t}$, we have for all $T$
\[
\omega_i\Exp[L_i(\theta_i) - L_i(\theta_i^*)] = \bigO{\frac{\|C\|G_i\log T}{\sqrt T}}.
\]
Multiplying by $\frac{1}{\omega_i}$ and summing over tasks $i \in \{1, \dots, m\}$, we have
\begin{align*}
\Exp[L(\theta)] - L(\theta^*)
&= \bigO{\frac{\|\Ccal\|\Gamma\log T}{\sqrt T}\sum_{i=1}^m\sqrt{n_i^2 + \frac{Td}{2\omega_i^2}}} \\
&= \bigO{\|\Ccal\|\Gamma\log T\sqrt{m} \sqrt{\frac{\sum_{i=1}^m n_i^2}{T} + \sum_{i=1}^m\frac{d}{2\omega_i^2}}}
\end{align*}
where we used $\sum_{i = 1}^m \sqrt{a_i} \leq \sqrt{m\sum_{i = 1}^m a_i}$ (by concavity). Finally, setting $T$ to equate the terms under the square root, we get $T = \frac{2}{d}\frac{\sum_{i = 1}^m n_i^2}{\sum_{i = 1}^m 1/\omega_i^2}$, and with this choice of $T$,
\[
\Exp[L(\theta)] - L(\theta^*) = \bigOtilde{\|\Ccal\|\Gamma \sqrt{md}\sqrt{\sum_{i = 1}^m \frac{1}{\omega_i^2}}},
\]
as desired.

We now turn to the strongly convex case. Applying Lemma~\ref{lem:strongly-convex} to $f = \omega_i L_i$, $G = G_i$ (same as above), strong convexity constant $\omega_i n_i \lambda$, and $\eta_i = \frac{1}{\omega_i n_i \lambda t}$, we have for all~$T$,
\[
\omega_i\Exp[L_i(\theta_i) - L_i(\theta_i^*)] = \bigO{\frac{G_i^2\log T}{n_i\omega_i\lambda T}}.
\]
Multiplying by $\frac{1}{\omega_i}$ and summing over tasks, we get
\[
\Exp[L(\theta)] - L(\theta^*) = \bigO{\frac{\Gamma^2 \log T}{\lambda}\left(\frac{1}{T}\sum_{i = 1}^m n_i + \frac{d}{2}\sum_{i = 1}^m \frac{1}{\omega_i^2n_i}\right)}.
\]
Setting $T$ to equate the last two sums, we get $T = \frac{2}{d}\frac{\sum_{i = 1}^m n_i}{\sum_{i = 1}^m 1/\omega_i^2n_i}$ and with this choice of $T$,
\[
\Exp[L(\theta)] - L(\theta^*) = \bigOtilde{\frac{\Gamma^2d}{\lambda}\sum_{i = 1}^m \frac{1}{\omega_i^2n_i}},
\]
as desired.
\end{proof}

\subsection{Theorem~\ref{thm:exact} (Privacy-utility trade-off under adaptive weights)}
To prove the theorem, we will use the following concentration result:
\begin{lemma}[Concentration bound on the privacy budget]
\label{lem:concentration}Suppose that Assumption~\ref{assump:omega} holds and let $B>0$ be given. Let $n_i = |\Omega_i|$ and $\omega_i=n_i^{-(1+\gamma)/4}/\sqrt{\sum_{i'} n_{i'}^{(1-\gamma)/2}/nB}$. Then there exists $c>0$ (that does not depend on $B$) such that, with high probability, for all $i$, $\sum_{i\in\Oj} \omega_i^2 \leq Bc\log n$.
\end{lemma}
\begin{proof}
Fix a user $j$. We seek to bound $\sum_{i \in \Oj} \omega_i^2$. Let $p_{ij} = P((i,j) \in \Omega)$. Recall that by Assumption~\ref{assump:omega}, $p_{ij} = n_i/n$.

For each $i$, define $X_i$ as the random variable which takes value $\omega_i^2 = nB n_i^{-(1+\gamma)/2}/\sum_{i'} n_{i'}^{(1-\gamma)/2}$ with probability $n_i/n$ and $0$ otherwise. Then bounding $\sum_{i \in \Oj} \omega_i^2$ is equivalent to bounding $\sum_{i = 1}^m X_i$.

We have
\begin{align*}
\mean[X_i]
&= B n_i^{(1-\gamma)/2}/\sum_{i'} n_{i'}^{(1-\gamma)/2}\,,\\
\var[X_i]
&\leq \mean[X_i^2]
= (n_i/n) \cdot \Big(nB n_i^{-(1+\gamma)/2}/\sum_{i'} n_{i'}^{(1-\gamma)/2}\Big)^2\\
&= nB^2 n_i^{-\gamma}/\Big(\sum_{i'} n_{i'}^{(1-\gamma)/2}\Big)^2 \leq B^2/n\,,
\end{align*}
where the last inequality is by the fact that $\forall i, \ n_i\geq 1$ (see Assumption~\ref{assump:omega}). Hence we have that
\[
\sum_{i = 1}^m \mean[X_i] = B\,,\quad \sum_{i = 1}^m \var[X_i] \leq (m/n)B^2\leq B^2,
\]
where we use the assumption that $n \geq m$. Furthermore, we have $\omega_i^2 \leq B$ since $\forall j,\ n_i\geq 1$. By the standard Bernstein's inequality (\cite{boucheron2013concentration} \S 2.8), there exists $c>0$ (that does not depend on $B$ or $n_i$) such that:
\[
\Pr\left(\sum_{i = 1}^m X_i \geq B c \log n \right) \leq 1/n^3\,.
\]
The high probability bound follows by taking the union over all the $i$'s.
\end{proof}

\begin{proof}[Proof of Theorem~\ref{thm:exact}]
Let $\omega^*$, $w^*$ be the weights given in eq.~\eqref{eq:w}-\eqref{eq:w_final}, respectively ($\omega^*$ are the optimal weights, and $w^*$ are their clipped version). We apply Lemma~\ref{lem:concentration} with $B=\bar\beta=\beta/c\log n$ and hence obtain that with high probability, for all $i$, $\sum_{j\in\Omega_i} {\omega_i^*}^2 \leq \beta$. In other words, w.h.p. clipping does not occur and $w^*_{ij} = \omega_i$ for all $(i, j) \in \Omega$. Conditioned on this high probability event, we can apply the general utility bounds in Theorems~\ref{thm:tradeoff-ssp} and~\ref{thm:tradeoff-dpsgd}, with $\omega_i = \omega_i^*$. This yields the desired bounds.
\end{proof}

\section{Utility analysis under approximate counts}
\label{app:private_counts}

In Section~\ref{sec:optimal}, the optimal choice of weights $\omega^*$ assumes knowledge of the counts $n_i$. If the counts $n_i$ are not public, we can use DP estimates $\hat n_i$, and use them to solve the problem. Since differentially private counting has been studied extensively, we only state the utility bound in terms of the accuracy of the noisy counts -- the final privacy bound can be done through standard composition. For any constant $s>0$, we call a counting procedure $\mathcal{M}$ $s$-accurate, if w.h.p., the counts $\hat n_i$ produced by $\mathcal{M}$ satisfy that $|\hat n_i - n_i|\leq s$ for all $i$.

We will give an analysis of Algorithm~\ref{alg:ssp} with approximate counts. The general idea is to apply the algorithm with weights $\hat\omega$ of the same form as the optimal weights $\omega^*$, but where the exact counts are replaced with estimated counts. Suppose we are given $s$-accurate count estimates $\hat n_i$. Define
\begin{equation}
\label{eq:approx_w}
\hat \omega_i = \frac{1}{(\hat n_i + s)^{1/2}\sqrt{m/n\bar\beta}},
\end{equation}
This corresponds to eq.~\eqref{eq:w}, but with $n_i$ replaced by $\hat n_i + s$ (it will become clear below why we need to slightly over-estimate the counts).

As in the exact case, we need to clip the weights, so that the privacy guarantee always holds. Define the clipped version $\hat w$ as
\begin{equation}
\label{eq:approx_w_final}
\hat w_{ij} = \hat\omega_i \min\Big(1, \frac{\beta^{1/2}}{(\sum_{i' \in \Oj} \hat\omega^2_{i'})^{1/2}}\Big) \quad \forall (i, j) \in \Omega.
\end{equation}
\begin{remark}
The input to the algorithm are the unclipped weights $\hat\omega$, which are computed differentially privately. The clipped weights are not released as part of the procedure. They are simply used as a scaling factor in the computation of $\hat A_i, \hat b_i$ (Lines~\ref{line:Ai}-\ref{line:bi}).
\end{remark}

\begin{theorem}[Privacy-utility trade-off of Algorithm~\ref{alg:ssp} with estimated counts]
\label{thm:approximate}
Suppose Assumptions~\ref{assump:SSP} and~\ref{assump:omega} hold. Let $\hat n_i \in \Rbb^m$ be count estimates from an $s$-accurate procedure for some $s > 0$. Define $\hat \omega, \hat w$ as in eq.~\eqref{eq:approx_w}-\eqref{eq:approx_w_final}. Let $\hat \theta$ be the output of Algorithm~\ref{alg:ssp} run with weights $\hat w$. Then Algorithm~\ref{alg:ssp} is $(\alpha, \alpha\beta)$-RDP for all $\alpha > 1$, and w.h.p.,
\begin{align}
L(\hat \theta) - L(\theta^*) = \bigOtilde{\frac{\Gamma_x^4\Gamma_*^2dm}{n\lambda\beta} \sum_{i = 1}^m \frac{\hat n_i + s}{n_i}}.
\label{eq:final_bound_approx}
\end{align}

\end{theorem}

\begin{proof}[Proof of Theorem~\ref{thm:approximate}]
We first show that w.h.p., the approximate weights $\hat \omega_i$ \emph{under-estimate} the exact weights $\omega_i^*$. This will allow us to argue that w.h.p., clipping will not occur.

First, define the function $g : \Rbb^d \to \Rbb^d$, 
\[
g_i(\eta) = \frac{1}{\eta_i^{1/2}\sqrt{m/n\bar\beta}}.
\]
Observe that whenever $\eta \leq \eta'$ (coordinate-wise), $g(\eta) \geq g(\eta')$ (coordinate-wise).

Let $\eta^*$ be the vector of true counts ($\eta^*_i = n_i$), and $\hat\eta$ be the vector or approximate counts $\hat \eta_i = n_i$. We have that $\omega^* = g(\eta^*)$ and $\hat \omega = g(\hat \eta + s)$. But since by assumption, the counts estimates are $s$-accurate, then w.h.p., $|\hat n_i - n^*_i| \leq s$, thus $\eta^* \leq \hat \eta + s$ and
\[
\hat \omega = g(\hat \eta + s) \leq \omega(\eta^*) = \omega^*,
\]
that is, the estimated weights $\hat \omega$ under-estimate the optimal weights $\omega^*$ w.h.p. (this is precisely why, in the definition of $\hat \omega$, we used the adjusted counts $\hat \eta + s$ instead of $\hat \eta$).

Next, $\hat \omega \leq \omega^*$ implies that
\begin{equation}
\label{proof4:budget_bound}
\sum_{i \in \Oj} \hat \omega_i^2 \leq \sum_{i \in \Oj} {\omega^*_i}^2.
\end{equation}
We also know, from Lemma~\ref{lem:concentration} (applied with $B = \bar\beta$), that w.h.p., $\sum_{i \in \Oj} {\omega_i^*}^2 \leq c\bar\beta\log n = \beta$ for all~$j$. Combining this with~\eqref{proof4:budget_bound}, we have that w.h.p., $\hat\omega$ satisfies the RDP budget constraint, therefore $w_{ij} = \hat \omega_i^2$ (clipping does not occur since the budget constraint is satisfied).

Applying Theorem~\ref{thm:tradeoff-ssp} with weights $\hat \omega$ yields the desired result.
\end{proof}

We compare the utility bounds under adaptive weights with exact counts~\eqref{eq:final_bound_ssp}, adaptive weights with approximate counts~\eqref{eq:final_bound_approx}, and uniform weights. The bounds are, respectively, $\bigOtilde{\frac{\Gamma_x^4\Gamma_*^2dm^2}{n\lambda\beta}}$, $\bigOtilde{\frac{\Gamma_x^4\Gamma_*^2dm}{n\lambda\beta} \sum_{i = 1}^m \frac{\hat n_i + s}{n_i}}$, and $\bigOtilde{\frac{\Gamma_x^4\Gamma_*^2d}{n\lambda\beta} \sum_{i = 1}^m n_i\sum_{i = 1}^m \frac{1}{n_i}}$.

The cost of using approximation counts (compared to exact counts) is a relative increase by the factor
\begin{align*}
\frac{1}{m}\sum_{i = 1}^m \frac{\hat n_i + s}{n_i}
&\leq \frac{1}{m}\sum_{i = 1}^m \frac{n_i + 2s}{n_i} \\
&= 1 + \frac{2}{m}\sum_{i = 1}^m \frac{s}{n_i}.
\end{align*}
where we used that w.h.p., $\hat n_i \leq n_i + s$. The last term is the \emph{average relative error} of count estimates. If the estimates are accurate on average, we don't expect to see a large utility loss due to using approximate counts.

\begin{figure}[h]
\centering
\includegraphics[width=0.33\textwidth]{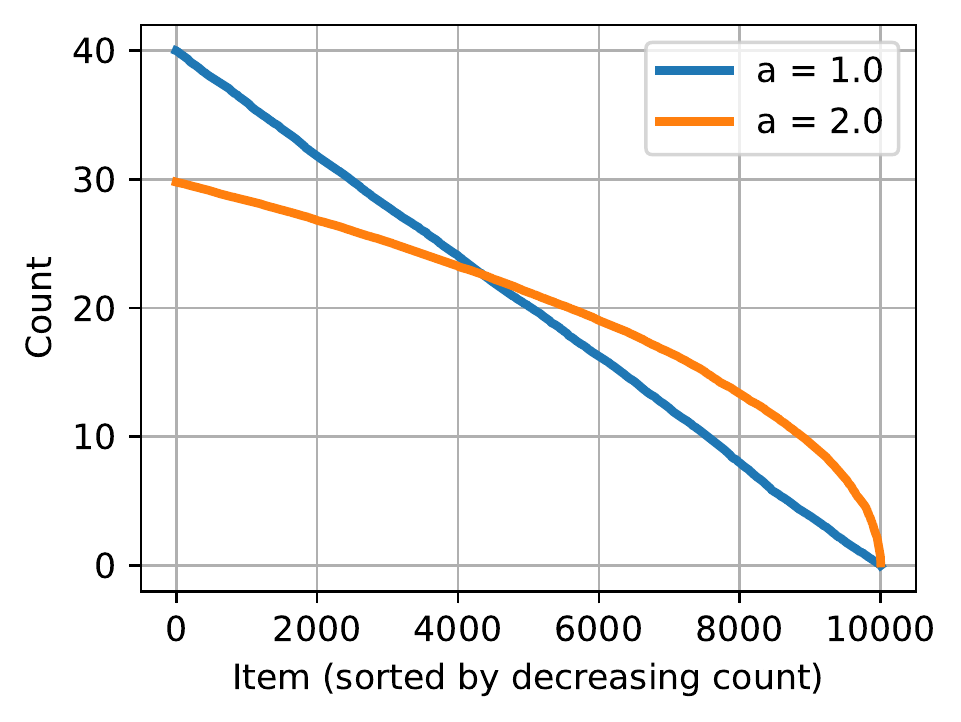}
\caption{Task distribution in the synthetic data.}
\label{fig:powerlaw_data}
\end{figure}

\section{Synthetic Experiments}
\label{app:synthetic_expt}
We conduct experiments following the setup in Section~\ref{sec:budget}. Specifically, we consider $m=100$ linear regression tasks of dimension $d=5$ and data from $n=10,000$ users.
We assume that each task has an optimal solution $\theta_i$, and each user $j$ has a vector $u_j$, such that $x_{ij} = u_j$, and $y_{ij} = \innerp{u_i}{\theta_j} + \mathcal{N}(0, \sigma_F^2)$, with $\sigma_F$ representing some inherent data noise.
We generate $u_j$ and $\theta_i$ from a Gaussian distribution $\mathcal{N}^d$ and projected to the unit ball, and set $\sigma_F = 10^{-3}$. 


To model a skewed distribution of tasks, we first sample, for each task $i$, a value $q_i \in [0, 1]$ following a power law distribution with parameter $a=1, 2$ (the pdf of the distribution is $f(x) \propto a x^{a-1}$), representing two level of skewness. We then normalize $q_i$s such that they sum up to $20$.
Then we construct $\Omega$ by sampling each $(i,j)$ with probability $q_i$. This way, in expectation, each user contributes to $20$ tasks, and $q_i n$ users contribute to task $i$. The distribution of $q_i n$ is plotted in Figure~\ref{fig:powerlaw_data}. We partition the data into training and test sets following an 80-20 random split.


We run AdaDPALS and AdaDPSGD with weights set to $\omega_i \propto n_i^{-\mu}$ for varying $\mu$, where $n_i$ is the number of users contributing to task $i$. The weights are normalized per user. Setting $\mu=0$ corresponds to the standard unweighted objective function. Since the purpose of the experiment is to validate the analysis and illustrate the algorithm in an stylized setting, we run the algorithm using the exact $n_i$s instead of estimating them privately.
In Figure~\ref{fig:synth1.0}-\ref{fig:synth2.0dpsgd}, we plot the RMSE for different values of $\mu$ and $\epsilon$, for both algorithms and skewness. We set $\delta = 10^{-5}$.

\begin{figure}[h]
\centering
\begin{subfigure}[b]{0.49\textwidth}
\includegraphics[width=1\textwidth]{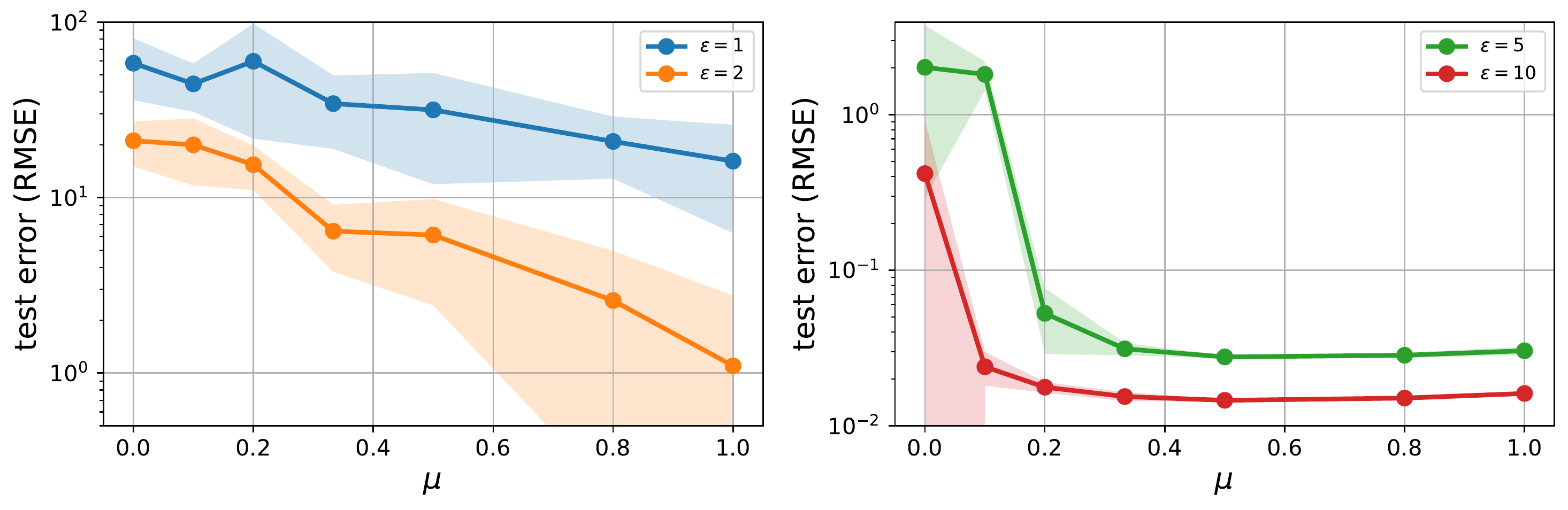}
\caption{DPMultiRegression. Skewness $a=1$.}
\label{fig:synth1.0}
\end{subfigure}
\begin{subfigure}[b]{0.49\textwidth}
\includegraphics[width=1\textwidth]{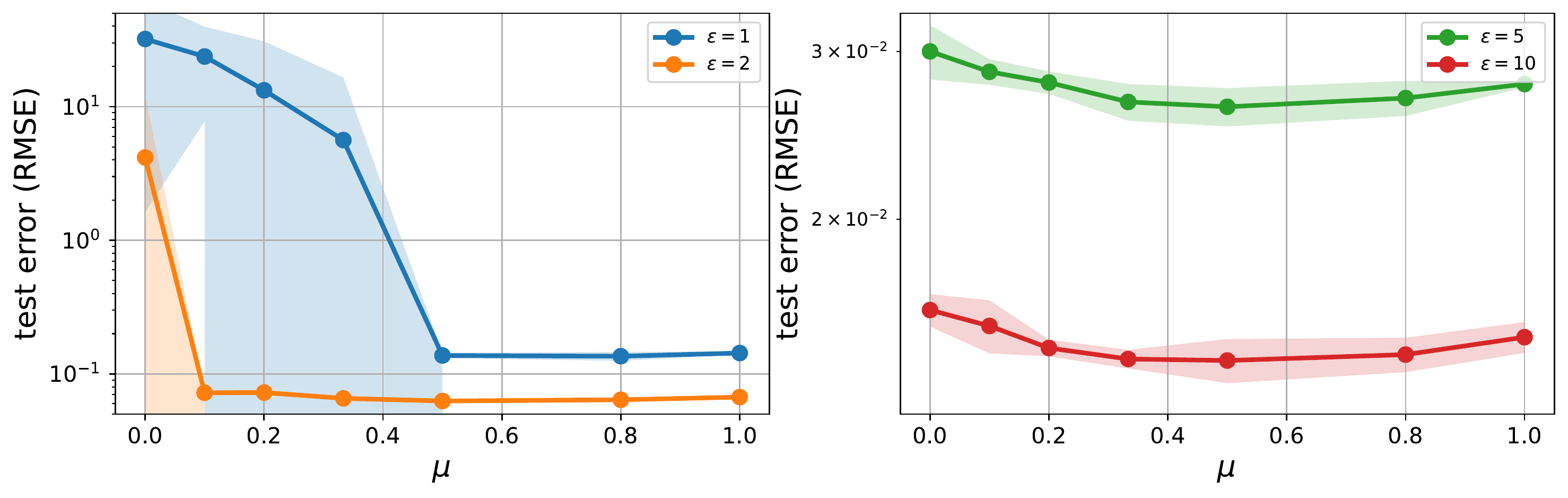}
\caption{DPMultiRegression. Skewness $a=2$.}
\label{fig:synth2.0}
\end{subfigure}
\begin{subfigure}[b]{0.49\textwidth}
\includegraphics[width=1\textwidth]{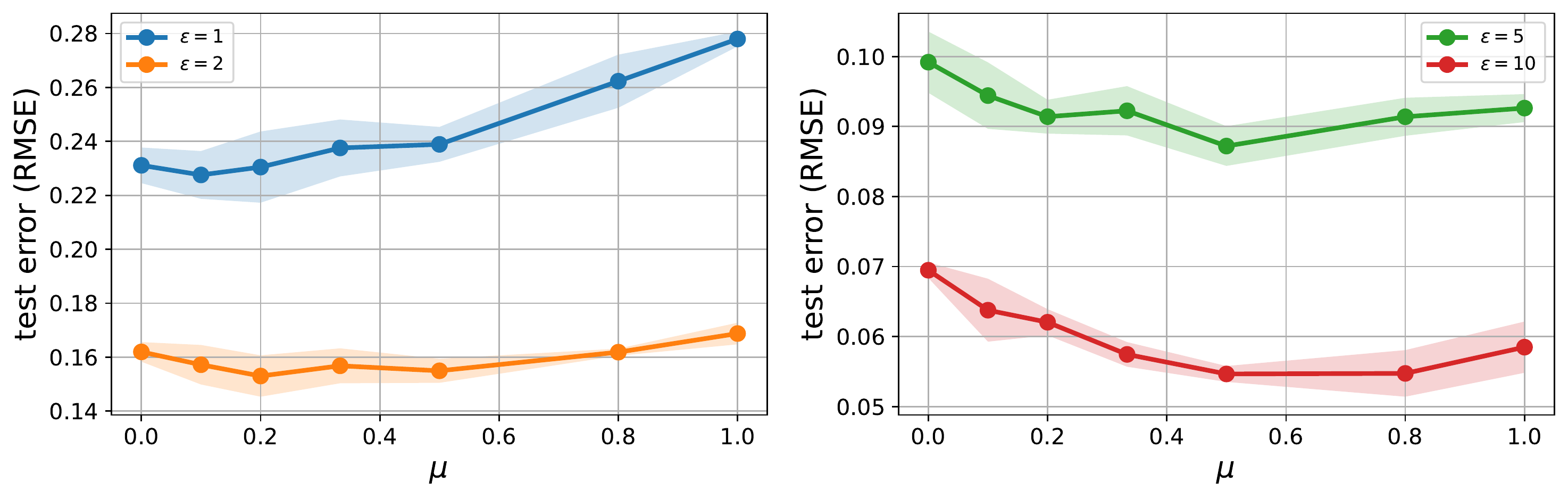}
\caption{DPSGD. Skewness $a=1$.}
\label{fig:synth1.0dpsgd}
\end{subfigure}
\begin{subfigure}[b]{0.49\textwidth}
\includegraphics[width=1\textwidth]{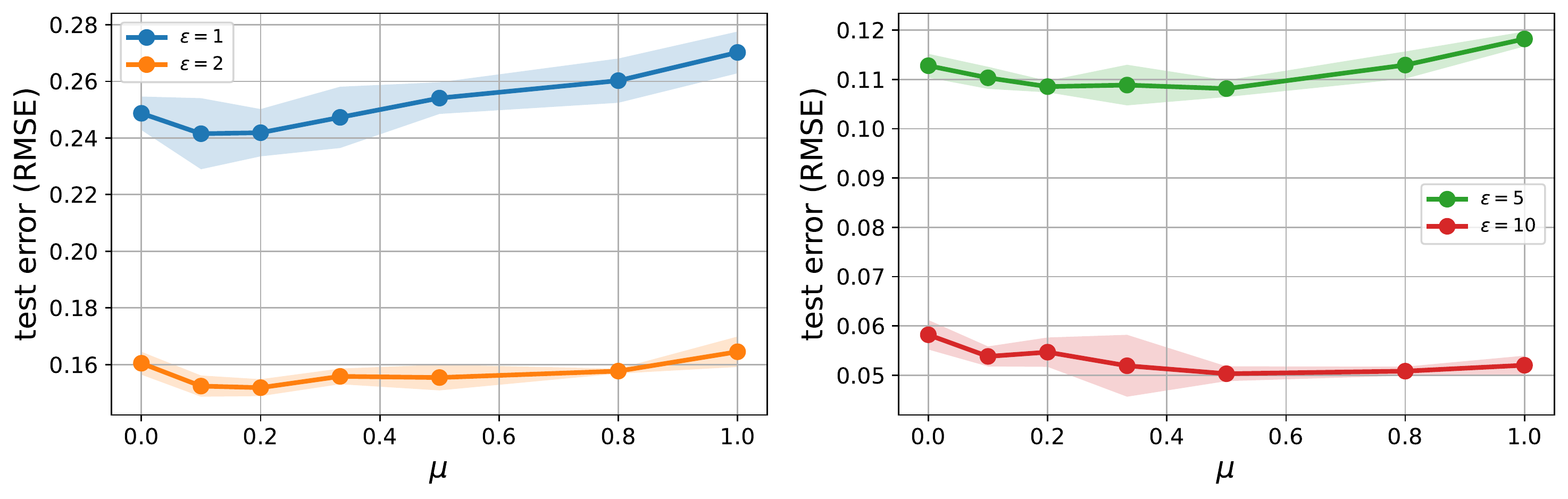}
\caption{DPSGD. Skewness $a=2$.}
\label{fig:synth2.0dpsgd}
\end{subfigure}
\caption{RMSE vs. $\mu$ on synthetic data. $\delta = 10^{-5}$.}
\end{figure}

We observe that with uniform weights ($\mu=0$), the quality of the estimate can be quite poor, especially at lower values of $\epsilon$. The quality improves significantly as we increase $\mu$. 
Theoretical analysis (Equation \eqref{eq:w}) suggests $\mu = 1/2$ as the optimal choice, yet the empirically optimal $\mu$ can vary case by case.
This may be due to the fact that the analysis makes no assumptions about the feature and label distribution, while in the experiment, the data is sampled from a linear model; we leave further analysis into designing better weighting strategies for {\em easier} data for future work.

\section{Additional Experiments on MovieLens and Million Song Data}
\label{app:movielens_expt}
In this section, we report additional experiments on MovieLens and Million Song Data~\citep{bertinmahieux2011msd} data sets.

\begin{table}[h!]
\centering
\small
\caption{Statistics of the MovieLens and Million Song data sets.}
\label{tbl:movielens}
\begin{tabular}{cccccc} 
\toprule
 & ML10M & ML20M & MSD \\ [0.5ex]
\midrule
$n$ (number of users)  & 69,878 & 136,677 & 571,355 \\
$m$ (number of items)  & 10,677 & 20,108 & 41,140 \\
$|\Omega|$ (number of observations) & 10M & 9.99M & 33.63M \\
\bottomrule
\end{tabular}
\end{table}


\begin{figure}[h]
\centering
\includegraphics[width=.36\textwidth]{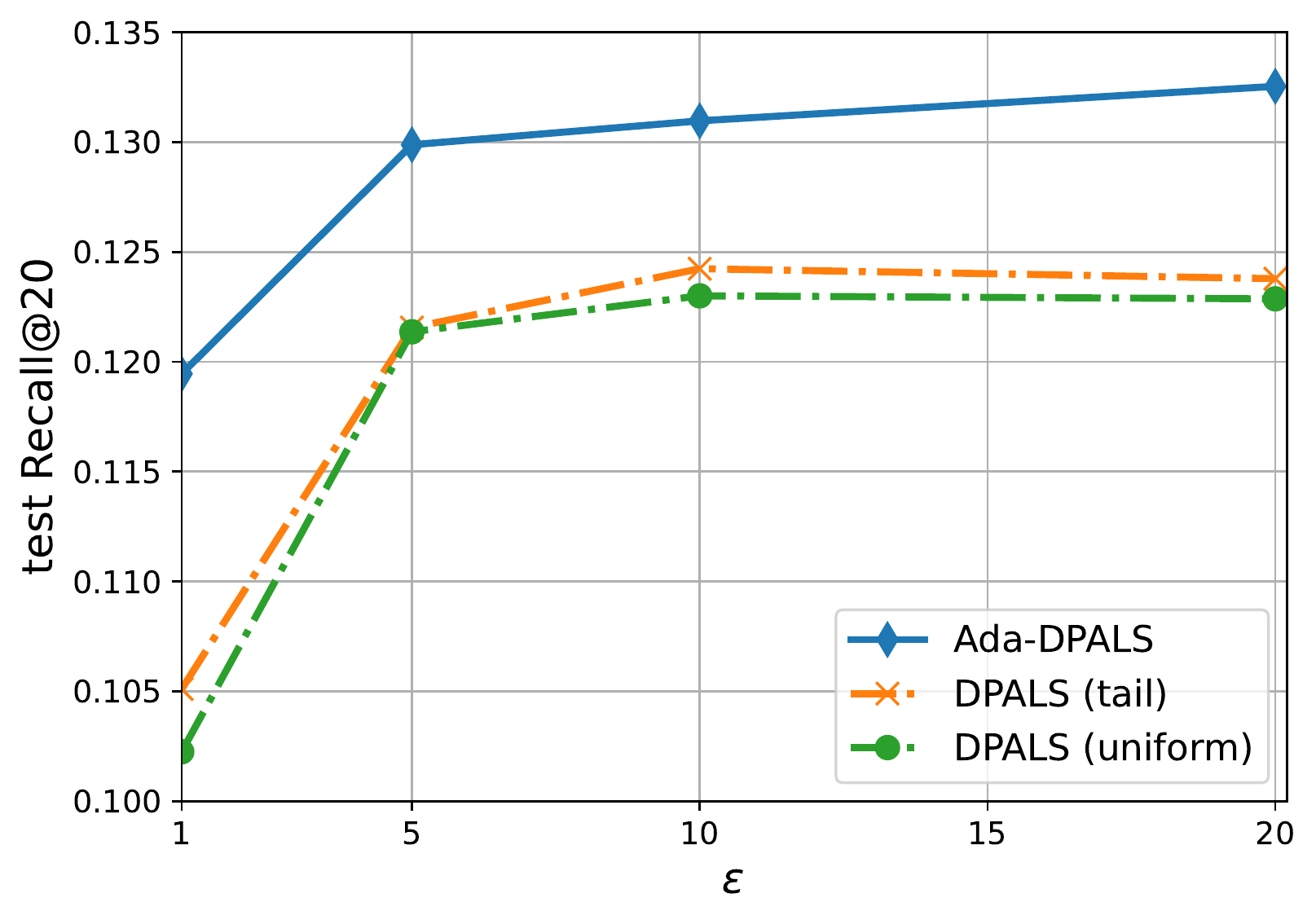}
\caption{Privacy-utility trade-off on the Million Song Data benchmark.}
\label{fig:app_msd}
\end{figure}

We report additional experiments on the larger Million Song Data benchmark (abbreviated as MSD), see Figures~\ref{fig:app_msd}-\ref{fig:app_sliced}. Due to the larger size of the data, we were only able to tune models of smaller size (embedding dimension of up to $32$), but we expect the trends reported here to persist in larger dimension.

\begin{figure}[h]
\centering
\begin{subfigure}[b]{0.37\textwidth}
\includegraphics[width=\textwidth]{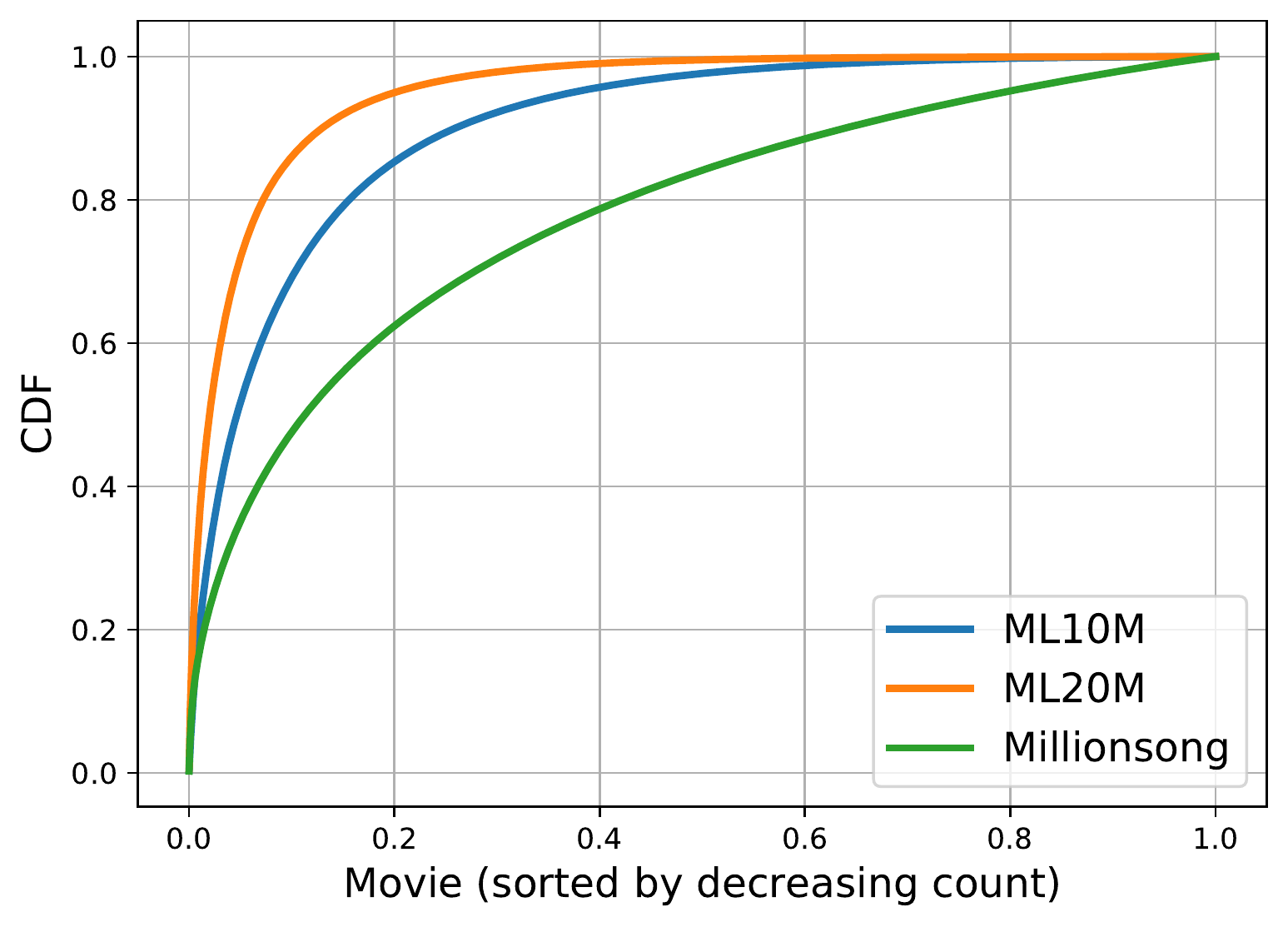}
\caption{Cumulative distribution function}
\end{subfigure}%
\hspace{0.1\textwidth}%
\begin{subfigure}[b]{0.37\textwidth}
\includegraphics[width=\textwidth]{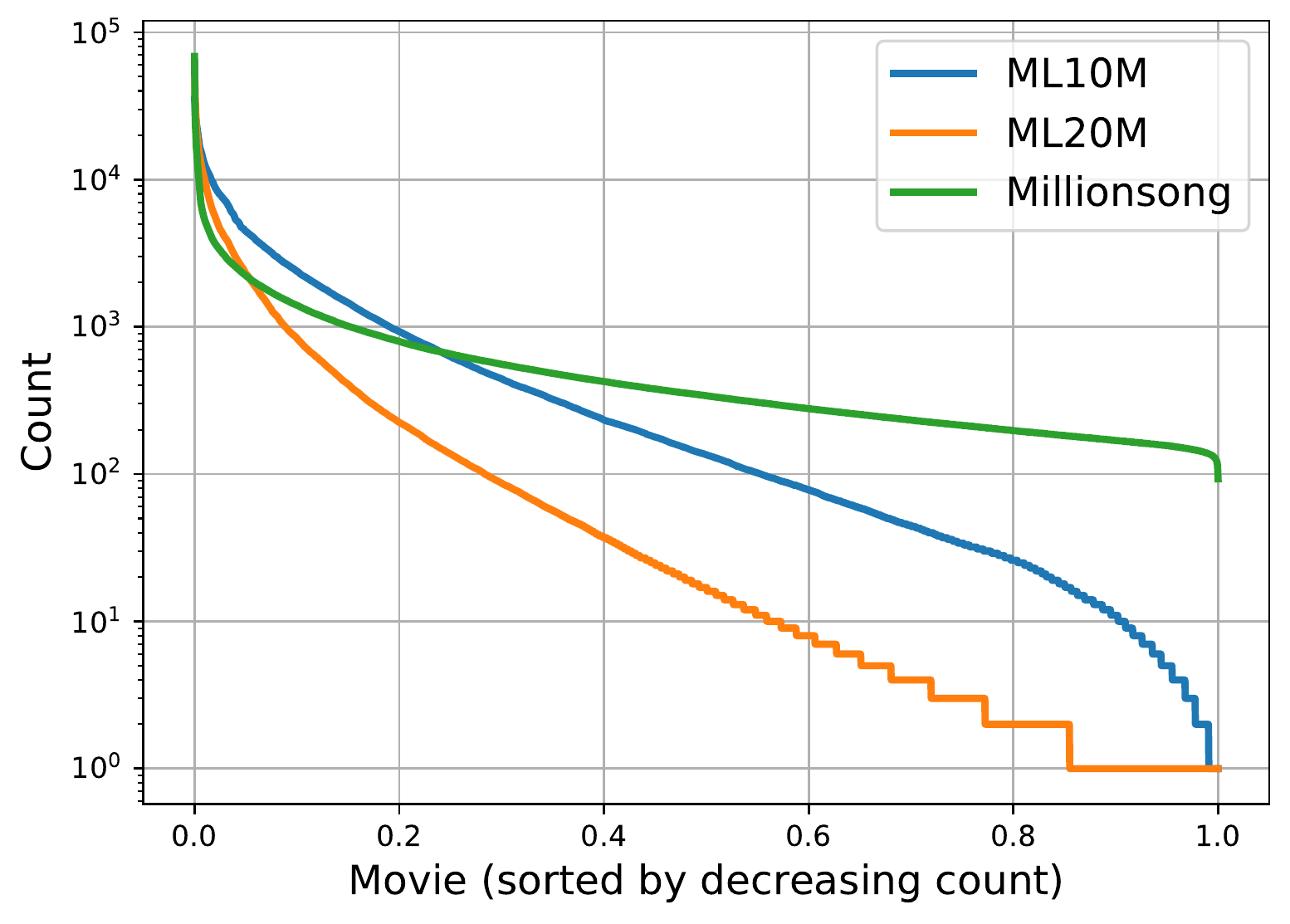}
\caption{Sorted task counts}
\end{subfigure}
\caption{Task distribution skew in MovieLens and Million Song Data.}
\label{fig:app_distribution_skew}
\end{figure}

\subsection{Detailed experimental setting}
We follow the same experimental setting as~\cite{jain2018differentially,DPALS,jain2021private}. When reporting $(\epsilon, \delta)$ DP guarantees, we use values of $\epsilon \in [1, 20]$, and take $\delta = 10^{-5}$ for ML10M and $\delta = 1/n$ for ML20M and MSD. All hyper-parameter values are specified in the source code.

We solve a private matrix completion problem, where the training data is a partially observed rating matrix $(y_{ij})_{(i,j) \in \Omega}$ (where $y_{ij}$ is the rating given by user $j$ to item $i$), and the goal is to compute a low-rank approximation $Y \approx UV^\top$, by minimizing the following objective function:
\[
L(U, V) = \sum_{(i,j) \in \Omega} (\innerp{u_i}{v_j} - y_{ij})^2 + \lambda|u_i|^2 + \lambda|v_j|^2.
\]
The matrix $U \in \Rbb^{n \times d}$ represents item embeddings, and the matrix $V \in \Rbb^{m \times d}$ represents user embeddings.

\paragraph{Algorithms}
We consider a family of alternating minimization algorithms studied by~\cite{DPALS,jain2021private}, in which we will use our algorithm as a sub-routine. This is summarized in Algorithm~\ref{alg:am}.

\begin{algorithm}[h]
\caption{Alternating Minimization for Private Matrix Completion\label{alg:am}}
\begin{algorithmic}[1]
\STATE{\bf Inputs}: Training data $\{y_{ij}\}_{(i,j) \in \Omega}$, number of steps $T$, initial item matrix $\hat U^{(0)}$, pre-processing RDP budget $\beta_0$, training RDP budget $\beta$.\\
\STATE Pre-process the training data (with RDP budget $\beta_0$).
\FOR{$1 \leq t\leq T$}
\STATE $\hat V^{(t)}\leftarrow \argmin_{V} L(\hat U^{(t-1)}, V)$\\
\STATE Compute a differentially private solution $\hat U^{(t)}$ of $\min_{U} L(U, \hat V^{(t)})$ (with RDP budget $\beta$) \label{line:u_update}
\ENDFOR\\
\STATE {\bf Return} $\hat U^{(T)}$
\end{algorithmic}
\end{algorithm}
Algorithm~\ref{alg:am} describes a family of algorithms, that includes DPALS~\citep{DPALS}, and the AltMin algorithm of~\cite{jain2021private}. It starts by pre-processing the training data (this includes centering the data, and computing private counts to be used for sampling or for computing adaptive weights). Then, it alternates between updating $V$ and updating $U$. Updating $V$ is done by computing an exact least squares solution, while updating $U$ is done differentially privately. By simple RDP composition, the final mechanism is $(\alpha, \alpha(\beta_0 + T\beta))$-RDP, which we translate to $(\epsilon, \delta)$-DP.

\begin{remark}
The algorithm only outputs the item embedding matrix $\hat U$. This is indeed sufficient for the recommendation task: given an item matrix $\hat U$ (learned differentially privately), to generate predictions for user $j$, one can compute the user's embedding $v_j^* = \argmin_{v\in \Rbb^d} \sum_{i \in \Omega^j} (\innerp{\hat  u_i}{v} - y_{ij})^2 + \lambda|v|^2$, then complete the $j$-th column by computing $\hat Uv_j^*$. This minimization problem only depends on the published matrix $\hat U$ and user $j$'s data, therefore can be done in isolation for each user (for example on the user's own device), without incurring additional privacy cost. This is sometimes referred to as the billboard model of differential privacy, see for example~\citep{jain2021private}.
\end{remark}

Notice that updating the item embeddings $U$ (Line~\ref{line:u_update} in Algorithm~\ref{alg:am}) consists in solving the following problem differentially privately:
\[
\min_{U} \sum_{i = 1}^m \sum_{j \in \Omega_i} (\innerp{u_i}{v_j} - y_{ij})^2 + \lambda |u_i|^2.
\]
This corresponds to our multi-task problem~\eqref{eq:obj}, where we identify each item to one task, with parameters $\theta_i = u_i$, features $x_{ij} = v_j$, and labels $y_{ij}$. Then we can either apply Algorithm~\ref{alg:ssp} (SSP with adaptive weights), or Algorithm~\ref{alg:dpsgd} (DPSGD with adaptive weights) to compute the item update. Our algorithms and the baselines we compare to are summarized in Table~\ref{tbl:algorithms}.
\begin{table}[h!]
\centering
\small
\caption{Algorithms used in experiments.}
\label{tbl:algorithms}
\begin{tabular}{lp{5cm}l}
\toprule
Method & Sub-routine used to update $U$ (Line~\ref{line:u_update} in Algorithm~\ref{alg:am}) & Budget allocation method \\ [0.5ex]
\midrule
DPALS (uniform)  & SSP & Uniform sampling \\
DPALS (tail) & SSP & Tail-biased sampling of~\citep{DPALS} \\
AdaDPALS & Algorithm~\ref{alg:ssp} & Adaptive weights (eq.~\eqref{eq:w}-\eqref{eq:w_final}) \\
DPSGD (uniform)  & DPSGD & Uniform sampling \\
DPSGD (tail) & DPSGD & Tail-biased sampling of~\citep{DPALS} \\
AdaDPSGD & Algorithm~\ref{alg:dpsgd} & Adaptive weights (eq.~\eqref{eq:w}-\eqref{eq:w_final}) \\
\bottomrule
\end{tabular}
\end{table}

We compare to two algorithms: the DPALS method~\citep{DPALS} which is the current SOTA on these benchmarks, and the DPSGD algorithm (which we found to perform well in the full-batch regime, at the cost of higher run times). In each case, we compare three methods to perform budget allocation: using uniform sampling, the tail-biased sampling heuristic of~\cite{DPALS}, and our adaptive weights method. Note that tail sampling is \emph{already adaptive to the task skew}: it estimates task counts and uses them to select the tasks to sample for each user.

Both for tail-sampling and adaptive weights, movie counts are estimated privately, and we account for the privacy cost of doing so. The proportion of RDP budget spent on estimating counts (out of the total RDP budget) is $20\%$ for $\epsilon = 20$, $14\%$ for $\epsilon = 5$ and $12\%$ for $\epsilon = 1$. This was tuned on the DPALS baseline (with tail sampling).


\paragraph{Metrics}
\def\OmegaT{\Omega^{\text{test}}}
\def\OjT{{\Oj}^{\text{test}}}
The quality metrics that the benchmarks use are defined as follows: Let $\OmegaT$ be the set of test ratings. Then for a given factorization $U, V$,
\[
\text{RMSE}(U, V) = \sqrt{\frac{\sum_{(i,j) \in \OmegaT} (\innerp{u_i}{v_j} - y_{ij})^2}{|\OmegaT|}}
\]
Recall is defined as follows. For a given user $i$, let $\OjT$ be the set of movies rated by the user $j$. If we denote by $\hat\Oj$ the set of top $k$ predictions for user $j$, then the recall is defined as
\[
\text{Recall@k} = \frac{1}{n}\sum_{j = 1}^n \frac{|\OjT \cap \hat\Omega^j|}{\min(k, |\OjT|)}.
\]

\begin{figure}[h]
\centering
\begin{subfigure}[b]{0.33\textwidth}
\includegraphics[width=\textwidth]{figures/ml10m_sliced_eps1.pdf}
\caption{ML10M, $\epsilon = 1$.}
\label{fig:app_10m_a}
\end{subfigure}\hfill
\begin{subfigure}[b]{0.33\textwidth}
\includegraphics[width=\textwidth]{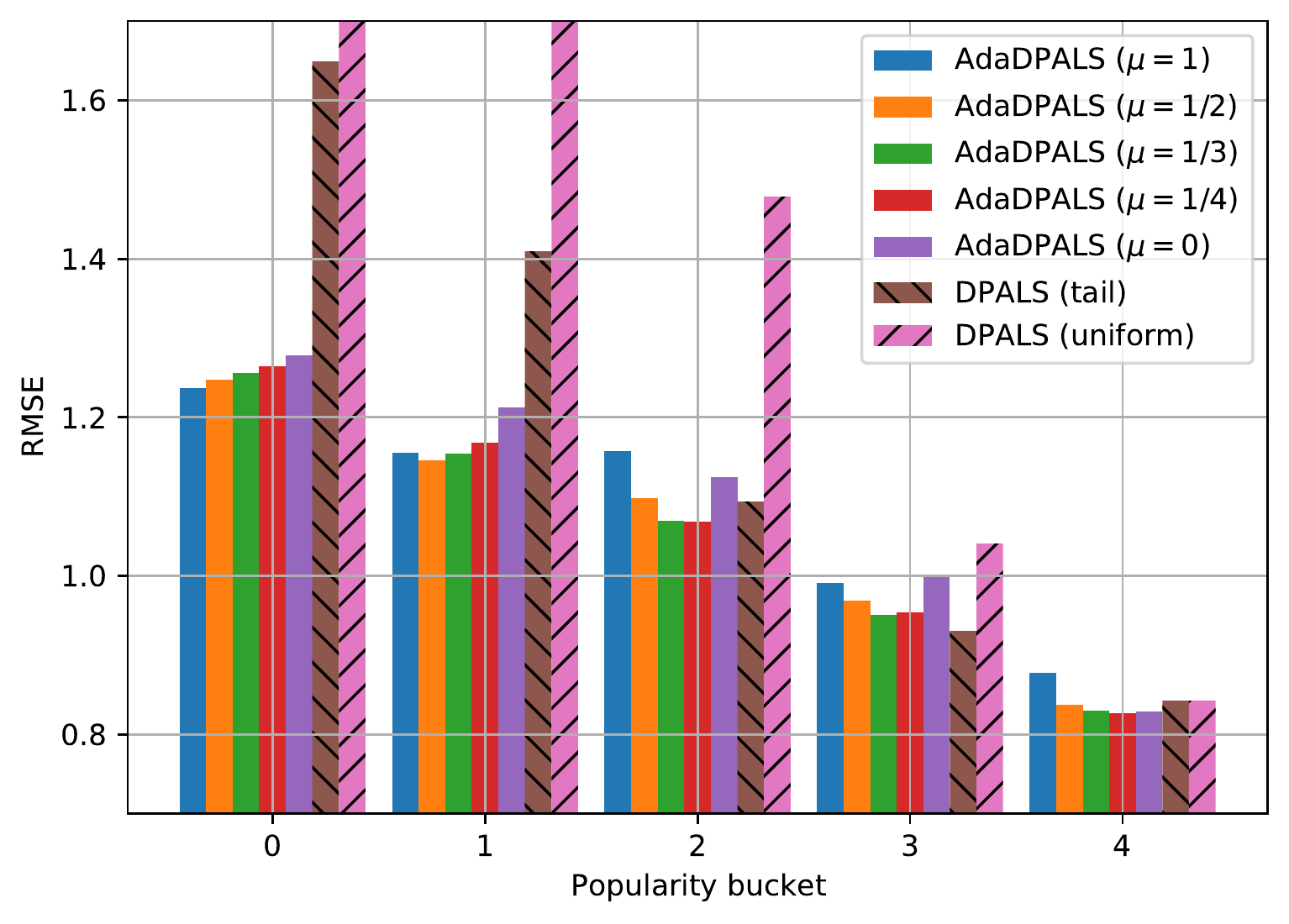}
\caption{ML10M, $\epsilon = 5$.}
\end{subfigure}
\begin{subfigure}[b]{0.33\textwidth}
\includegraphics[width=\textwidth]{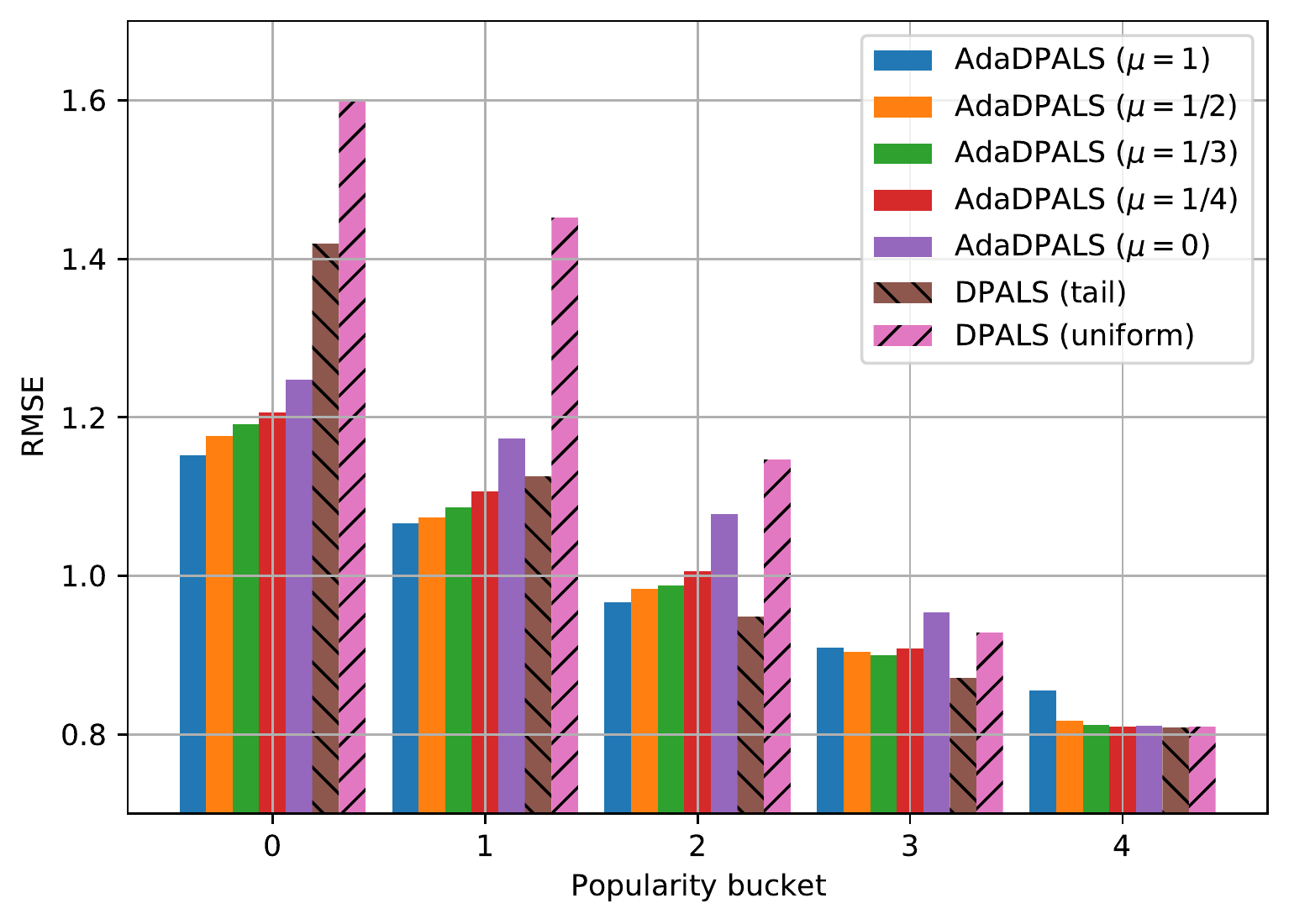}
\caption{ML10M, $\epsilon = 20$.}
\end{subfigure}

\begin{subfigure}[b]{0.33\textwidth}
\includegraphics[width=\textwidth]{figures/ml20m_sliced_eps1.pdf}
\caption{ML20M, $\epsilon = 1$.}
\end{subfigure}\hfill
\begin{subfigure}[b]{0.33\textwidth}
\includegraphics[width=\textwidth]{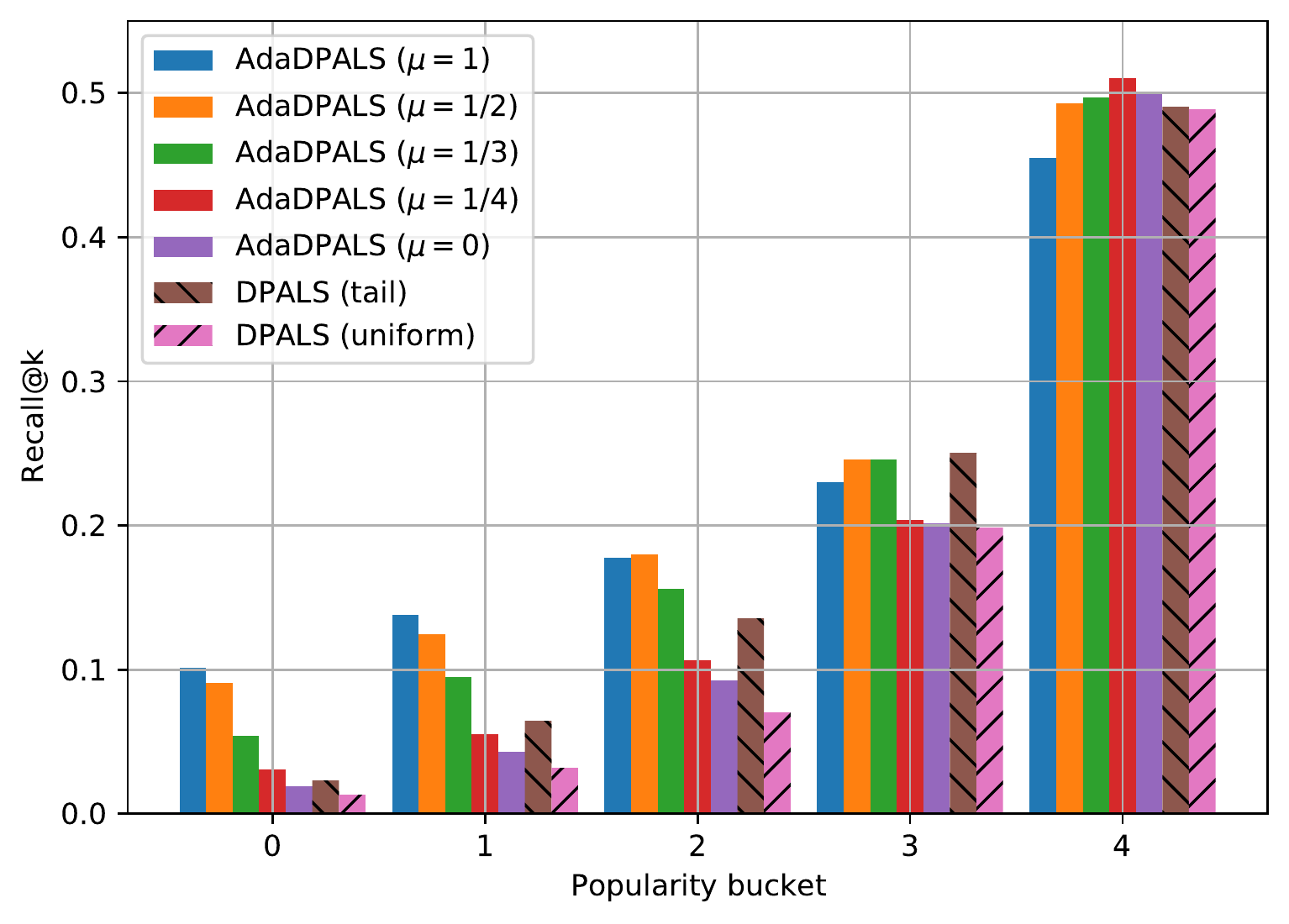}
\caption{ML20M, $\epsilon = 5$.}
\end{subfigure}
\begin{subfigure}[b]{0.33\textwidth}
\includegraphics[width=\textwidth]{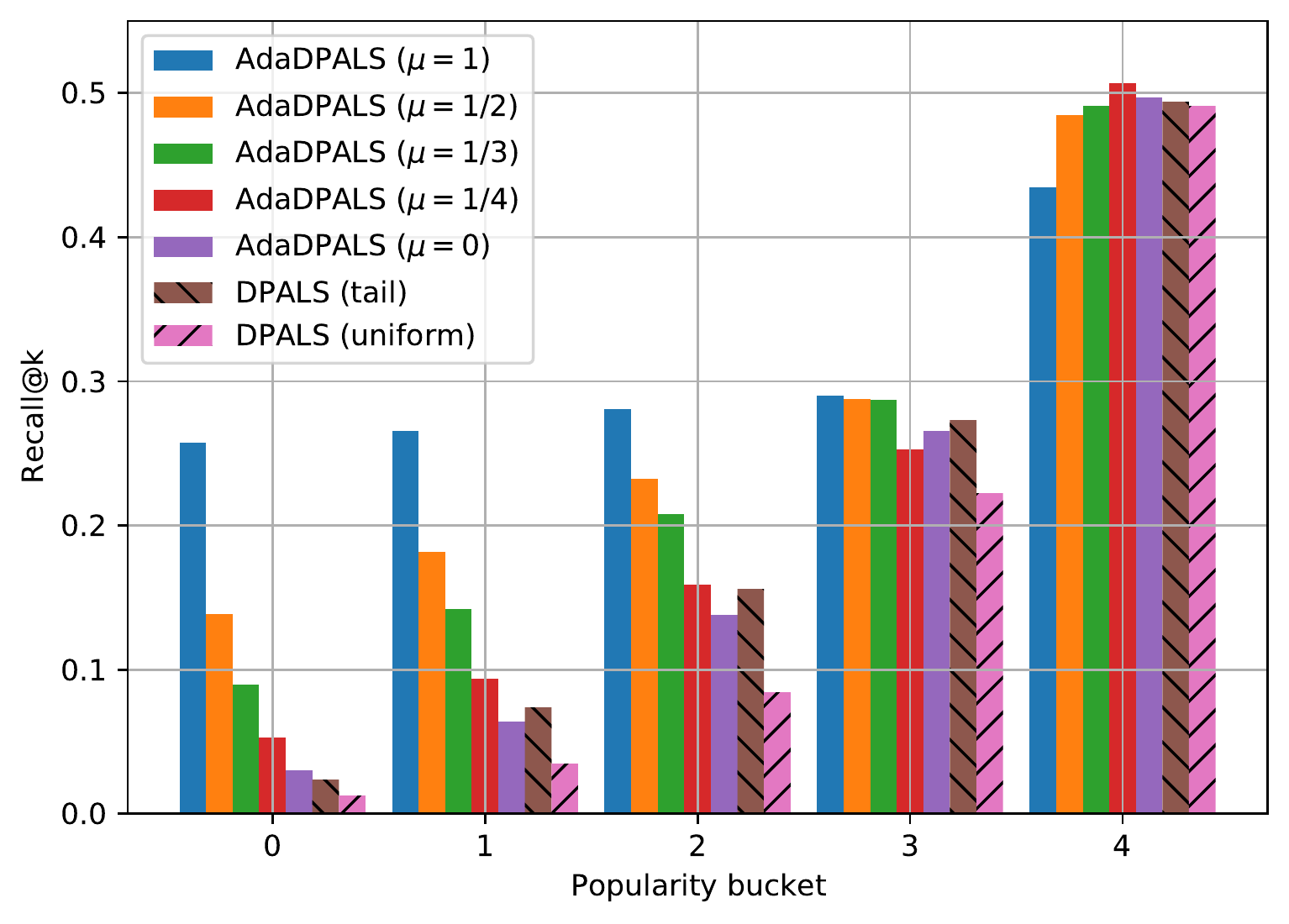}
\caption{ML20M, $\epsilon = 20$.}
\end{subfigure}

\begin{subfigure}[b]{0.33\textwidth}
\includegraphics[width=\textwidth]{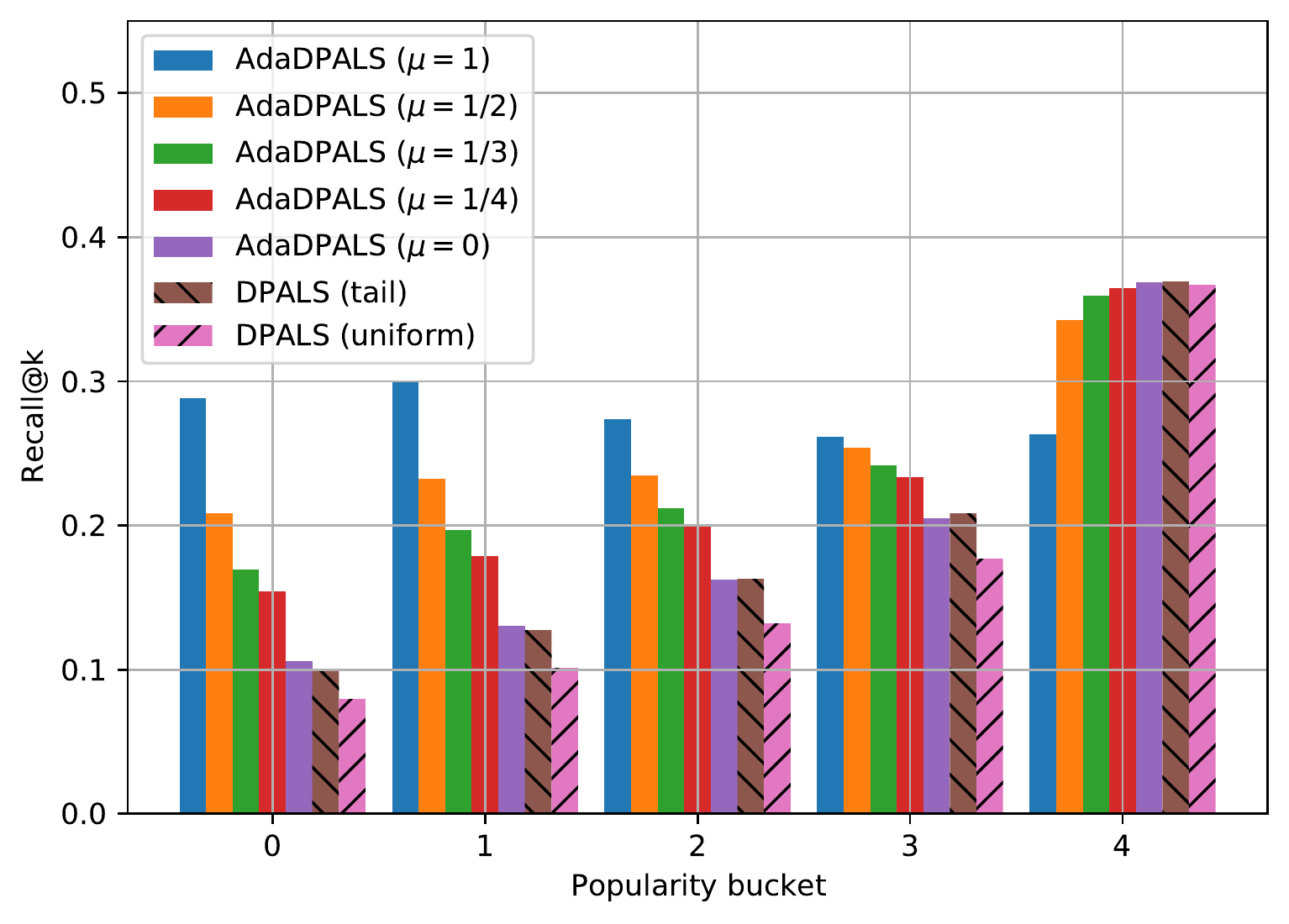}
\caption{MSD, $\epsilon = 1$.}
\end{subfigure}\hfill
\begin{subfigure}[b]{0.33\textwidth}
\includegraphics[width=\textwidth]{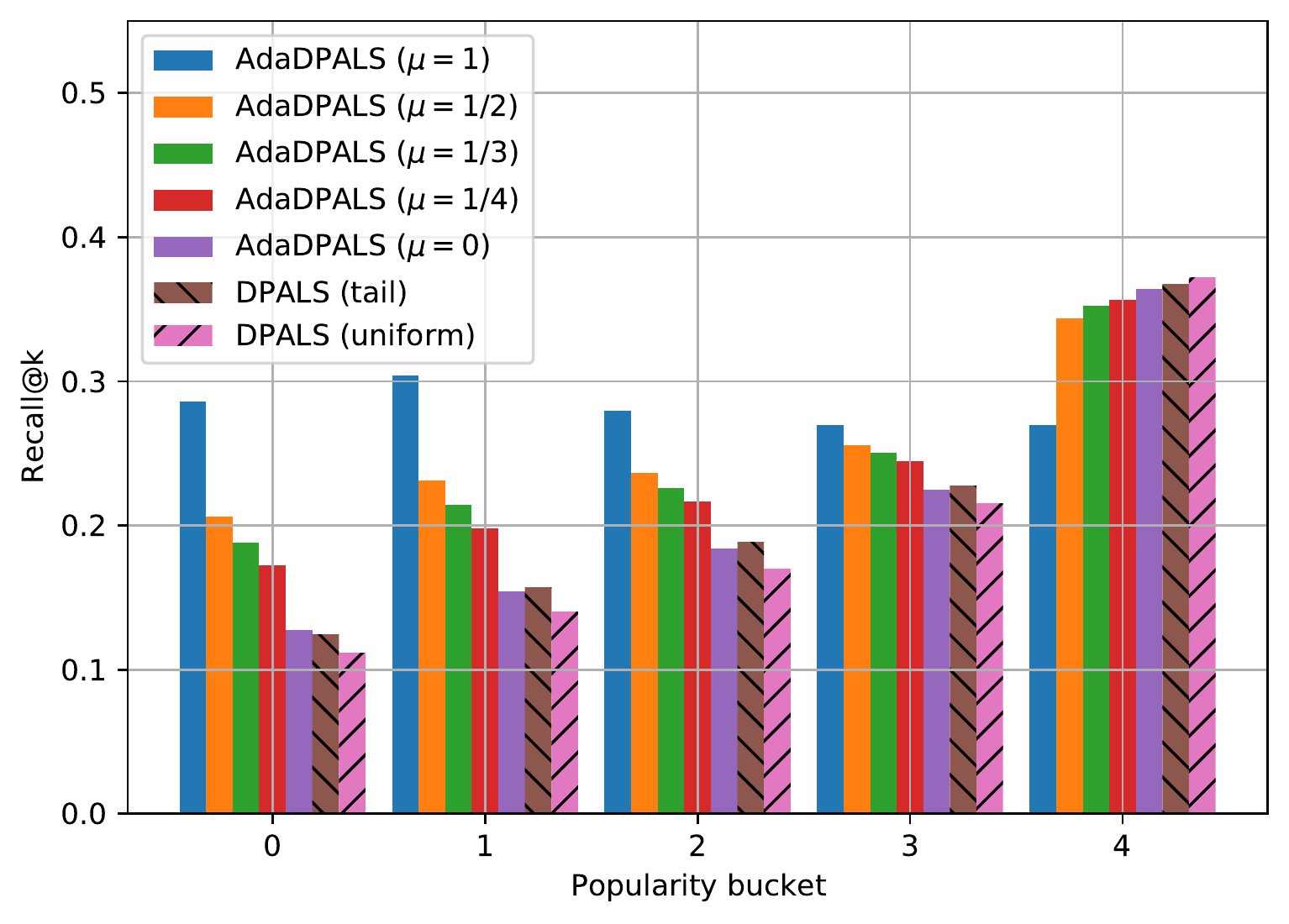}
\caption{MSD, $\epsilon = 5$.}
\end{subfigure}
\begin{subfigure}[b]{0.33\textwidth}
\includegraphics[width=\textwidth]{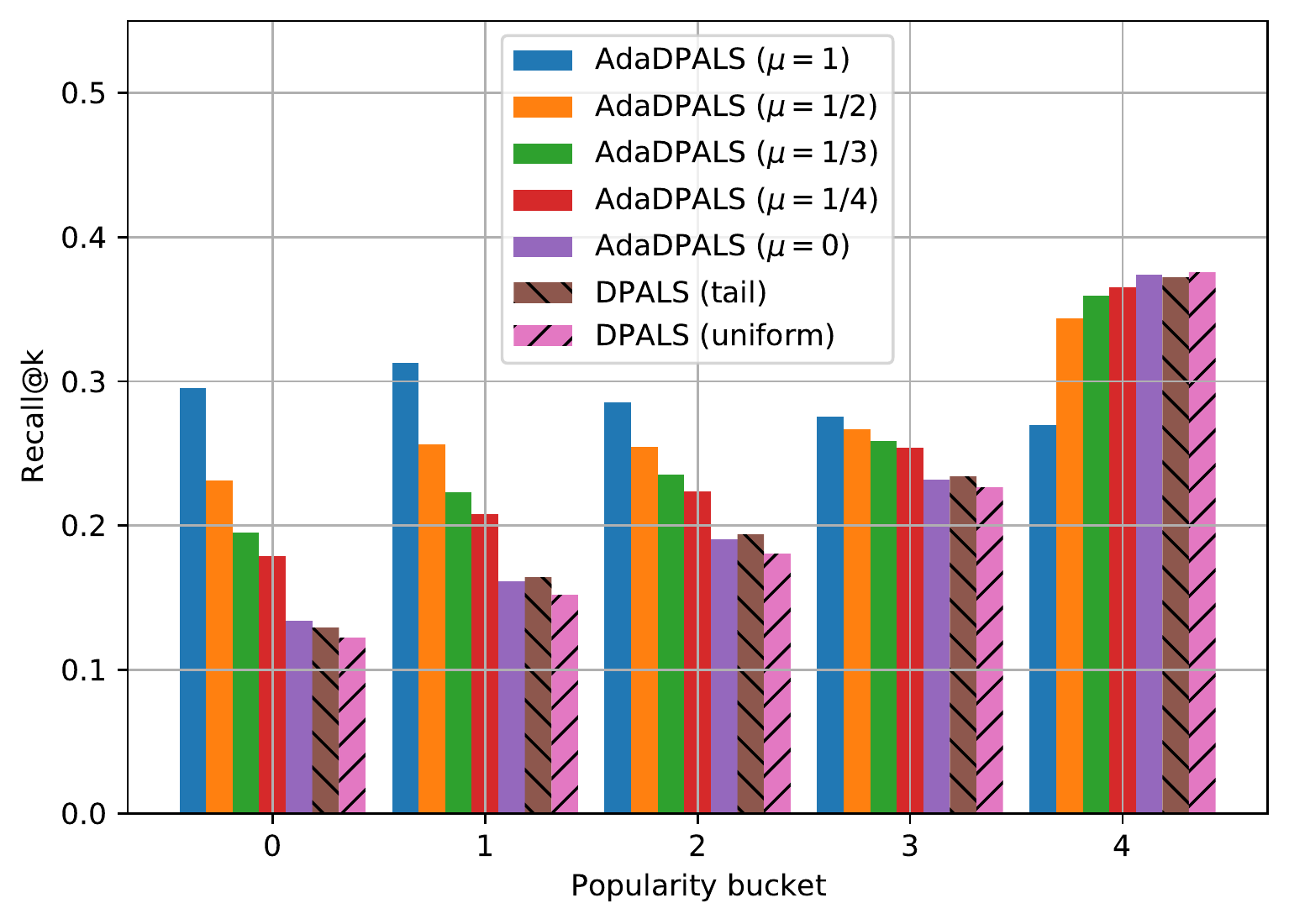}
\caption{MSD, $\epsilon = 20$.}
\end{subfigure}
\caption{Metrics sliced by movie popularity (i.e. frequency), on the ML10M, ML20M, and MSD benchmark, using the DPALS method. Each bucket contains an equal number of movies. Buckets are ordered by increasing frequency.}
\label{fig:app_sliced}
\end{figure}

\begin{figure}[h!]
\centering
\includegraphics[width=0.33\textwidth]{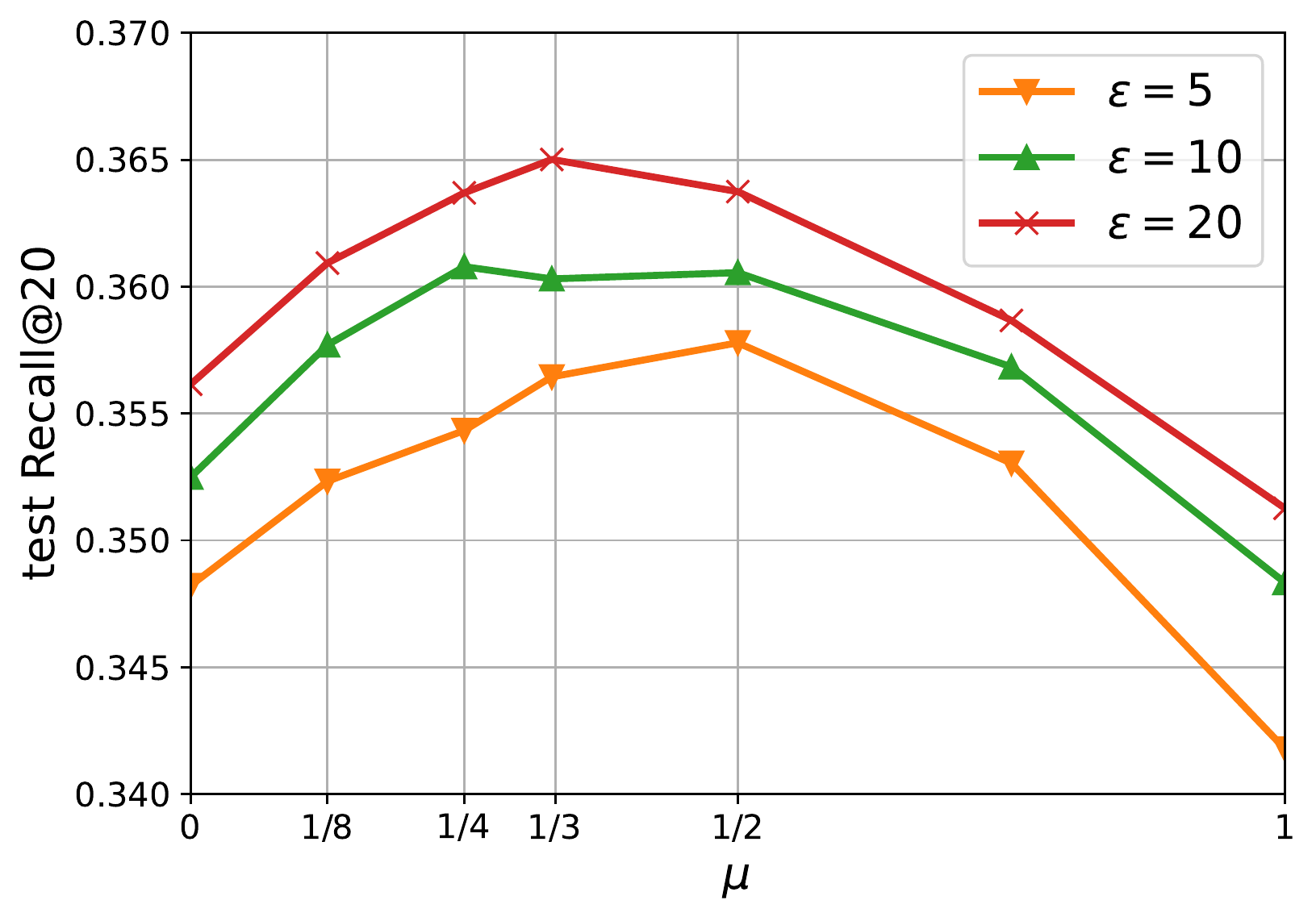}
\hspace{.1\textwidth}
\includegraphics[width=0.33\textwidth]{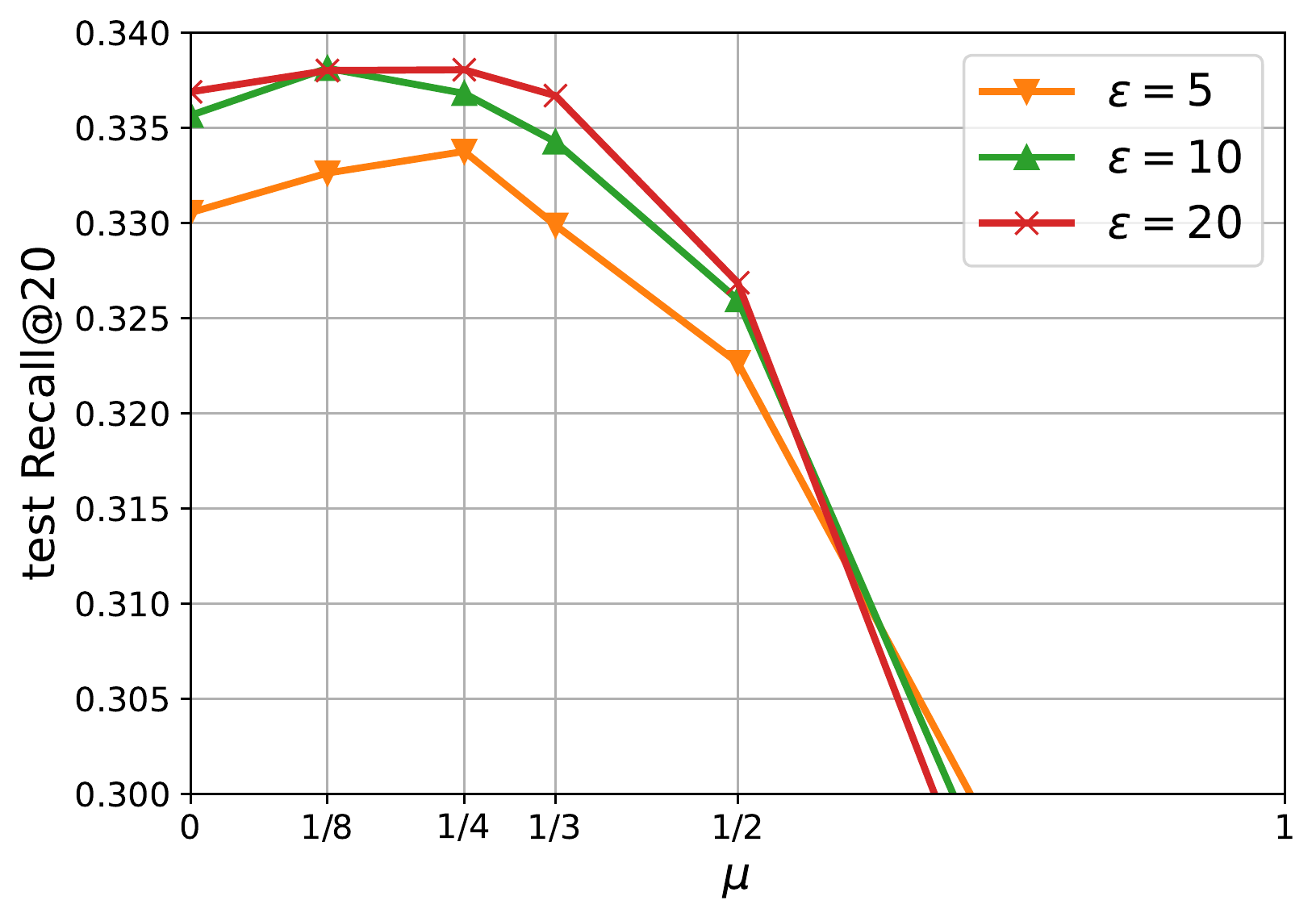}
\caption{Comparison of the adaptive weights method with different values of $\mu$ on ML20M, when applied to DPALS (left) and DPSGD (right).}
\label{fig:app_tradeoff_mu}
\end{figure}


\subsection{Quality impact on head/tail movies}

In Figure~\ref{fig:app_sliced}, we report sliced RMSE (on ML10M) and Recall (on ML20M and MSD), for different values of $\epsilon$.

The following trend can be observed on all benchmarks, across all values of $\epsilon$. DPALS with tail sampling improves upon uniform sampling, especially on the tail buckets. Our method (AdaDPALS) further improves upon tail sampling. The improvement is quite significant for lower values of $\epsilon$, and for tail buckets. For example, on ML10M with $\epsilon = 1$ (Figure~\ref{fig:app_10m_a}) we observe a large gap in RMSE, \emph{across all values of $\mu$}; the improvement is at least 21.6\% on bucket 0, at least 23.7\% on bucket 1, at least 22.8\% on bucket 3, and at least 8.4\% on bucket 4. The gap narrows as $\epsilon$ increases, which is consistent with the global privacy/utility trade-off plots in Figure~\ref{fig:trade-off}.

The exponent $\mu$ controls the trade-off between head and tail tasks: recall that the weights are defined as $\omega_i \propto 1/\hat n_i^{\mu}$ where $\hat n_i$ are the count estimates. A larger value of $\mu$ induces larger weights (and hence better quality) on the tail. This is visible on both benchmarks and across all values of $\epsilon$: on lower buckets 0 and 1, better performance is obtained for larger values of $\mu$, while the trend is reversed for the top bucket.

When comparing performance on the overall objective, we find that the best performance is typically achieved when $\mu = 1/4$, see Figures~\ref{fig:mu_ml10m} and~\ref{fig:app_tradeoff_mu}. When applied to DPALS performance remains high for a range of $\mu \in [1/4, 1/2]$. When applied to DPSGD, performance seems more sensitive to $\mu$, and the best performance is achieved for $\mu = 1/4$. 

\subsection{Qualitative evaluation on ML20M}
To give a qualitative evaluation of the improvements achieved by our method, we inspect a few example queries. Though anecdotal, these examples give a perhaps more concrete illustration of some of the quality impact that our method can have, especially on tail recommendations. We will compare the following models: the ALS non-private baseline (same model in Figure~\ref{fig:trade-off}-b), and two private models with $\epsilon = 1$: the DPALS method with tail-biased sampling, and Ada-DPALS (with adaptive weights) with $\mu = 1/3$ (we found that values of $\mu \in [1/4, 1/2]$ are qualitatively similar).

We evaluate the models by displaying the nearest neighbor movies to a given query movie, where the similarity between movies is defined by the learned movie embedding matrix $\hat V$ (the similarity between two movies $i_1$ and $i_2$ is $\innerp{\hat v_{i_1}}{\hat v_{i_2}}$. We select a few examples in Table~\ref{tbl:neighbors}; for additional examples, the models can be trained and queried interactively using the provided code. We select examples with varying levels of frequency (shown in the last column), to illustrate how privacy may affect quality differently depending on item frequency.

The first query is {\bf The Shawshank Redemption}, the most frequent movie in the data set. We see a large overlap of the top nearest neighbors according to all three models. In other words, privacy has little impact on this item. A similar observation can be made for other popular items.

The second query is {\bf Pinocchio}, a Disney animated movie from 1940. The nearest neighbors according to the non-private baseline (ALS) are other Disney movies from neighboring decades. The DPALS results are noisy: some of the top neighbors are Disney movies, but the fifth and sixth neighbors seem unrelated (Action and Drama movies). The AdaDPALS model returns more relevant results, all of the neighbors being Disney movies.

The third example is {\bf Harry Potter and the Half-Blood Prince}, an Adventure/Fantasy movie. The nearest neighbors from the ALS baseline are mostly other Harry Potter movies. The DPALS model misses the Harry Potter neighbors, and instead returns mostly action/adventure movies with varying degrees of relevance. The AdaDPALS model recovers several of Harry Potter neighbors.

The next example is {\bf Nausicaä of the Valley of the Wind}, a Japanese animated movie from Studio Ghibli, released in 1984. The nearest neighbors according to the ALS model are similar movies from Studio Ghibli. The neighbors returned by DPALS are much more noisy: The first result (Spirited Away) is a Studio Ghibli movie and is the most relevant in the list. Other neighbors in the list are arguably unrelated to the query. AdaDPALS returns much more relevant results, all neighbors are Japanese animation movies, and five out of six are from Studio Ghibli.

The last example is {\bf Interstellar}, a Sci-Fi movie released in 2014. The ALS neighbors are other popular movies released around the same time (2013-2014) with a bias towards Action/Sci-Fi. Both private models (DPALS and AdaDPALS) return mixed results. Some results are relevant (Action/Sci-Fi movies) but others are arguably much less relevant, for example the third DPALS neighbor is a Hitchcock movie from 1945. AdaDPALS neighbors appear slightly better overall, in particular it returns two movies from the same director.

As can be seen from these examples, AdaDPALS generally returns better quality results compared to DPALS. If we compare both models to the non-private baseline (ALS), the overlap between AdaDPALS and ALS is generally larger then the overlap between DPALS and ALS. To quantify this statement, we generate, for each movie, the top-20 nearest neighbors according to each model, then compute the percentage overlap with ALS. The results are reported in Figure~\ref{fig:overlap}.

\begin{figure}[h!]
\centering
\includegraphics[width=.33\textwidth]{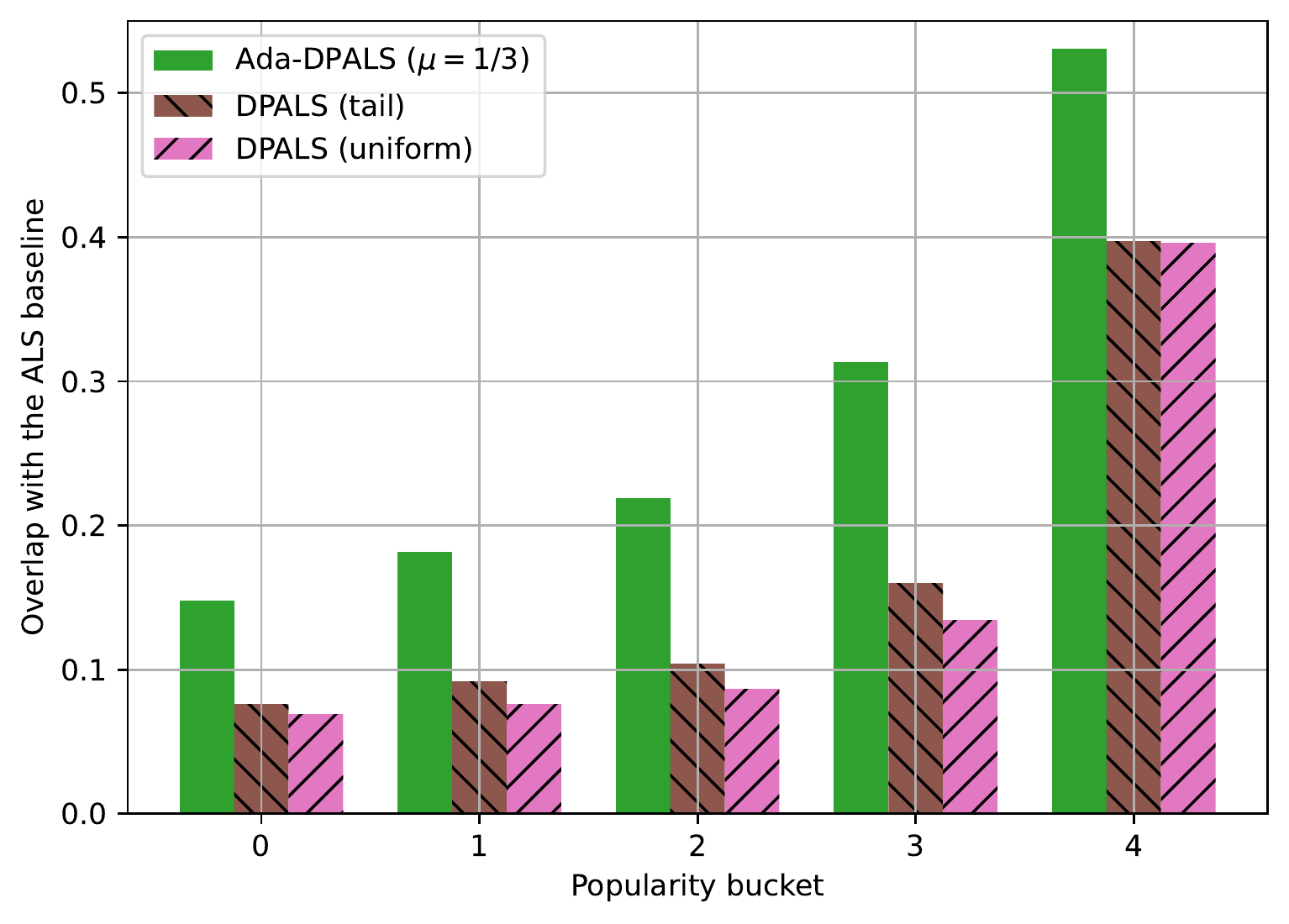}
\caption{Percentage overlap between the top-20 nearest neighbors according to private models ($\epsilon = 1$) with the top-20 nearest neighbors according to the non-private baseline (ALS). Each bucket contains an equal number of movies (ordered by increasing frequency).}
\label{fig:overlap}
\end{figure}

\begin{table}[h!]
\centering
\resizebox{\textwidth}{!}{%
\begin{tabular}{lllp{2.5cm}}
\toprule
Model & Movie title & Genres & Frequency\newline (in training set)\\
\midrule
\midrule
      &  \bf Shawshank Redemption, The (1994) &                  Crime|Drama &        47518 \\
\midrule
ALS
      &        Usual Suspects, The (1995) &       Crime|Mystery|Thriller &        33834 \\
      &  Silence of the Lambs, The (1991) &        Crime|Horror|Thriller &        42738 \\
      &               Pulp Fiction (1994) &  Comedy|Crime|Drama|Thriller &        44626 \\
      &           Schindler's List (1993) &                    Drama|War &        35334 \\
      &                  Apollo 13 (1995) &         Adventure|Drama|IMAX &        26192 \\
      &               Forrest Gump (1994) &     Comedy|Drama|Romance|War &        40422 \\
\midrule
DPALS
      &  Silence of the Lambs, The (1991) &        Crime|Horror|Thriller &        42738 \\
      &        Usual Suspects, The (1995) &       Crime|Mystery|Thriller &        33834 \\
      &           Schindler's List (1993) &                    Drama|War &        35334 \\
      &               Pulp Fiction (1994) &  Comedy|Crime|Drama|Thriller &        44626 \\
      &               Forrest Gump (1994) &     Comedy|Drama|Romance|War &        40422 \\
      &                 Braveheart (1995) &             Action|Drama|War &        32735 \\
\midrule
AdaDPALS
      &  Silence of the Lambs, The (1991) &        Crime|Horror|Thriller &        42738 \\
      &               Pulp Fiction (1994) &  Comedy|Crime|Drama|Thriller &        44626 \\
      &        Usual Suspects, The (1995) &       Crime|Mystery|Thriller &        33834 \\
      &               Forrest Gump (1994) &     Comedy|Drama|Romance|War &        40422 \\
      &           Schindler's List (1993) &                    Drama|War &        35334 \\
      &                 Braveheart (1995) &             Action|Drama|War &        32735 \\
\midrule
\midrule
      &                    \bf Pinocchio (1940) &            Animation|Children|Fantasy|Musical &         5120 \\
\midrule
ALS
      &  Snow White and the Seven Dwarfs (1937) &      Animation|Children|Drama|Fantasy|Musical &         7865 \\
      &                            Dumbo (1941) &              Animation|Children|Drama|Musical &         3580 \\
      &                       Cinderella (1950) &    Animation|Children|Fantasy|Musical|Romance &         3957 \\
      &                  Aristocats, The (1970) &                            Animation|Children &         2669 \\
      &                         Fantasia (1940) &            Animation|Children|Fantasy|Musical &         6135 \\
      &              Alice in Wonderland (1951) &  Adventure|Animation|Children|Fantasy|Musical &         3487 \\
\midrule
DPALS
      &  Snow White and the Seven Dwarfs (1937) &         Animation|Children|Drama|Fantasy|Musical &         7865 \\
      &                          Jumanji (1995) &                       Adventure|Children|Fantasy &         6203 \\
      &             Beauty and the Beast (1991) &  Animation|Children|Fantasy|Musical|Romance|IMAX &        16391 \\
      &                          Aladdin (1992) &      Adventure|Animation|Children|Comedy|Musical &        19912 \\
      &                        Assassins (1995) &                            Action|Crime|Thriller &         1146 \\
      &           Miracle on 34th Street (1947) &                                     Comedy|Drama &         2445 \\
\midrule
AdaDPALS
      &  Snow White and the Seven Dwarfs (1937) &  Animation|Children|Drama|Fantasy|Musical &         7865 \\
      &                         Fantasia (1940) &        Animation|Children|Fantasy|Musical &         6135 \\
      &                       Pocahontas (1995) &  Animation|Children|Drama|Musical|Romance &         2815 \\
      &          Sword in the Stone, The (1963) &        Animation|Children|Fantasy|Musical &         2217 \\
      &                            Dumbo (1941) &          Animation|Children|Drama|Musical &         3580 \\
      &                 Jungle Book, The (1994) &                Adventure|Children|Romance &         2765 \\
\midrule
\midrule
      & \bf       Harry Potter and the Half-Blood Prince (2009) &                  Adventure|Fantasy|Mystery|Romance|IMAX &         2176 \\
\midrule
ALS
      &  Harry Potter and the Deathly Hallows: Part 1 (2010) &                           Action|Adventure|Fantasy|IMAX &         2099 \\
      &     Harry Potter and the Order of the Phoenix (2007) &                            Adventure|Drama|Fantasy|IMAX &         2896 \\
      &  Harry Potter and the Deathly Hallows: Part 2 (2011) &             Action|Adventure|Drama|Fantasy|Mystery|IMAX &         2265 \\
      &                               Sherlock Holmes (2009) &                           Action|Crime|Mystery|Thriller &         2877 \\
      &           Harry Potter and the Goblet of Fire (2005) &                         Adventure|Fantasy|Thriller|IMAX &         4773 \\
      &                                       Tangled (2010) &  Animation|Children|Comedy|Fantasy|Musical|Romance|IMAX &         1259 \\
\midrule
DPALS
      &                          Animatrix, The (2003) &                           Action|Animation|Drama|Sci-Fi &         1216 \\
      &                             Ratatouille (2007) &                                Animation|Children|Drama &         4728 \\
      &                           Avengers, The (2012) &                            Action|Adventure|Sci-Fi|IMAX &         2770 \\
      &                                 Tangled (2010) &  Animation|Children|Comedy|Fantasy|Musical|Romance|IMAX &         1259 \\
      &                     Slumdog Millionaire (2008) &                                     Crime|Drama|Romance &         5415 \\
      &      Sherlock Holmes: A Game of Shadows (2011) &          Action|Adventure|Comedy|Crime|Mystery|Thriller &         1166 \\
\midrule
AdaDPALS
      &  Harry Potter and the Deathly Hallows: Part 2 (2011) &     Action|Adventure|Drama|Fantasy|Mystery|IMAX &         2265 \\
      &      Harry Potter and the Prisoner of Azkaban (2004) &                          Adventure|Fantasy|IMAX &         6433 \\
      &                                   Ratatouille (2007) &                        Animation|Children|Drama &         4728 \\
      &            Sherlock Holmes: A Game of Shadows (2011) &  Action|Adventure|Comedy|Crime|Mystery|Thriller &         1166 \\
      &     Harry Potter and the Order of the Phoenix (2007) &                    Adventure|Drama|Fantasy|IMAX &         2896 \\
      &                                        Avatar (2009) &                    Action|Adventure|Sci-Fi|IMAX &         4960 \\
\bottomrule
\end{tabular}
}
\caption{Nearest neighbors according to ALS (non-private), DPALS, and AdaDPALS ($\mu =1/3$).}
\label{tbl:neighbors}
\medskip
\medskip
\end{table}

\newgeometry{lmargin=1.9cm,rmargin=1.9cm}
\begin{table}[h]
\centering
\resizebox{\textwidth}{!}{%
\begin{tabular}{lllp{2.5cm}}
\toprule
Model & Movie title & Genres & Frequency\newline (in training set)\\
\midrule
\midrule
      &  \bf Nausicaä of the Valley of the Wind (Kaze no tani no Naushika) (1984) &            Adventure|Animation|Drama|Fantasy|Sci-Fi &         2151 \\
\midrule
ALS
      &             Laputa: Castle in the Sky (Tenkû no shiro Rapyuta) (1986) &  Action|Adventure|Animation|Children|Fantasy|Sci-Fi &         2227 \\
      &                          My Neighbor Totoro (Tonari no Totoro) (1988) &                    Animation|Children|Drama|Fantasy &         3593 \\
      &                    Kiki's Delivery Service (Majo no takkyûbin) (1989) &          Adventure|Animation|Children|Drama|Fantasy &         1421 \\
      &                    Porco Rosso (Crimson Pig) (Kurenai no buta) (1992) &          Adventure|Animation|Comedy|Fantasy|Romance &         1022 \\
      &                        Grave of the Fireflies (Hotaru no haka) (1988) &                                 Animation|Drama|War &         2026 \\
      &                    Howl's Moving Castle (Hauru no ugoku shiro) (2004) &                 Adventure|Animation|Fantasy|Romance &         3503 \\
\midrule
DPALS
      &                  Spirited Away (Sen to Chihiro no kamikakushi) (2001) &               Adventure|Animation|Fantasy &         9161 \\
      &                                                  Terminal, The (2004) &                      Comedy|Drama|Romance &         1963 \\
      &                                              Finding Neverland (2004) &                                     Drama &         3371 \\
      &                                                      Ring, The (2002) &                   Horror|Mystery|Thriller &         3535 \\
      &                                              Kill Bill: Vol. 1 (2003) &                     Action|Crime|Thriller &        12467 \\
      &                                      Man Who Wasn't There, The (2001) &                               Crime|Drama &         2157 \\
\midrule
AdaDPALS
      &                  Spirited Away (Sen to Chihiro no kamikakushi) (2001) &                         Adventure|Animation|Fantasy &         9161 \\
      &                          My Neighbor Totoro (Tonari no Totoro) (1988) &                    Animation|Children|Drama|Fantasy &         3593 \\
      &                              Princess Mononoke (Mononoke-hime) (1997) &            Action|Adventure|Animation|Drama|Fantasy &         6101 \\
      &                            Ghost in the Shell (Kôkaku kidôtai) (1995) &                                    Animation|Sci-Fi &         4070 \\
      &                    Howl's Moving Castle (Hauru no ugoku shiro) (2004) &                 Adventure|Animation|Fantasy|Romance &         3503 \\
      &             Laputa: Castle in the Sky (Tenkû no shiro Rapyuta) (1986) &  Action|Adventure|Animation|Children|Fantasy|Sci-Fi &         2227 \\
\midrule
\midrule
      & \bf              Interstellar (2014) &              Sci-Fi|IMAX &         1048 \\
\midrule
ALS
      &                  Gone Girl (2014) &           Drama|Thriller &          847 \\
      &           Edge of Tomorrow (2014) &       Action|Sci-Fi|IMAX &          881 \\
      &                    Gravity (2013) &       Action|Sci-Fi|IMAX &         1248 \\
      &    Guardians of the Galaxy (2014) &  Action|Adventure|Sci-Fi &         1072 \\
      &   Wolf of Wall Street, The (2013) &       Comedy|Crime|Drama &         1060 \\
      &  Grand Budapest Hotel, The (2014) &             Comedy|Drama &         1339 \\
\midrule
DPALS
      &    Day After Tomorrow, The (2004) &  Action|Adventure|Drama|Sci-Fi|Thriller &         1233 \\
      &                      Alive (1993) &                                   Drama &          994 \\
      &                 Spellbound (1945) &                Mystery|Romance|Thriller &         1365 \\
      &                Source Code (2011) &    Action|Drama|Mystery|Sci-Fi|Thriller &         1595 \\
      &  Fast and the Furious, The (2001) &                   Action|Crime|Thriller &         1406 \\
      &             Ice Storm, The (1997) &                                   Drama &         3020 \\
\midrule
AdaDPALS
      &        Django Unchained (2012) &                             Action|Drama|Western &         2692 \\
      &               Inception (2010) &  Action|Crime|Drama|Mystery|Sci-Fi|Thriller|IMAX &         9147 \\
      &  Dark Knight Rises, The (2012) &                      Action|Adventure|Crime|IMAX &         2848 \\
      &            Intouchables (2011) &                                     Comedy|Drama &         1803 \\
      &             Source Code (2011) &             Action|Drama|Mystery|Sci-Fi|Thriller &         1595 \\
      &                  Avatar (2009) &                     Action|Adventure|Sci-Fi|IMAX &         4960 \\
\bottomrule
\end{tabular}
}
\caption{Nearest neighbors according to ALS (non-private), DPALS, and AdaDPALS ($\mu =1/3$).}
\label{tbl:neighbors2}
\medskip
\medskip
\end{table}

\end{document}